\documentclass{article}
\usepackage{microtype}
\usepackage{graphicx}
\usepackage{subcaption}
\usepackage{booktabs}
\usepackage[dvipsnames,table]{xcolor}
\usepackage{soul}
\usepackage{fontawesome5}

 \usepackage[eandd, final]{neurips_2026}

\usepackage{natbib}
\usepackage[framemethod=tikz]{mdframed}
\mdfsetup{skipbelow=0pt} %
\usepackage{tikz}
\usepackage{wrapfig}
\usepackage{fontawesome5}
\usepackage[utf8]{inputenc} %
\usepackage[T1]{fontenc}    %
\usepackage{hyperref}       %
\usepackage{url}            %
\usepackage{booktabs}       %
\usepackage{amsfonts}       %
\usepackage{nicefrac}       %
\usepackage{microtype}      %
\usepackage{amsmath}
\usepackage{amssymb}
\usepackage{mathtools}
\usepackage{amsthm}

\usepackage{mathtools}
\mathtoolsset{showonlyrefs}
\usepackage{enumitem}
\usepackage[table]{xcolor}

\usepackage{amsmath}
\usepackage{amssymb}
\usepackage{mathtools}
\usepackage{amsthm}

\usepackage{mathtools}
\mathtoolsset{showonlyrefs}
\usepackage{enumitem}
\usepackage[table]{xcolor}

\usepackage[capitalize,noabbrev]{cleveref}
\usepackage[normalem]{ulem}

\theoremstyle{plain}
\newtheorem{theorem}{Theorem}[section]
\newtheorem{proposition}[theorem]{Proposition}
\newtheorem{lemma}[theorem]{Lemma}

\theoremstyle{definition}
\newtheorem{definition}[theorem]{Definition}

\theoremstyle{remark}

\newtheorem{example}{Example}

\usepackage[framemethod=tikz]{mdframed}
\usepackage[dvipsnames]{xcolor}        %
\usepackage{hyperref}      %
\usepackage{natbib}
\hypersetup{
    colorlinks=true,       %
    linkcolor=blue,       %
    urlcolor=blue,    %
}
\usepackage{graphicx}
\usepackage{algorithm}
\usepackage{algorithmic}
\usepackage{tikz}
\usetikzlibrary{positioning, calc, backgrounds, shapes.geometric, decorations.pathmorphing, shapes.callouts}
\usepackage{pgfplots}
\pgfplotsset{compat=1.18}
\definecolor{lgPSO}{rgb}{0.9725,0.3843,0.3843}
\definecolor{lgOMD}{rgb}{0.9412,0.5059,0.4157}
\definecolor{lgDFP}{rgb}{0.4980,0.0667,0.9608}
\definecolor{lgFict}{rgb}{0.0353,0.2118,0.7843}
\definecolor{lgFP}{rgb}{0.3882,0.6902,0.9725}
\definecolor{lgPI}{rgb}{0.9804,0.5608,0.7490}
\definecolor{lgSPI}{rgb}{0.4392,0.2471,0.3843}
\definecolor{lgBPI}{rgb}{0.5922,0.7255,0.7647}
\newcommand{\lgentry}[3]{%
  \tikz[baseline=-0.6ex]{\draw[#1, line width=0.9pt] (0,0) -- (0.45,0);
    \node[#2, fill=#1, draw=#1, inner sep=0pt] at (0.225,0) {};}~#3}
\newcommand{\algolegend}{%
  {\par\centering\footnotesize
  \lgentry{lgPSO}{circle, minimum size=4pt}{PSO}\hspace{0.8em}%
  \lgentry{lgOMD}{rectangle, minimum size=3.6pt}{OMD}\hspace{0.8em}%
  \lgentry{lgDFP}{diamond, minimum size=4.8pt}{Damped FP}\hspace{0.8em}%
  \lgentry{lgFict}{regular polygon, regular polygon sides=3, minimum size=5pt}{Fict.\ Play}\hspace{0.8em}%
  \lgentry{lgFP}{regular polygon, regular polygon sides=3, minimum size=5pt, shape border rotate=180}{FP}\hspace{0.8em}%
  \lgentry{lgPI}{regular polygon, regular polygon sides=5, minimum size=4.8pt}{PI}\hspace{0.8em}%
  \lgentry{lgSPI}{regular polygon, regular polygon sides=6, minimum size=4.8pt}{Smooth PI}\hspace{0.8em}%
  \lgentry{lgBPI}{star, star points=5, star point ratio=2.2, minimum size=5.5pt}{Boltzmann PI}\par}}

\usepackage[most]{tcolorbox}
\definecolor{darkgreen}{rgb}{0.0, 0.5, 0.0}%
\definecolor{pastelblue}{rgb}{0.30, 0.60, 0.80}
\definecolor{salmon}{rgb}{0.98, 0.50, 0.45}

\tikzstyle{arrow} = [thick,->,>=stealth]

\newcommand{\R}{\mathbb{R}}

\newcommand{\EE}{\mathbb{E}}

\newcommand{\BR}{\mathbf{BR}}
\newcommand{\states}{\mathcal{X}}
\newcommand{\actions}{\mathcal{A}}
\newcommand{\noises}{\mathcal{C}}
\newcommand{\exploitability}{\mathcal{E}}
\newcommand{\unif}{\mathrm{Unif}}
\newcommand{\cP}{\mathcal{P}}
\newcommand{\cR}{\mathcal{R}}

\DeclareMathOperator*{\argmin}{arg\,min}

\definecolor{pythonblue}{HTML}{3776AB}
\definecolor{pythonyellow}{HTML}{FFD43B}
\newcommand{\pythonicon}{\begin{tikzpicture}[baseline=(py.base)]
  \node[inner sep=0pt,text=pythonblue] (py) {\faPython};
  \clip (py.south west) -- (py.south east) -- (py.north east) -- cycle;
  \node[inner sep=0pt,text=pythonyellow] at (py.center) {\faPython};
\end{tikzpicture}}
  
\hypersetup{
    colorlinks=true,      %
    allcolors=blue!60!black,       %
}
\title{Bench-MFG: A Benchmark Suite for Learning in Stationary Mean Field Games}

\author{%
  Lorenzo Magnino \thanks{Emails: lm2183@cam.ac.uk, js12556@nyu.edu, mathieu.lauriere@nyu.edu} \\
  University of Cambridge %
  \And
  Jiacheng Shen \\
  NYU Center for Data Science \\
  NYU Shanghai %
  \AND
  Matthieu Geist\thanks{Work conducted while at Earth Species Project.} \\
  DeepFin Research
  \And
  Olivier Pietquin \\
  NorthSea AI %
  \And
  Mathieu Lauri\`ere \thanks{NYU Shanghai Center for Data Science, NYU-ECNU Institute of Mathematical Sciences at NYU Shanghai} \\
  NYU Shanghai %
}

\begin{document}

\maketitle
\begin{abstract}

    The rise of Mean Field Games (MFGs) has fostered a growing family of algorithms designed to solve large-scale multi-agent systems. However, the field currently lacks a standardized evaluation protocol, forcing researchers to rely on bespoke, isolated, and often simplistic environments. This fragmentation makes it difficult to assess the robustness, generalization, and failure modes of emerging methods. To address this gap, we propose a comprehensive benchmark suite for MFGs, focusing on the discrete-time, discrete-space, stationary setting. We introduce a taxonomy of problem classes, ranging from no-interaction and monotone games to potential and dynamics-coupled games, and provide prototypical environments for each. Furthermore, we present \textbf{MF-Garnets}, a method for generating random MFG instances to facilitate rigorous statistical testing. We benchmark a variety of learning algorithms across these environments, including a novel population-based approach (\textbf{MF-PSO}). To our knowledge, this is the first benchmark for MFGs. 

\noindent\hspace*{-3cm}%
\begin{minipage}{\dimexpr\linewidth+6cm\relax}
\centering
\faGithub\ \texttt{\href{https://github.com/lorenzomagnino/Bench-MFG}{lorenzomagnino/Bench-MFG}}
\hspace{2em}
\pythonicon\ \texttt{pip install bench-mfg-suite}
\end{minipage}
\end{abstract}
\vspace{-0.4cm}
\begin{figure}[h]
    \centering
    \resizebox{1\textwidth}{!}{%
        \input{tikz-figures/overview_icon}
    }
    \caption{\small \textbf{Bench-MFG} Overview}
    \vspace{-0.5cm}
\end{figure}
\vspace{-0.04cm}
\section{Introduction}
\vspace{-0.3cm}
The scalability of multi-agent learning remains a fundamental obstacle in the pursuit of intelligent decentralized systems. While Deep Reinforcement Learning (DRL) has demonstrated superhuman performance in zero-sum games such as Go and Chess~\citep{silver2016mastering,silver2017masteringchess}, and complex team-based environments like StarCraft~\citep{samvelyan2019starcraft}, these successes rarely translate effortlessly to systems with massive populations. As the number of agents increases, the joint state-action space expands exponentially, rendering standard multi-agent learning approaches computationally infeasible~\citep{busoniu2008comprehensive,yang2020game,zhang2021multi}. 
\textbf{Mean-field learning} approaches have therefore been developed to make learning systems scalable in large populations, either by approximating interactions through aggregate population effects or by enabling decentralized learning in many-agent systems~\citep{yang2018mean}.

The Mean Field Game (MFG) framework, pioneered by~\citet{MR2295621} and~\citet{huang2006largeMKV}, circumvents this curse of dimensionality by invoking a continuum approximation. Rather than modeling pairwise interactions, MFGs focus on a representative agent interacting with the statistical distribution of the population. This formulation has proven efficient in domains requiring the modeling of large crowds, such as finance~\citep{cardaliaguet2018mean,carmona2017mean}, economics~\citep{MR2647032,achdou2022income}, and epidemiology~\cite{laguzet2015individual,olmez2022modeling,aurell2022optimal}. In the stationary setting, the problem reduces to finding a fixed point: a Nash Equilibrium where the agent's optimal policy induces a stationary distribution that, in turn, justifies the policy.

Historically, MFGs were solved using numerical methods for partial differential equations~\citep{MR2679575}. However, these grid-based methods struggle in high-dimensional state spaces. Consequently, the research community has pivoted toward learning-based approaches that leverage RL to solve MFGs without knowledge of the underlying model dynamics~\citep{subramanian2019reinforcementMFG,guo2019learningMFG,elie2020convergence,cui2021approximately}. Recent surveys~\citep{lauriere2022learning} highlight a surge in algorithms adapting well-known methods such as Fictitious Play (FP)~\citep{perrin2021meanflock,magnino2025solving}, Actor-Critic architectures~\citep{fu2019actorcriticMFG}, Q-learning~\citep{angiuli2020unifiedRLMFGMFC} and Online Mirror Descent~\citep{hadikhanloo2018learningphd,perolat2021scalingMFGOMD} to this setting.

\sethlcolor{salmon!20!white}
\hl{\textbf{Challenges.}}
 Despite this methodological progress, empirical evaluation in the MFG literature remains fragmented. Unlike reinforcement learning, for example, which benefits from standardized benchmarks like the Arcade Learning Environment~\cite{bellemare2013arcade}, MuJoCo~\cite{todorov2012mujoco}, OpenSpiel~\cite{lanctot2019openspiel}, BenchMARL~\cite{bettini2024benchmarl}, or standardized MARL evaluation protocols~\cite{gorsane2022towards}, \textbf{MFG algorithms are typically evaluated on ad-hoc, isolated examples}, 
 failing to capture the rich diversity of interaction types found in real-world applications. 
 Although some libraries exist~\citep{lanctot2019openspiel,guo2023mfglib}, they do not attempt to cover a wide taxonomy of MFGs. 
 Benchmarks have played a crucial role in driving progress within machine learning communities. 
 
 This work represents the first step toward a comprehensive framework for developing new advances in large multi-agent systems.
Our \textbf{contributions} are:
\vspace{-0.3cm}
\looseness=-1
\begin{mdframed}[
    linecolor=blue!20,      %
    linewidth=0pt,        %
    roundcorner=5pt,        %
    backgroundcolor=blue!5, 
    innertopmargin=10pt, 
    innerbottommargin=10pt
]
\begin{enumerate}[noitemsep,topsep=0pt, partopsep=1pt, parsep=0pt,leftmargin=*]
    \item \textbf{Taxonomy and Prototypical Examples:} We categorize MFGs into classes based on their mathematical structure (Sec. \ref{sec:classes}). For each, we provide open-source, prototypical environments that isolate specific algorithmic challenges (congestion, coupling, etc.).
    \item \textbf{Procedural Generation (MF-Garnets):} We introduce MF-Garnets (Sec. \ref{sec:garnet}). This generator constructs random MFG instances with \textit{controllable} density and interaction types (additive vs. multiplicative), enabling statistical stress-testing of solvers beyond fixed examples.
    \item \textbf{Systematic Evaluation:} We benchmark state-of-the-art solvers, including Fixed Point Iteration and Policy Iteration variants, across our suite, and we present the \textit{first population-based}, zeroth-order solver \textbf{(MF-PSO)}.
    \item \textbf{Infrastructure}: We provide a highly optimized infrastructure with both a pure Python implementation (using standard open-source methods) and a GPU-ready implementation in JAX, achieving 10–100× speedups. The code is configurable, easy to integrate and reproduce, and aims to accelerate progress within the MFG community.  We make our code available at \href{https://github.com/lorenzomagnino/Bench-MFG}{github.com/lorenzomagnino/Bench-MFG}.
\end{enumerate}
\end{mdframed}

\vspace{-0.2cm}
\section{Model}

\noindent\textbf{Notation.} If  $\mathcal L$ is a finite set, $\Delta_{\mathcal L}$ denotes the set of probability distributions on $\mathcal L$; $\unif(\mathcal{L})$ denotes the uniform distribution over $\mathcal{L}$. $P^\top$ denotes the transpose of a matrix $P$. $\mathbb N := \{0,1,2,\dots\}$.

\noindent\textbf{MFG model.} We focus on stationary models with discounted cumulative reward. A {\bf mean field game (MFG) model} is defined by a tuple $(\states, \actions, p, r, \gamma)$ where:
$\states, \actions$ are the finite state space and action space,
$p: \states \times \actions \times \Delta_{\states} \to \Delta_{\states}$ is the one-step transition probability kernel,
$r: \states \times \actions \times \Delta_{\states} \to \mathbb R$ is the one-step reward function,
$\gamma \in [0,1)$ is the discount factor. 
Instead of $p$, it is often convenient to write:
\(
    x_{n+1} = F(x_n, a_n, \mu_n, \epsilon_n), \quad \epsilon_n \sim \xi \in \Delta_{\noises} \text{ i.i.d.},
\)
where $F: \states \times \actions \times \Delta_{\states} \times \noises \to \states$ is a {\bf transition function}, $\xi$ is the noise distribution and $\noises$ is the noise space, 
so that $p(x'|x,a,\mu) = \xi(\{ e \in \noises : F(x,a,\mu,e) = x' \})$.

\noindent\textbf{Policies.} Let $\Pi = (\Delta_\actions)^{\states} = \Delta_\actions^{\states}$ denote the set of stationary policies. A policy $\pi \in \Pi$ maps a state $x \in \states$ to a distribution $\pi(\cdot|x) \in \Delta_\actions$.

\noindent\textbf{Objective function.} Assume the population distribution is stationary and given by $\mu \in \Delta_{\states}$. The  total \textit{discounted} reward for a player using policy $\pi \in \Pi$ facing $\mu$ is:
\vspace{-0.05cm}
\begin{equation}\label{eq: cost_def}
    J_{\gamma}(\pi;\mu) = \EE\left[ \sum_{n=0}^{\infty}\gamma^n r(x_n, a_n, \mu)\right], \quad     x_0 \sim \mu, \quad 
    a_n \sim \pi(\cdot|x_n), \quad 
    x_{n+1} \sim p(\cdot|x_n, a_n, \mu).
\end{equation}

\noindent\textbf{Distribution evolution.} For stationary policy $\pi$ and population distribution $\mu$, define the state-to-state transition matrix
\(
    P_{\mu,\pi}(x'|x) := \sum_{a\in\actions} p(x'|x,a,\mu) \, \pi(a|x),
\)
so that the distribution of the agent at time $n$ evolves as
\[
    \mu_{n+1}^{\mu,\pi} = P_{\mu,\pi}^\top \mu_n^{\mu,\pi}, \quad n \ge 0.
\]

We assume that for every stationary $\pi$, there exists a unique stationary $\mu$ solving $\mu = P_{\mu,\pi}^\top \mu$.

\noindent\textbf{Nash Equilibrium.} Define the {\bf best response} map
\(
    \BR: \Delta_{\states} \to 2^{\Pi}, 
    \BR(\mu) := \operatorname*{argmax}_{\pi \in \Pi} J_\gamma(\pi;\mu),
\)
and the {\bf population behavior} map
\(
    \mathbf M: \Pi \to 2^{\Delta_{\states}}, 
    \mathbf M(\pi) := \{\mu \in \Delta_{\states} \mid \mu = P_{\mu,\pi}^\top \mu\}.
\)
The solution concept we study is:
\begin{mdframed}[
    linewidth=0pt,        %
    roundcorner=5pt,        %
    backgroundcolor=blue!5, 
    innertopmargin=6pt, 
    innerbottommargin=2pt
]
\begin{definition}\label{def: NE}
A pair $(\pi^*, \mu^*) \in \Pi \times \Delta_{\states}$ is a {\bf mean field Nash equilibrium (MFNE)} if
$\pi^* \in \BR(\mu^*),$ and $\mu^* = \mathbf M(\pi^*).$
\end{definition}
\end{mdframed}
Equivalently, $(\pi^*, \mathbf M(\pi^*))$ is a MFNE if and only if $\exploitability(\pi^*)=0$, where $\exploitability$ is
the \textbf{exploitability} of a policy, quantifying the maximal gain achievable by deviating:
$
    \exploitability(\pi) = \max_{\pi' \in \Pi} J_\gamma(\pi'; \mu^\pi) - J_\gamma(\pi; \mu^\pi),
$ where $\mu^\pi = \mathbf M(\pi)$ is the stationary distribution induced by $\pi$.

\vspace{-0.2cm}
\section{Algorithms}
\vspace{-0.1cm}
This section presents different methods that aim to find the Nash equilibria in the setting presented above. We focus on three types of algorithms. First, \textbf{fixed point-based methods}, which have been used since the introduction of MFGs theoretically (to show existence of solutions) and numerically. Second, \textbf{policy iteration-based methods} which have attracted a growing interest. Last, since MFNE policies are characterized as minimizers of the exploitability, we also include a class of algorithms that attempt to directly minimize the exploitability; here we introduce \textbf{MF-PSO}, a black-box optimization method that adapts Particle Swarm Optimization (PSO) in the MFG setting. See \textbf{Table}~\ref{tab:mfg_algorithms}.

\begin{table}[t]
\centering
\small
\renewcommand{\arraystretch}{1.5}
\begin{tabular}{p{0.35\linewidth} p{0.6\linewidth}}
\hline
\textbf{Algorithm} & \textbf{Description} \\ \hline

\multicolumn{2}{c}{\cellcolor{darkgreen!5}\textbf{\color{darkgreen} BR-Based Fixed Point}} \\
\hline

\textbf{Damped Fixed Point} & 
Iteratively computes the best response (BR) $\pi_{k}$ against the current mean field (MF) $\mu_{k-1}$, then updates the mean field to $\mu_{k}$ by taking a weighted average with damping coefficient $\lambda_k$ of $\mu_{k-1}$ and the MF $\mu^{\pi_{k}}$ induced by $\pi_{k}$:
$\mu_{k} = (1-\lambda_k) \mu_{k-1} + \lambda_k\mu^{\pi_{k}}$. 
\\ \hline

\textbf{Fixed Point} & 
Damped FP with $\lambda_k \equiv 1$. 
\\ \hline

\textbf{Fictitious Play} & 
Damped FP with $\lambda_k = 1/(k+1)$. 
\\ \hline\hline

\multicolumn{2}{c}{\cellcolor{darkgreen!5}\textbf{\color{darkgreen} Policy-Eval.\ Based}} \\
\hline

\textbf{Policy Iteration (PI)} & 
Alternates between evaluating the current policy $\pi_{k-1}$ by computing a value function $Q_{k-1}$, and improving the policy by defining $\pi_{k}$ as a greedy policy for $Q_{k-1}$, then updating the MF $\mu_{k} = \mu^{\pi_{k}}$ assuming the population fully adopts this new policy. 
\\ \hline

\textbf{Smoothed PI (SPI)} & 
Similar to PI but updates the MF by averaging it with the previous MF (using either a constant weight $\lambda$ or $\lambda_k = 1/(k+1)$). 
\\ \hline

\textbf{Boltzmann PI (BPI)} & 
Similar to PI but updates the policy by taking a softmax instead of greedy. 
\\ \hline

\textbf{Online Mirror Descent (OMD)} & 
Similar to BPI but updates $\tilde Q_k$ by summing with previous $Q$-functions instead of using only the latest one. 
\\ \hline\hline

\multicolumn{2}{c}{\cellcolor{darkgreen!5}\textbf{\color{darkgreen} Sample-Based Methods (Novel!)}} \\
\hline

\textbf{$\star$ Mean Field Particle Swarm Optimization (MF-PSO)} & 
It works by maintaining a population of candidate policies that sample the search space. It evaluates the exploitability function at those sampled solutions and updates them following a velocity update rule (cognitive + social components).
\\ \hline

\end{tabular}
\vspace{0.09cm}
\caption{\small Summary MFG algorithms; see Appx.~\ref{app:algos} for details.}
\label{tab:mfg_algorithms}
\vspace{-0.8cm}
\end{table}

\vspace{-0.2cm}
\section{Classes of MFGs}\label{sec:classes}
\vspace{-0.1cm}
We now describe several classes of MFGs. For each class, we provide a general definition and some  prototypical examples that will be studied numerically in the next section.

\subsection{No-Interaction Games (NI-MFG)}

\begin{mdframed}[
    linewidth=0pt,        %
    roundcorner=5pt,        %
    backgroundcolor=blue!5, 
    innertopmargin=8pt, 
    innerbottommargin=6pt
]
\begin{definition}
    An MFG is a {\bf no-interaction (NI)} MFG if $p$ and $r$ are independent of $\mu$. 
\end{definition}
\end{mdframed}
\vspace{-0.1cm}
Since there are no interactions, solving an NI-MFG amounts to solving an MDP (any optimal policy is also an MFNE).
\begin{example}[\textbf{Move forward}]
    \label{ex:NI1}
    Let $\states = \{0,1,2,3,4,5,6\}$, $\actions = \{-1, 0, 1\}$, $\noises = \{-1,0,1\}$, $F(x,a,\mu,e) = x+a+e$, $\xi = \unif(\noises)$, $r(x, a, \mu) = -c |a| + x$ with $c=0.1$. This represents a one-dimensional grid world in which the reward increases as the agent moves towards states of larger index. 
\end{example}

\subsection{Contractive Games (C-MFG)}
While general MFGs require advanced algorithms, a subset of games admits a unique equilibrium that can be found via simple iterative updates. This occurs when the feedback loop between the population distribution and the agent's optimal policy is a contraction. This class serves as a ``sanity check'' for learning algorithms: solvers should be expected to converge rapidly in this regime.
\vspace{0.2cm}
\begin{mdframed}[
    linewidth=0pt,        %
    roundcorner=5pt,        %
    backgroundcolor=blue!5, 
    innertopmargin=6pt, 
    innerbottommargin=2pt
]
\begin{definition}[Contractive MFG]
\label{def:contractive_mfg}
    Assume that for every $\mu \in \Delta_{\states}$, the best response $\pi_\mu = \BR(\mu)$ is unique. Define the \textbf{fixed point operator} $\Gamma: \Delta_{\states} \to \Delta_{\states}$ as the composition of the best response and the population behavior map: $\Gamma(\mu) := \mathbf{M}(\BR(\mu))$.
    An MFG is \textbf{Contractive} if $\Gamma$ is a contraction with respect to the $L_1$-norm (twice the total variation distance). That is, there exists a constant $\kappa < 1$ such that for all $\mu, \nu \in \Delta_{\states}$:
    \(
        \|\Gamma(\mu) - \Gamma(\nu)\|_1 \le \kappa \|\mu - \nu\|_1.
    \)
\end{definition}
\end{mdframed}
By the Banach Fixed Point Theorem, a Contractive MFG admits a unique Mean Field Nash Equilibrium, and the sequence $\mu_{k+1} = \Gamma(\mu_k)$ converges linearly to it. However appealing this class is, it is unfortunately impossible to have interesting MFGs that are contractive when the spaces are discrete; see Lemma~\ref{lem:contraction-constant} in the appendix in our setting and \citep[Theorem 2]{cui2021approximately} in the finite-horizon setting. The key difficulty is related to the continuity of the BR map w.r.t. $\mu$.

{ For benchmarking, we provide a contractive example (see Appx.~\ref{sec:contractive-ex-proof}).}
\begin{example}[\textbf{Coordination game}]
\label{ex:contractive-example}
    Consider $\states=\{0,1\}$ and $\actions=\{\text{Stay},\text{Switch}\}$. Let $f(x,\text{Stay})=x$ and $f(x,\text{Switch})=1-x$. With probability $1-\eta$ execute this action, and with probability $\eta$ reset to $q=(1/2,1/2)$:
    $
        P_\eta(x'\mid x,a)=(1-\eta)\mathbf1\{x'=f(x,a)\}+\eta q(x')
    $, ($\eta=0.1$).
    The agent incurs a high cost $C > 0$ for switching states, and a smaller congestion penalty scaled by $\alpha > 0$.
    \(
        r(x, a, \mu) = -C \cdot \mathbb{I}(a=\text{Switch}) - \alpha \mu(x),
    \)
    where $\mathbb{I}$ denotes the indicator function. 
If the switching cost $C$ dominates the maximum possible gain from avoiding congestion, the unique best response is to always ``Stay''. More precisely,
    if the parameters satisfy the condition:
    \(
        C > \frac{\alpha}{1-\gamma},
    \)
    then the Best Response is unique and constant: $\pi^*(a=\text{Stay}|x) = 1$ for all $\mu$. { The unique stationary population is $q$, so $\Gamma(\mu)=q$ is a contraction with coefficient zero (see Prop.~\ref{prop:contractive-ex-property} in the appendix).}
\end{example}

\subsection{Lasry-Lions Monotone Game (LL-MFG)}
Introduced in \citep{MR2295621}, the Lasry-Lions monotonicity condition has become a key assumption in MFG theory, particularly to show uniqueness of equilibrium or to prove convergence of algorithms, see e.g.~\citep{perrin2020fictitious,perolat2021scalingMFGOMD}.
The underlying principle is that the reward discourages agents from being in crowded areas, i.e., states $x$ where $\mu(x)$ is large. 
\begin{mdframed}[
    linewidth=0pt,        %
    roundcorner=5pt,        %
    backgroundcolor=blue!5, 
    innertopmargin=6pt, 
    innerbottommargin=0pt
]
\begin{definition}%
\label{def:ll_mono}
    Assume that the reward is separable, i.e. there exist two functions $\psi:\mathcal X \times \mathcal A\to \mathbb R$,  $g:\mathcal X \times \Delta_{\mathcal X}\to \mathbb R$ such that the reward can be written as $r(x,a,\mu)=\psi(x,a) + g(x, \mu)$. Then, the (resp. strict) \textbf{Lasry-Lions (LL) monotonicity condition} is: for any $\mu, \nu \in \Delta(\mathcal X)$, 
    $\sum_{x \in \mathcal X}\left(g(x, \mu) - g(x, \nu)\right)(\mu(x) - \nu(x)) \leq 0$ (resp. $<0$ whenever $\mu \neq \nu$).
\end{definition}

\begin{definition}[LL-MFG]
    An MFG is \textbf{Lasry-Lions (LL) monotone}  if the Lasry-Lions monotonicity condition in Def.~\ref{def:ll_mono} is satisfied.
\end{definition}\label{def:llm_game}
\end{mdframed}
\begin{example}[\textbf{Beach Bar Problem}]
    \label{ex:beach-bar}
    Let $\states, \actions, \noises$ be as in Ex.~\ref{ex:NI1}.
    { Let $\xi(0) = \zeta$ and $\xi(-1)=\xi(+1)=(1-\zeta)/2$, with $\zeta\in[0,1]$.}
    Let the reward be: $r(x,a,  \mu) = -c_1|a| - {c_2|x - x_{center}|^2} - \alpha \mu(x),$ where $x_{center} = 3$, and $c_1, c_2, \alpha$ are positive weights. This models a congestion cost where the agents want to reach the center but are penalized when moving and for being in a crowded state (with large value of $\mu$).
\end{example}
We now provide an example that does not satisfy such a condition.
\begin{example}[\textbf{Two–Beach Bar Coordination MFG}]
\label{ex:two-beach-bar}
Let $\states, \actions, \noises$ and the noise distribution $\xi$ be as in Example~\ref{ex:beach-bar}.  
Fix two symmetric target locations
$x_L = 2$ and $x_R = 4.$ Let $c_1,c_2,\alpha > 0$ and define the one–step reward by
\(
r(x,a,\mu)
=
- c_1 |a|
- c_2 \min\bigl\{|x-x_L|,\;|x-x_R|\bigr\}
+ \alpha \mu(x).
\)
The first term penalizes motion, the second term models attraction to the two beach bars, and the last term is a positive mean–field interaction favoring crowded states.  
\textbf{Note} that this game \emph{does not} satisfy the LL monotonicity condition because 
with $g(x,\mu)=\alpha\mu(x)$, $\sum_x (g(x,\mu) - g(x,\nu))(\mu(x)-\nu(x)) = \alpha\sum_x (\mu(x)-\nu(x))^2 > 0$ whenever $\mu \neq \nu$.
Furthermore, for parameters satisfying $\alpha \gg c_2 \gg c_1$, the model admits exactly two stationary mean field Nash equilibria, with population fully concentrated at one of the two bars.
\end{example}

\subsection{Potential Games  (P-MFG)}
Originally introduced by \citet{monderer1996potential} in relation to strategic form games, the concept of potential games has been studied in the continuous-time, continuous-space MFG framework by \citet{MR2295621} and \citet{cardaliaguet2017learning}, as well as in the discrete setting with discounted reward by~\citet{geist2022concave}.
\begin{mdframed}[
    linewidth=0pt,        %
    roundcorner=5pt,        %
    backgroundcolor=blue!5, 
    innertopmargin=6pt, 
    innerbottommargin=2pt
]
\begin{definition}[P-MFG]
\label{def: potential_game}
    An MFG $(\mathcal X, \mathcal A,p, r, \gamma)$ is a \textbf{Potential} (P) MFG if the reward function $r:\mathcal X \times \mathcal A \times \Delta_{\mathcal X}\to \R$ can be derived from a potential, i.e., there exists a differentiable map $\mathfrak{P}: \mathcal A \times \Delta_{\mathcal X}\to \mathbb R$ such that: for all  $x \in \mathcal X$, $a \in \mathcal A,$ $\mu \in \Delta_{\mathcal X}$, $r(x, a, \mu) = \nabla \mathfrak{P}(a, \mu)(x)$. 
\end{definition}
\end{mdframed}
\begin{proposition}%
\label{prop: potential_characterization}
    Consider a separable reward function $r(x, a, \mu) = \psi(x, a) + g(x, \mu)$, where $g: \states \times \Delta_{\states} \to \mathbb R$ is continuously differentiable with respect to $\mu$. {The game is a potential MFG if and only if the Jacobian is symmetric along simplex tangent directions:
    \[
        u^\top\bigl(D_\mu g(\mu)-D_\mu g(\mu)^\top\bigr)v=0,
        \qquad \mathbf{1}^\top u=\mathbf{1}^\top v=0.
    \]
    Here derivatives are taken on the relative interior and extended continuously to the boundary.}
\end{proposition}
The proof is provided in Appx. \ref{proof: potential_characterization}. 
\begin{example}[\textbf{4 rooms exploration}]
    \label{ex:P-MFG1}
    Let $\states = \{0, 1, \ldots, 10\}^2$ represent a $11 \times 11$ grid world, $\actions = \{(0,1), (1,0), (0, -1), (-1, 0), (0,0)\}$ corresponding to up, right, down, left, and stay, respectively, $\noises = \actions $ with uniform distribution $\xi = \unif(\noises)$, and $F(x, a, \mu, e) = x + a + e$ with boundary conditions and obstacle constraints (cross walls at row 5 and column 5, with doors at $(2,5)$, $(7,5)$, $(5,7)$, $(5,2)$). The reward is $r(x, a, \mu) = -\alpha \log(\mu(x))$ with $\alpha > 0$, \textit{modeling an aversion to crowded states where agents are penalized for being in locations with high population density.}
\end{example}

{\bf Cyclic Games.} The characterization in Prop.~\ref{prop: potential_characterization} allows us to construct a class of MFGs that are provably \textit{\textbf{non}-potential}. 
Consider a linear population reward given by an interaction matrix $A \in \mathbb R^{|\states| \times |\states|}$:
$
g(x, \mu) = [A \mu]_x = \sum_{y \in \states} A_{xy} \, \mu(y).
$
{The Jacobian of $g$ with respect to $\mu$ is $A$. Thus, the MFG is non-potential exactly when $u^\top(A-A^\top)v\neq0$ for some zero-sum vectors $u,v$. A particular subclass consists of \textit{Cyclic Games}, generated by skew-symmetric matrices $A^\top=-A$ satisfying this tangent condition. Such interactions can produce rotational learning dynamics.}
\begin{example}[\textbf{Rock, Paper, Scissors}]
    \label{ex: RPS}
    Let $\states = \{1, 2, 3\}$ representing states Rock, Paper, and Scissors respectively. The interaction matrix is given by:
    \(
        A = \scriptsize\begin{bmatrix}
             0 & -1 &  1 \\
             1 &  0 & -1 \\
            -1 &  1 &  0
        \end{bmatrix}.
    \)
    Here, $A_{xy} = 1$ implies state $x$ beats state $y$, and $A_{xy} = -1$ implies $x$ loses to $y$. The reward for being in state $x$ depends on the dominance over the current population distribution:
    \(
        g(x, \mu) = \mu(\text{prey of } x) - \mu(\text{predator of } x).
    \)
    {For $u=(1,-1,0)^\top$ and $v=(0,1,-1)^\top$, $u^\top(A-A^\top)v=-6\neq0$.} By Prop.~\ref{prop: potential_characterization}, this MFG is \emph{not a potential game}.
\end{example}

\subsection{Dynamics-Coupled Games (DC-MFG)}
\vspace{0.2cm}
\begin{mdframed}[
    linewidth=0pt,        %
    roundcorner=5pt,        %
    backgroundcolor=blue!5, 
    innertopmargin=6pt, 
    innerbottommargin=2pt
]
\begin{definition}[DC-MFG]\label{def: dc-mfg}
    An MFG is said to be \textbf{Dynamics-Coupled} (DC) if the transition probability kernel $p$ depends explicitly on the population distribution $\mu$, while the reward function $r$ may or may not depend on $\mu$. That is, for $x, x' \in \states$ and $a \in \actions$:
    \(
        \frac{\partial p(x'|x, a, \mu)}{\partial \mu} \neq 0.
    \)
\end{definition}
\end{mdframed}
This type of interaction has thus far been less considered in the MFG literature than interactions through the reward function. Two \textbf{important difficulties} are that {\bf (1)} by definition, DC-MFGs are not LL-MFGs (Def.~\ref{def:ll_mono}) since the latter assumes that the dynamics do not include interactions, and {\bf (2)} during the learning process, the mean field changes and this implies that the transition probabilities of the player's MDP change. It makes it particularly interesting to check the performance of learning algorithms on DC-MFGs. 

The following example is inspired by \citep{cui2021approximately}.
\begin{example}[\textbf{SIS Epidemic}]
    \label{ex:dc1}
    Let $\states = \{0, 1\}$ representing Susceptible and Infected states, respectively, $\actions = \{0, 1, \ldots, 4\}$ {indexing social interaction intensities $u_a=a/4\in[0,1]$}, and dynamics $F(x, a, \mu, e)$ where: if $x = 0$ (Susceptible), then $p(1|x, a, \mu) = {\min\{1,\max\{0,\beta u_a\mu(1)\}\}}$ (infection probability depends on action intensity and infected population prevalence), and if $x = 1$ (Infected), then $p(0|x, a, \mu) = \nu$ (recovery probability). The reward is $r(x, a, \mu) = {u_a} - C \cdot \mathbf{1}(x = 1)$, where $C > 0$ is the cost of being infected. This models a dynamics-coupled MFG where agents balance social activity benefits against infection risks that depend on the population distribution. {The implementation approximates these probabilities using 100 uniformly weighted midpoint noise values.}
\end{example}
\vspace{0.05cm}
\begin{example}[\textbf{Kinetic Congestion}]
    \label{ex:dc2}
    Let $\states = \{0, 1, \ldots, 4\}^2$ representing a $5 \times 5$ grid world, $\actions = \{(-1,0), (1,0), (0,-1), (0,1), (0,0)\}$ corresponding to up, down, left, right, and stay, respectively. Let the dynamics be $F(x, a, \mu, e)$ where movement success depends on the population density at the target state: $p(y|x, a, \mu) = 1 - \min(1, \phi(\mu(y)))$ if $y = x + a$ is the target state, where $\phi(\mu(y)) = \mu(y) / \tau$ if $\mu(y) < \tau$ and $\phi(\mu(y)) = 1$ otherwise, with capacity threshold $\tau > 0$. The reward is $r(x, a, \mu) = -\mathbf{1}(x \neq x_{\text{target}}) - c_{\text{move}} \cdot \mathbf{1}(a \neq (0,0))$, where $x_{\text{target}} = (4,4)$ is the target state and $c_{\text{move}} > 0$ is the movement cost. This represents a DC-MFG where high population density physically prevents movement, mimicking kinetic congestion or crowd dynamics.
\end{example}
\textbf{Remark}.  Ex.~\ref{ex:dc2} includes a kind of congestion effect through the dynamics. In contrast,  Ex.~\ref{ex:beach-bar} includes a form of congestion cost through the reward function. While we can expect that the population will spread through the state space in both cases, the effects will be quantitatively different.
\vspace{-3mm}
\section{Procedural Generation (MF-Garnets)}\label{sec:garnet}
\vspace{-2mm}
To evaluate the \textbf{robustness of numerical methods, we present a procedure to generate synthetic MFGs}, extending the methodology of Garnet (which stands for Generic Average Reward Non-stationary Environment Testbench)~\citep{archibald1995generation,bhatnagar2009natural}. The \textbf{key difference} is that we need to define \emph{random functions of the mean field}, and the space of such functions cannot be represented using a finite number of parameters. So we consider a few canonical choices, but many variants could be considered. 

\begin{mdframed}[
    linewidth=0pt,        %
    roundcorner=5pt,        %
    backgroundcolor=blue!5, 
    innertopmargin=10pt, 
    innerbottommargin=10pt
]
Generally speaking,
an \textbf{MF-Garnet} instance is defined by a tuple $(|\states|, |\actions|, \cP, \cR, \mathfrak{T})$, where $\cP$ is a parameterized class of transition probabilities, $\cR$ is a parameterized class of reward functions, and $\mathfrak{T}$ is a probability distribution over the transition and reward parameters.
\end{mdframed}

\noindent\textbf{Randomized Dynamics Coupling.} 
Let $P^0 \in [0,1]^{|\states| \times |\actions| \times |\states|}$ be a base transition tensor, generally taken to be relatively sparse. We sample random coupling coefficients $\Omega \in \mathbb{R}^{|\states| \times |\actions| \times |\states| \times |\states|}$ from $\mathcal{N}(0, 1)$, and scalar parameters $c_p, \rho_p \sim \unif[0, 1]$. 
We define the (unnormalized) population-dependent intensity $\tilde{P}(x'|x, a, \mu)$ using one of two structures and then normalize it:
\begin{itemize}
\item \noindent\textit{Additive Intensity:}
The mean field adds bias to the base transitions. Letting $\tilde P^0(x'|x,a,\mu) = c_p P^0(x'|x,a) + \rho_p \sum_{y \in \states} \Omega(x, a, x', y) \mu(y)$, define
\(
    \tilde{P}(x'|x, a, \mu) = \operatorname{ReLU}\big( \tilde P^0(x'|x,a,\mu) \big),
\)
where $\operatorname{ReLU}(z) = \max(0, z)$.

\item \noindent\textit{Multiplicative Intensity:}
The mean field acts as a gate, scaling the base transitions. Letting $\tilde \Omega(x,a,\mu,x') = c_p + \rho_p \sum_{y \in \states} \Omega(x, a, x', y) \mu(y)$, define
\(
    \tilde{P}(x'|x, a, \mu) = P^0(x'|x,a) \times \operatorname{ReLU}\big( \tilde \Omega(x,a,\mu,x')\big).
\)
\end{itemize}
{In both cases, the nonnegative intensities are normalized and sampled on a finite noise grid. The exact normalization and sampling rule is given in Appx.~\ref{app:garnet-transitions}.}

\noindent\textbf{Randomized Reward Coupling.}
Similarly, let $R^0 \in \mathbb{R}^{|\states| \times |\actions|}$ be the base reward. We sample an interaction matrix $M \in \mathbb{R}^{|\states| \times |\states|}$ and scalars $c_r, \rho_r$. We consider two types of structures:
\begin{itemize}
    \item \textit{Additive:}
    \(
        r(x, a, \mu) = c_r R^0(x, a) + \rho_r \sum_{y} M_{xy} \mu(y).
    \)

    \item \textit{Multiplicative:}
    \(
        r(x, a, \mu) = R^0(x, a) \times \big( c_r + \rho_r \sum_{y} M_{xy} \mu(y) \big).
    \)
\end{itemize}
So there are two models of rewards (additive or multiplicative) and in both cases they are parameterized by $(R^0, M, c_r, \rho_r) \in \mathbb{R}^{|\states| \times |\actions|} \times \mathbb{R}^{|\states| \times |\states|} \times \mathbb{R} \times \mathbb{R}$.
Intuitively, the multiplicative forms are particularly useful for modeling \textit{shocks} where the mean field amplifies or dampens the intrinsic transition probabilities or rewards of a state.

\sethlcolor{yellow!20!white}
\section{Experiments}
\begin{figure}[t]
    \centering
    \includegraphics[width=0.325\linewidth]{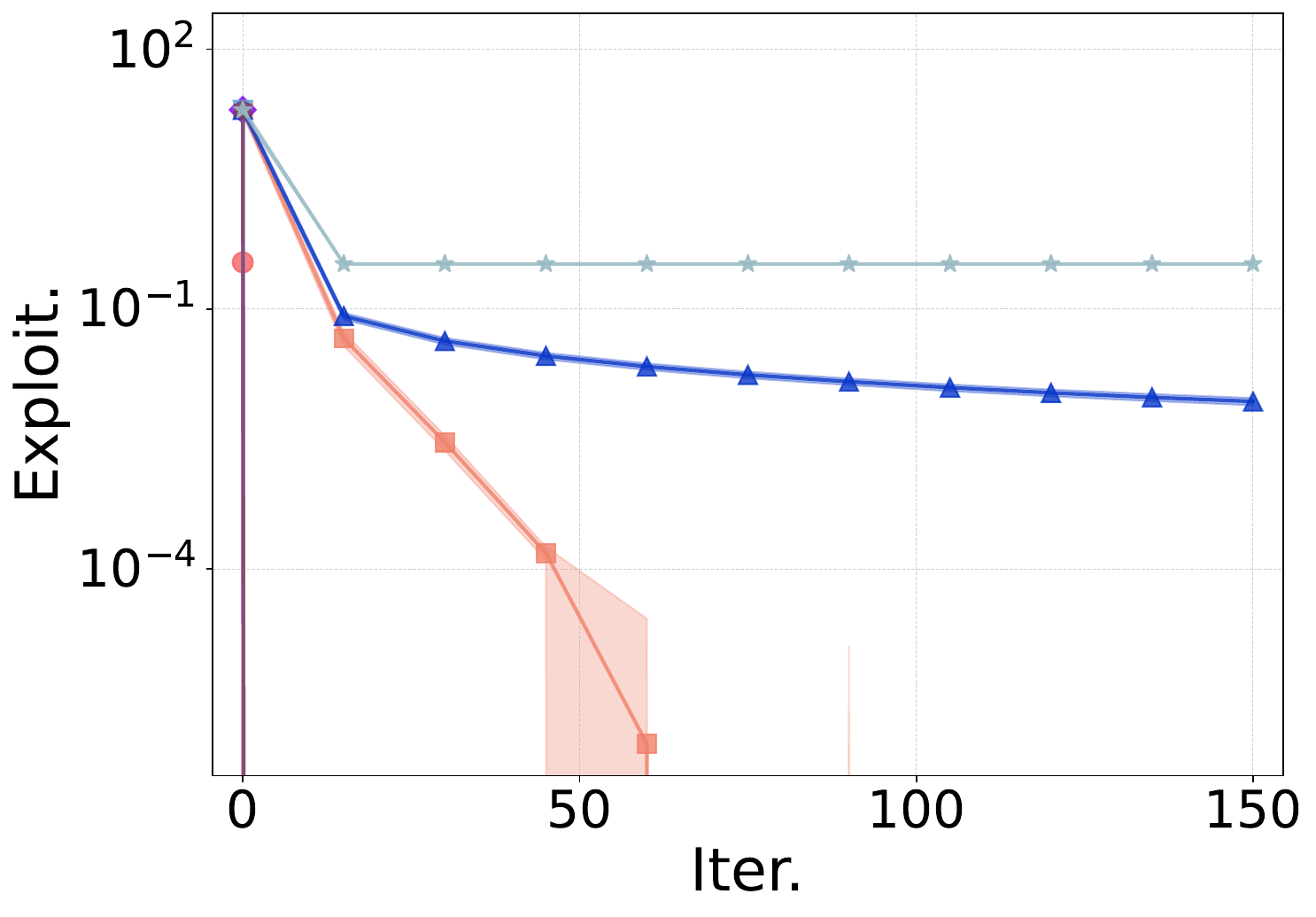}
    \centering
    \includegraphics[width=0.325\linewidth]{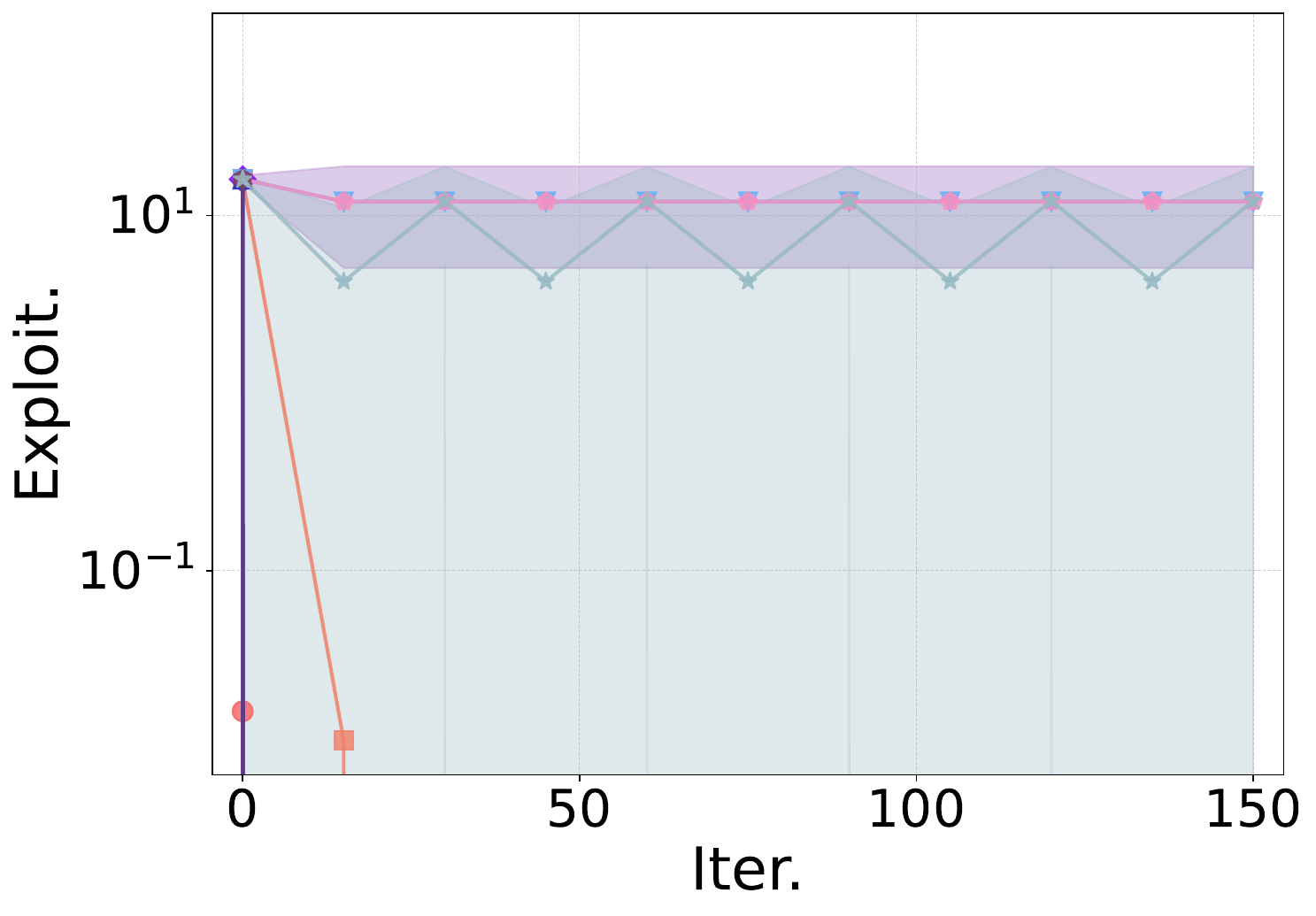}
    \includegraphics[width=0.325\linewidth]{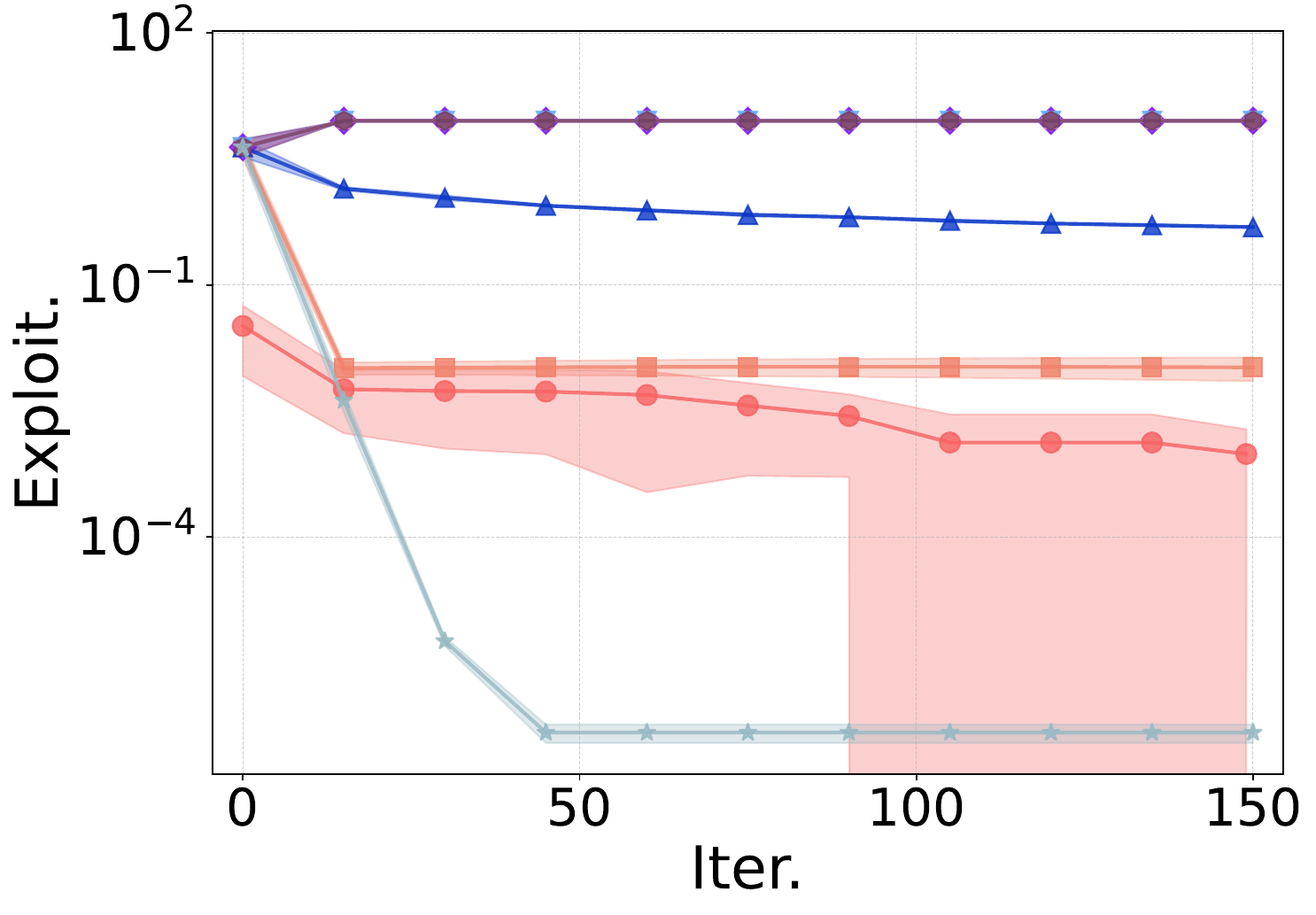}
    \includegraphics[width=0.32\linewidth]{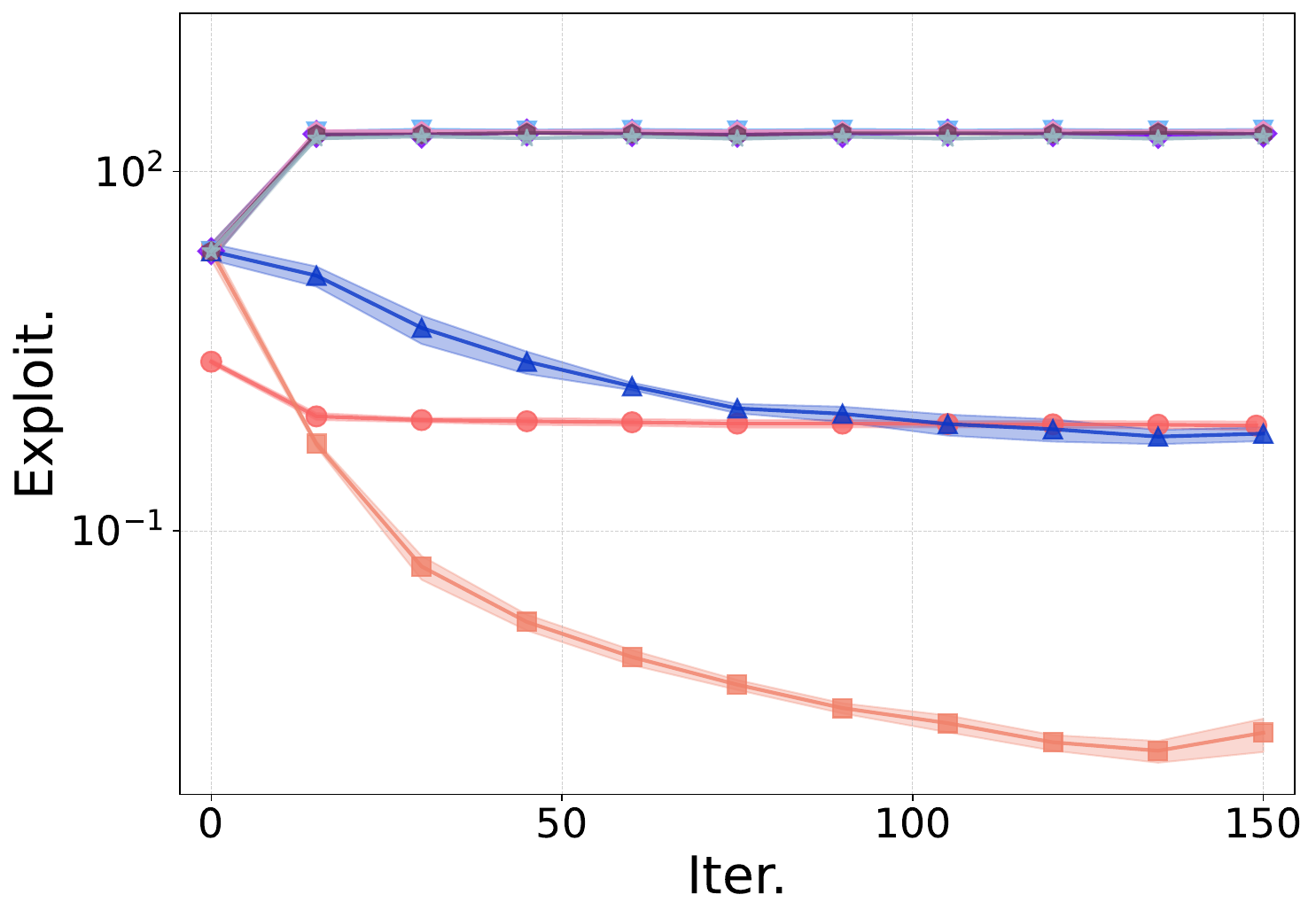}
    \centering
    \includegraphics[width=0.29\linewidth]{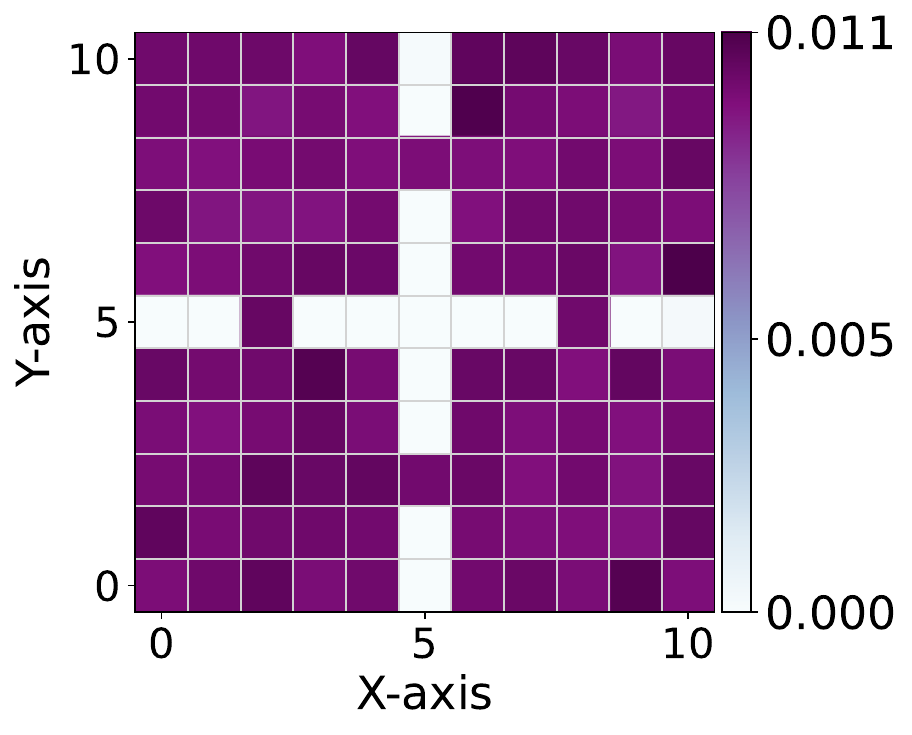}
    \includegraphics[width=0.32\linewidth]{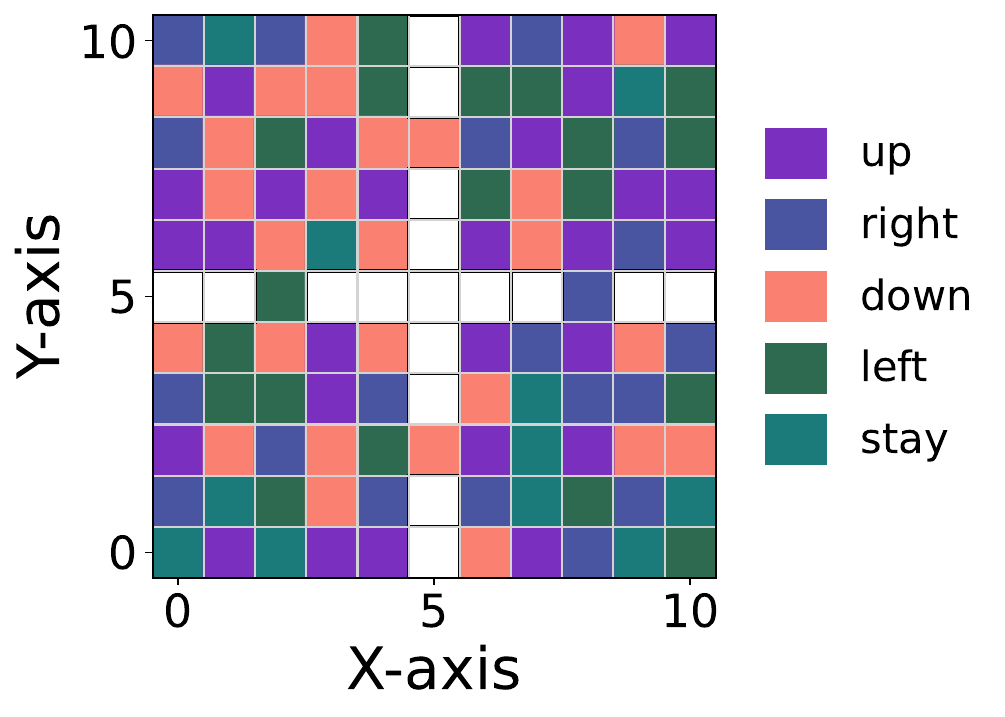}
    \vspace{1mm}
    \algolegend
    \caption{\small (Top) [Left] \textbf{NI-MFG} \textit{Move Forward}; [Middle] \textbf{C-MFG} \textit{Coordination Game}; [Right] \textbf{Cyclic Game} \textit{Rock, Paper, Scissors}. Different mathematical properties of the problem result in very \textit{different learning}. (Bottom) \textbf{P-MFG} \textit{4Rooms Exploration}, $\alpha=1$: [Left] Exploitabilities; [Middle, Right] Equilibrium $(\mu^*, \pi^*)$ (OMD). The crowd-aversion reward causes the population to spread out uniformly over the four rooms. }
    \label{fig_app: NI-MFG/C-MFG/RPS}
    \vspace{-2mm}
\end{figure}
\paragraph{Experimental setup}
The main benchmark experiments were implemented using JAX~\citep{jax2018github}, which provides significant computational speedups through just-in-time (JIT) compilation and efficient automatic differentiation. We observed speedups of approximately $\boldsymbol{10-100\times}$ compared to classical Python open-source implementations (the main benchmark and MF-Garnet scaling experiments were conducted on an NVIDIA L40 GPU; the particle-count ablation uses the hardware specified in Appx.~\ref{app:pso-scalability}). While our results use the JAX implementation for efficiency, we also provide the pure Python implementation for reference and reproducibility.
To ensure fair comparison across algorithms and environments, we initialize all methods from the same initial policy for each environment. Specifically, we use a uniform initialization scheme that generates initial policies from logits sampled uniformly in $[-1, 1]$, which are then converted to probability distributions via a Boltzmann policy with temperature $0.2$. For the MF-PSO algorithm, we employ $200$ particles unless otherwise stated in the particle-count ablation (Appx.~\ref{app:pso-scalability}).
\begin{mdframed}[
    linecolor=purple!70!blue,      %
    linewidth=2pt,        %
    roundcorner=5pt,        %
    innertopmargin=8pt, 
    innerbottommargin=6pt
]

\noindent\textbf{Evaluation and Visualization.} We evaluate the algorithms using four types of visualizations: {\bf (1)} \textbf{Exploitability} curves comparing all algorithms, where for each algorithm we select the best hyperparameter configuration based on the average final exploitability across all $4$ seeds; {\bf (2)} \textbf{Equilibrium} distribution and policy \textbf{visualizations} for the best model (See Figs.~\ref{fig_app: NI-MFG/C-MFG/RPS} - \ref{fig_app: SIS-pi_mu}) and {\bf (3)} \textbf{Hyperparameter} \textbf{sweep} results shown in the appendix (see \ref{app: sweep}), displaying the full grid search results for each algorithm and environment combination (see hyperparameter space in \ref{app: sweep}). {\bf (4)} \textbf{Wall-clock runtime} of each algorithm for each class of MFGs (see Figs. \ref{fig:wall-clock1} - \ref{fig:wall-clock4}). 
\end{mdframed}

\begin{figure}[h!]
    \centering
    \includegraphics[width=0.3\linewidth]{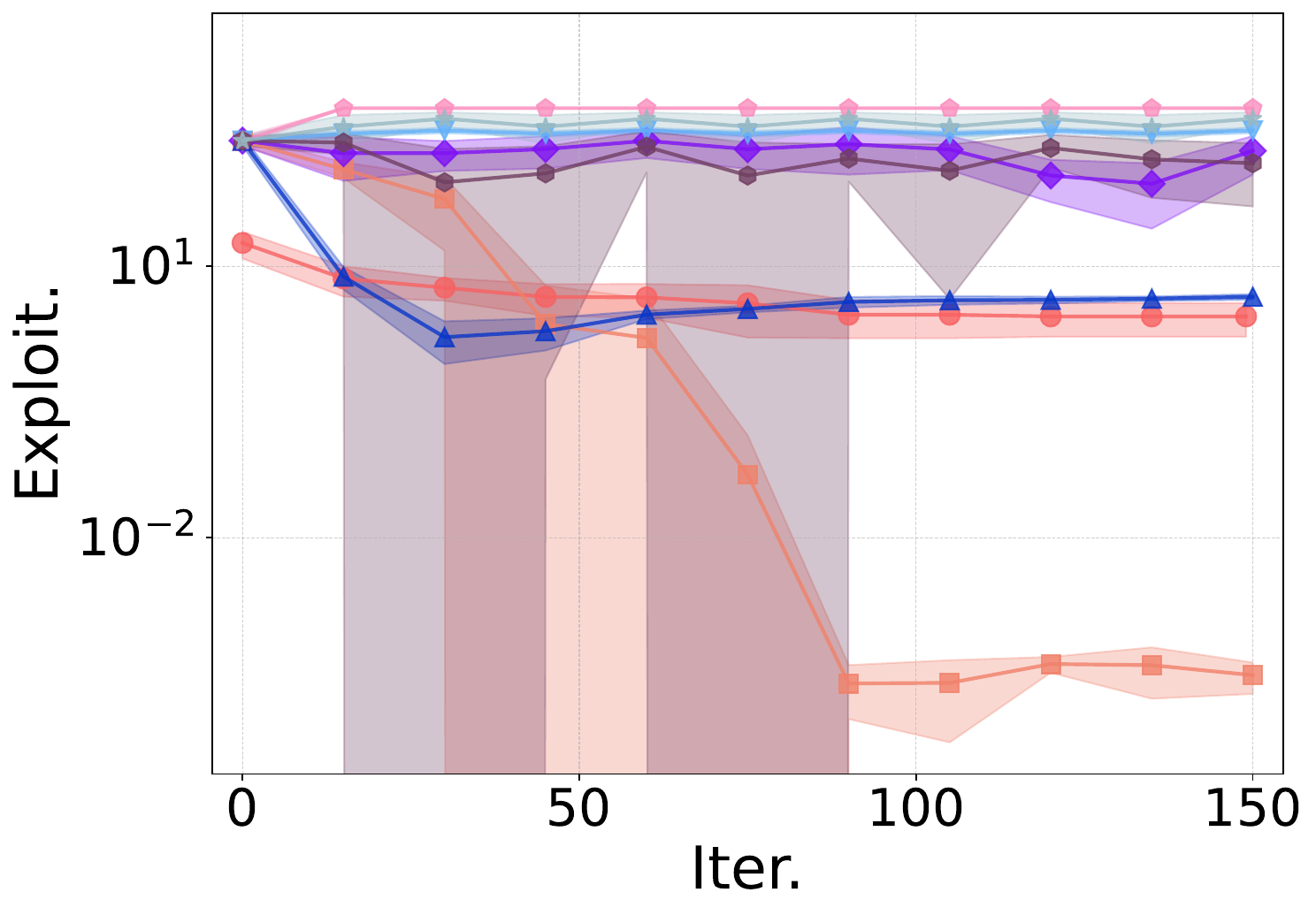}
    \centering
    \includegraphics[width=0.29\linewidth]{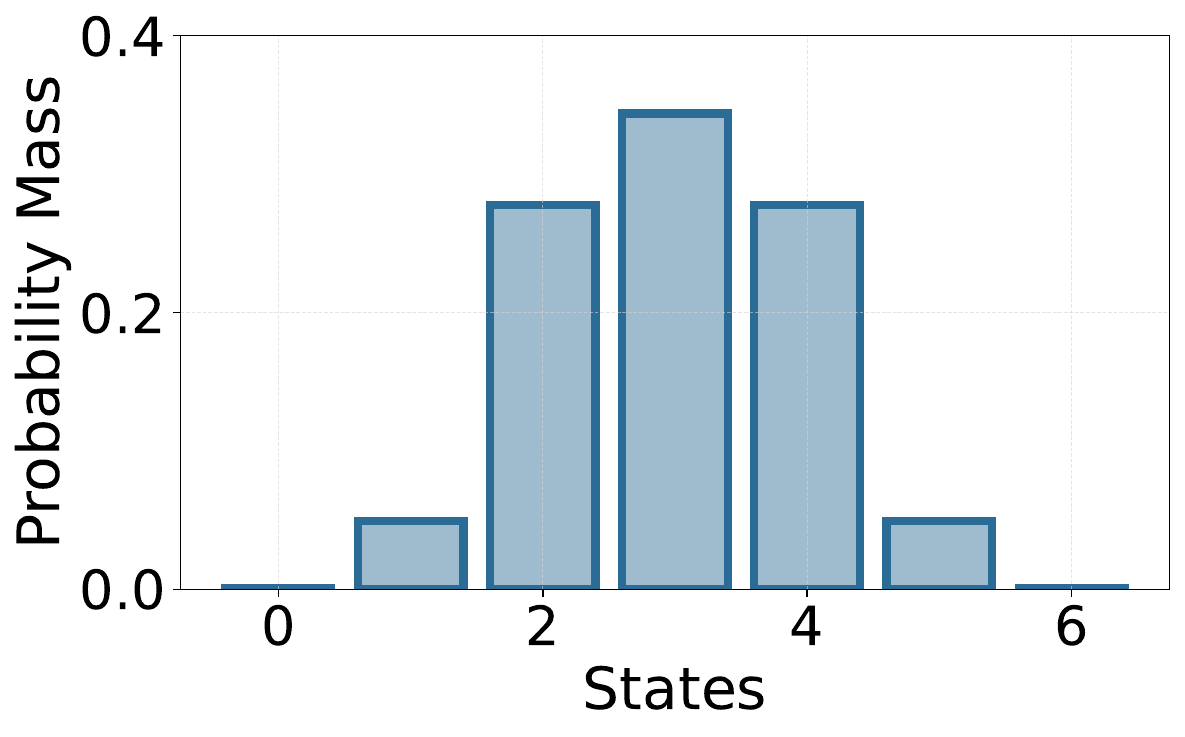}
    \includegraphics[width=0.29\linewidth]{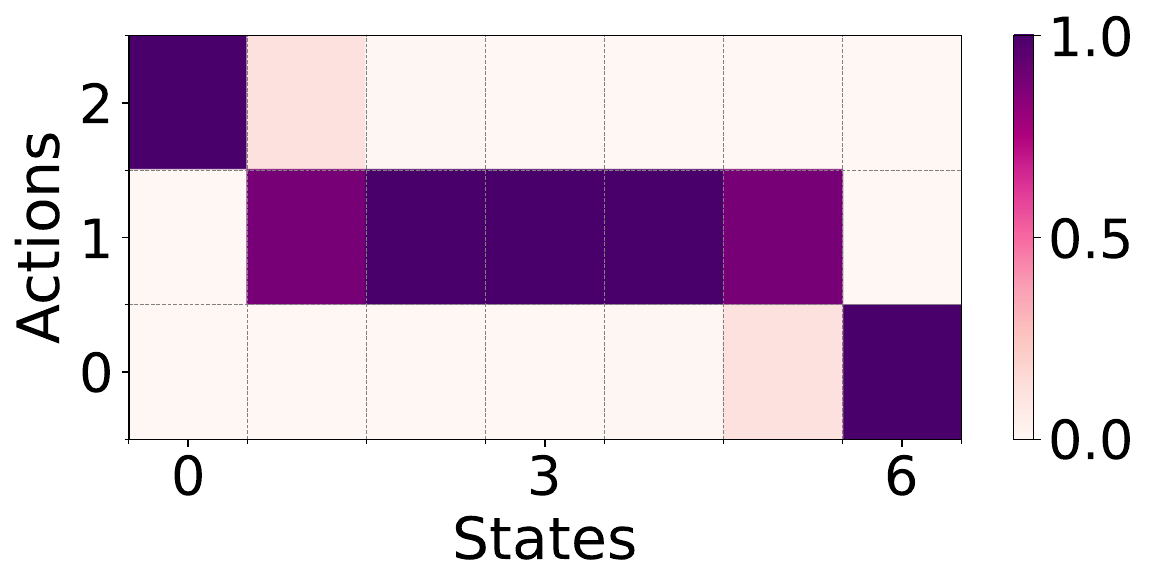}
        \centering
    \includegraphics[width=0.3\linewidth]{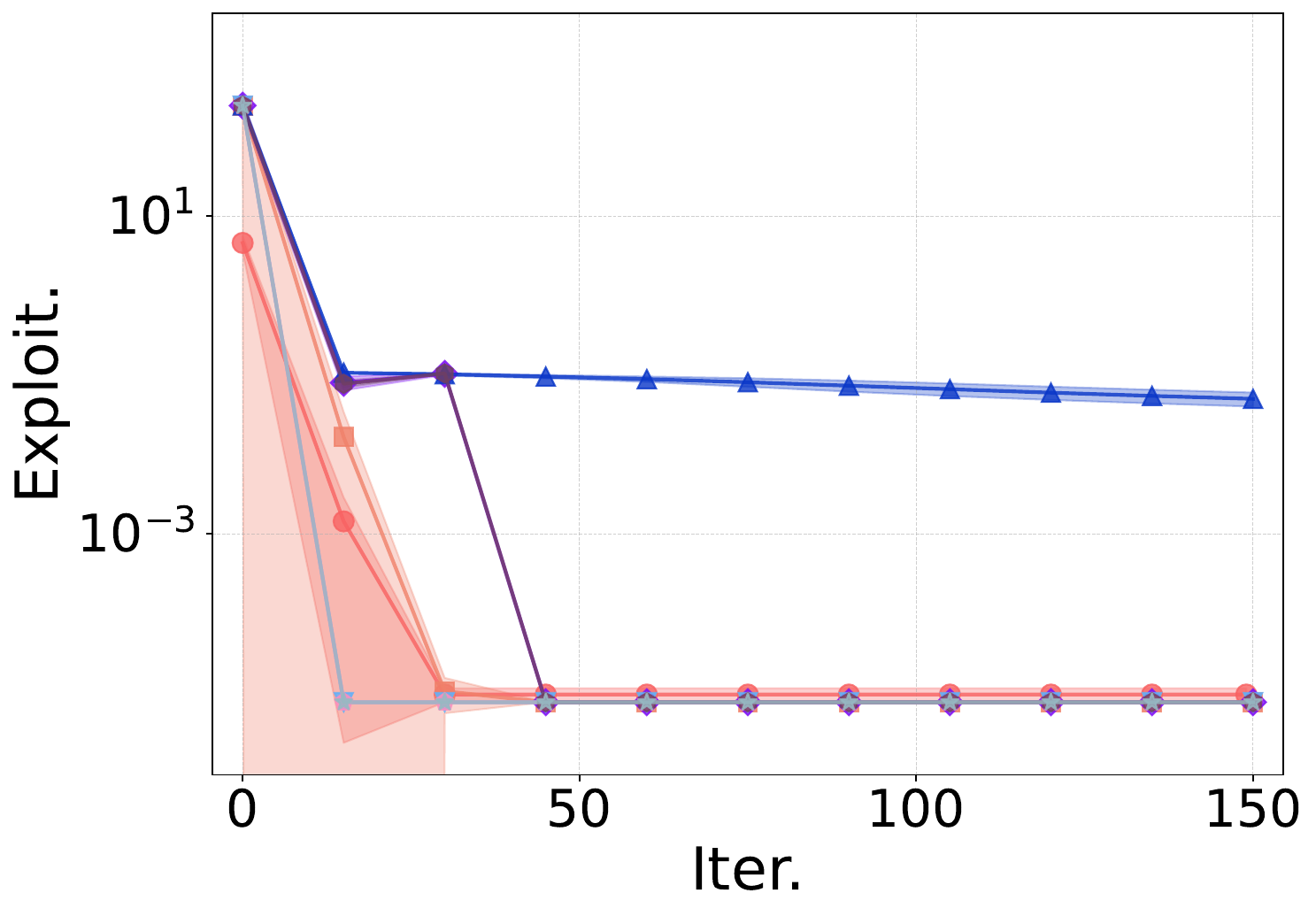}
    \includegraphics[width=0.29\linewidth]{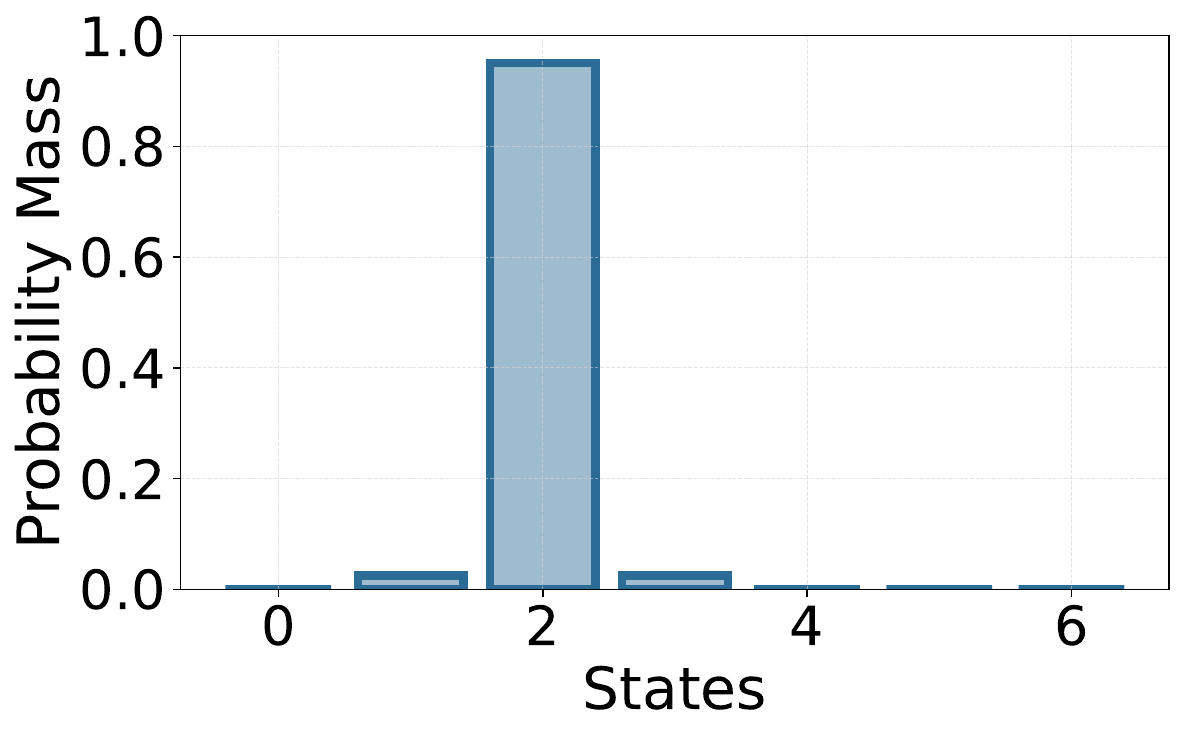}
    \includegraphics[width=0.29\linewidth]{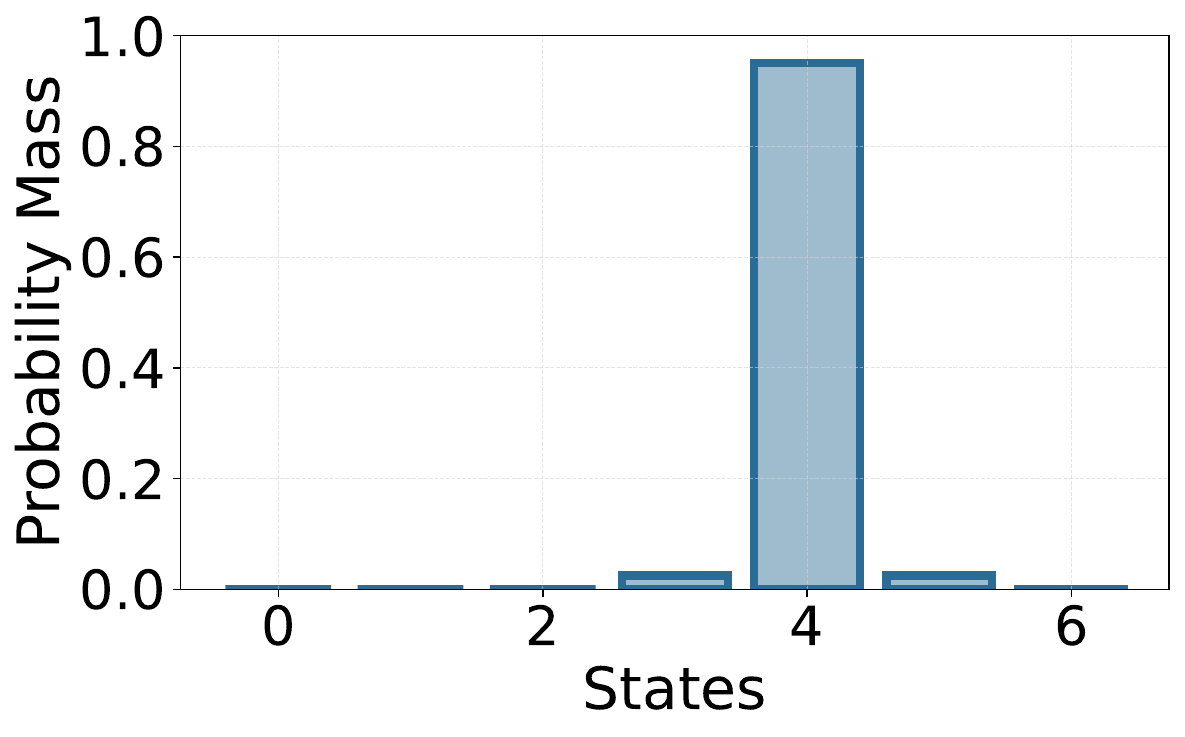}
    \vspace{1mm}
    \algolegend
    \caption{\small \textbf{LL-MFG}. (Top) \textit{Beach Bar Problem} $\alpha=5, c_2=5, c_1=2$.  Exploitabilities and Equilibrium $(\mu^*, \pi^*)$. (Bottom) \textit{Two Beach Bars Problem} $\alpha=60, c_2=15, c_1=0.5$. Note how OMD (left) and PSO (right) found \textit{different equilibria}.}
    \label{fig_app: LL-pi_mu}
\end{figure}

\begin{figure}[h!]
    \centering
    \includegraphics[width=0.3\linewidth]{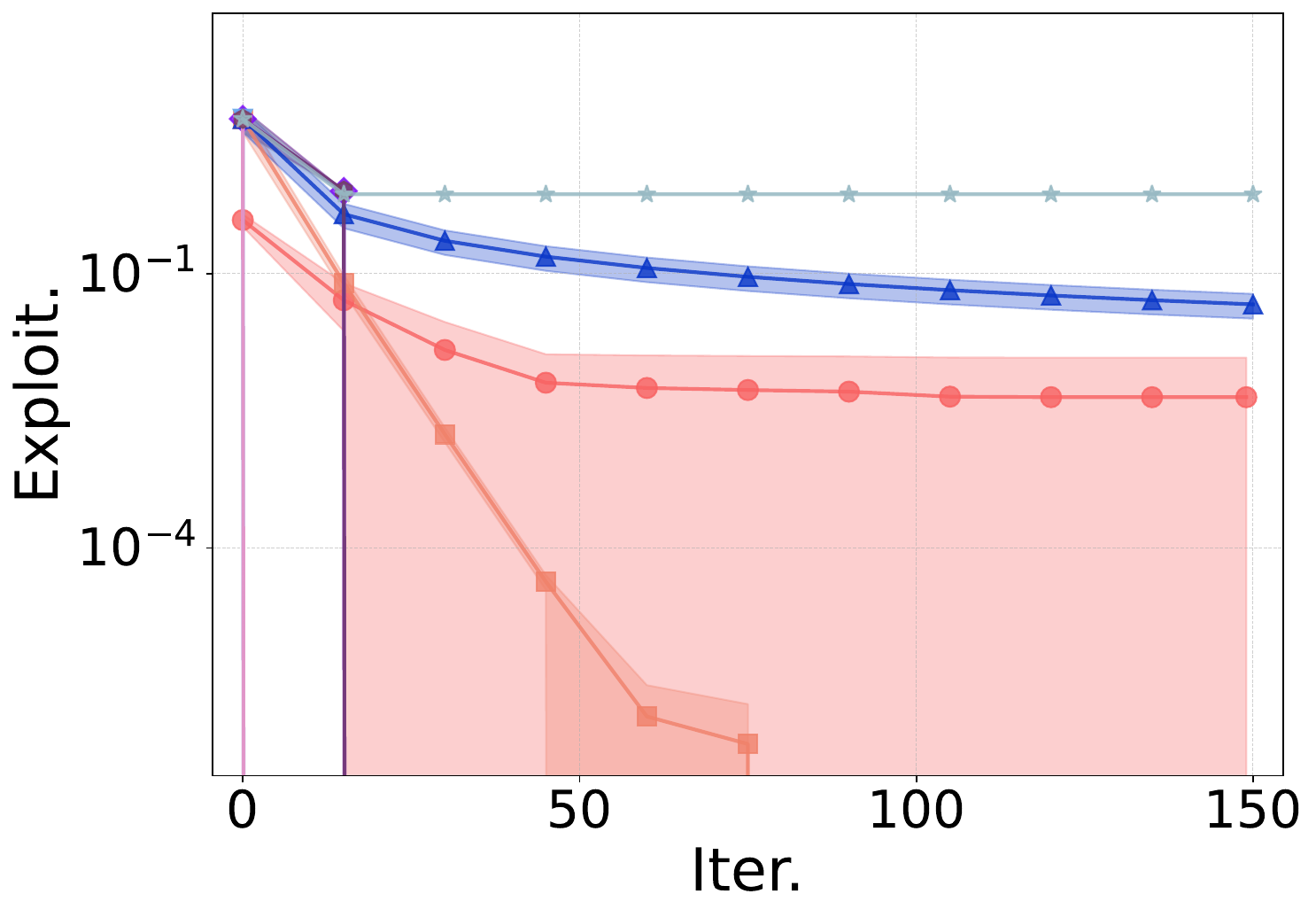}
    \centering
    \includegraphics[width=0.29\linewidth]{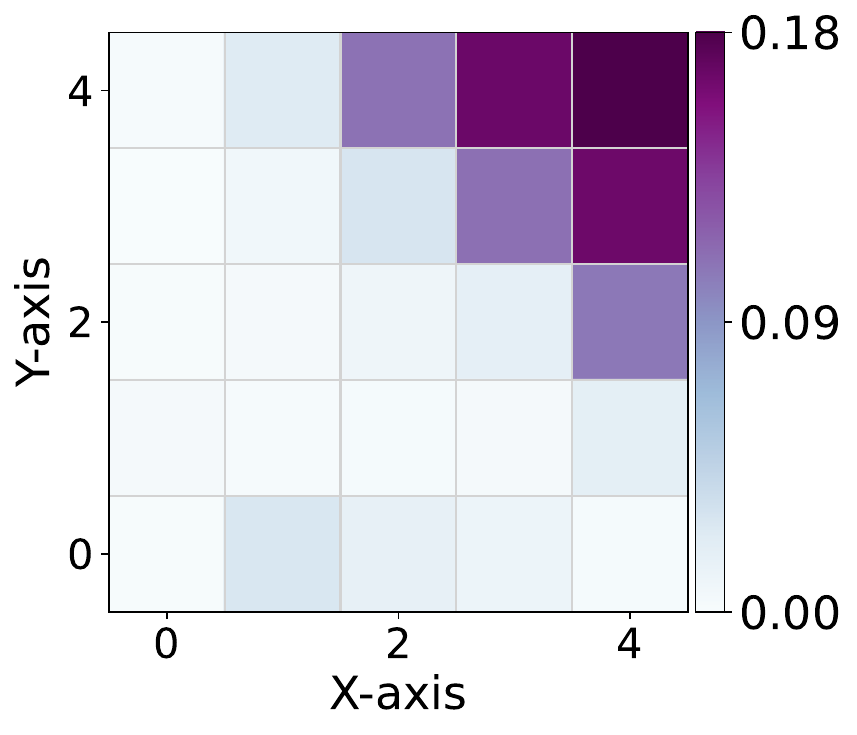}
    \includegraphics[width=0.3\linewidth]{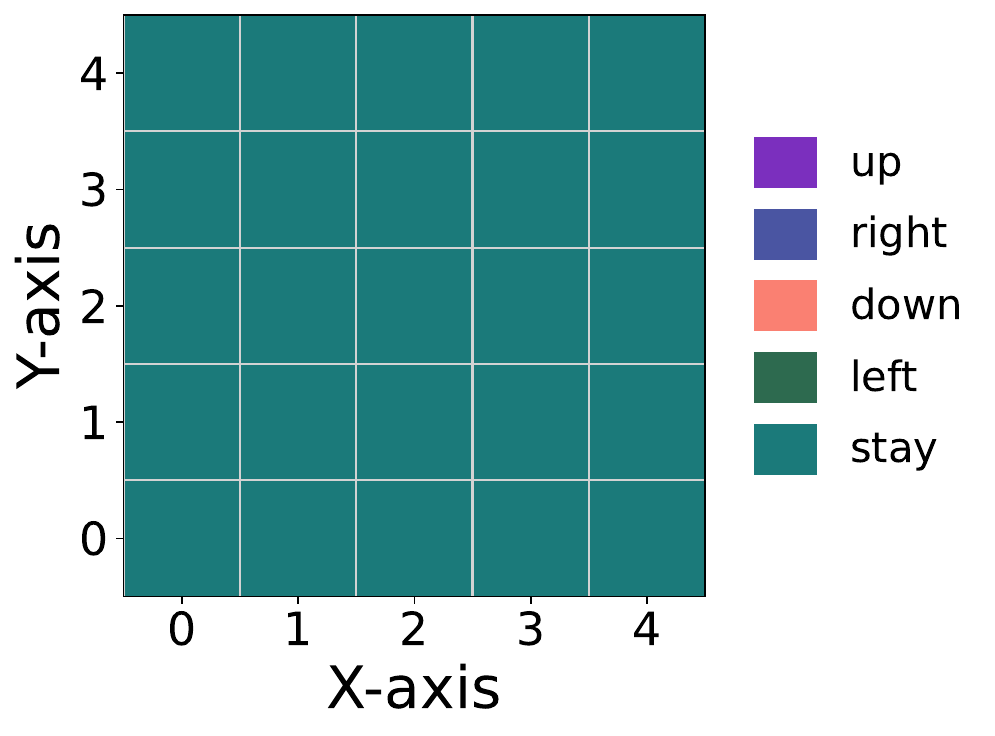}
    \centering
    \includegraphics[width=0.3\linewidth]{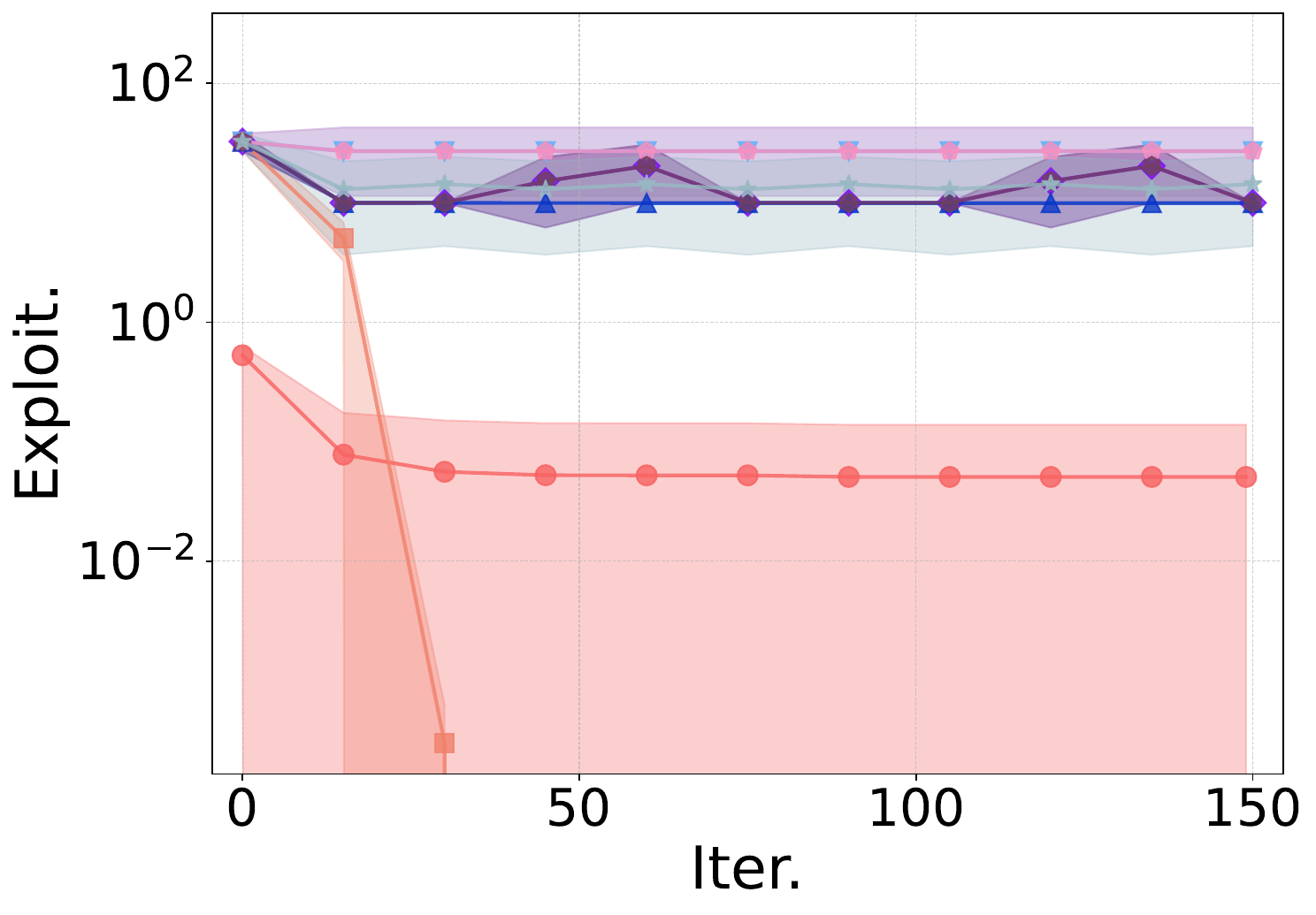}
    \centering
    \includegraphics[width=0.29\linewidth]{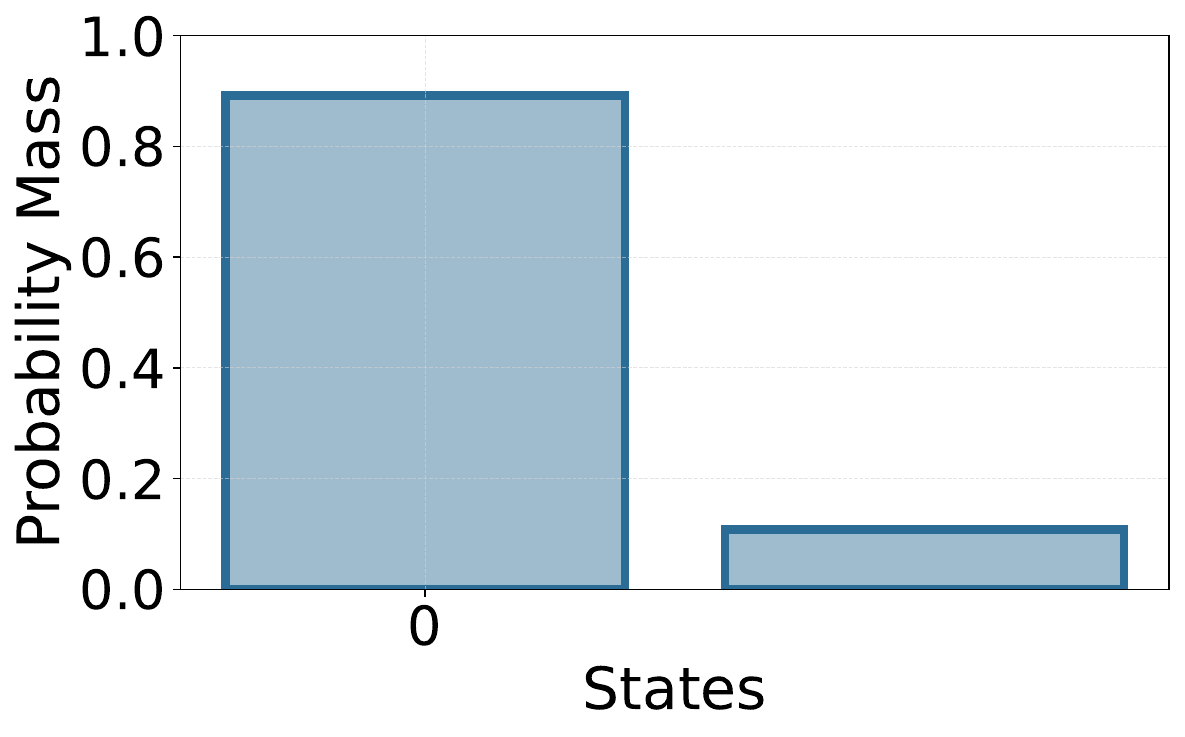}
    \includegraphics[width=0.29\linewidth]{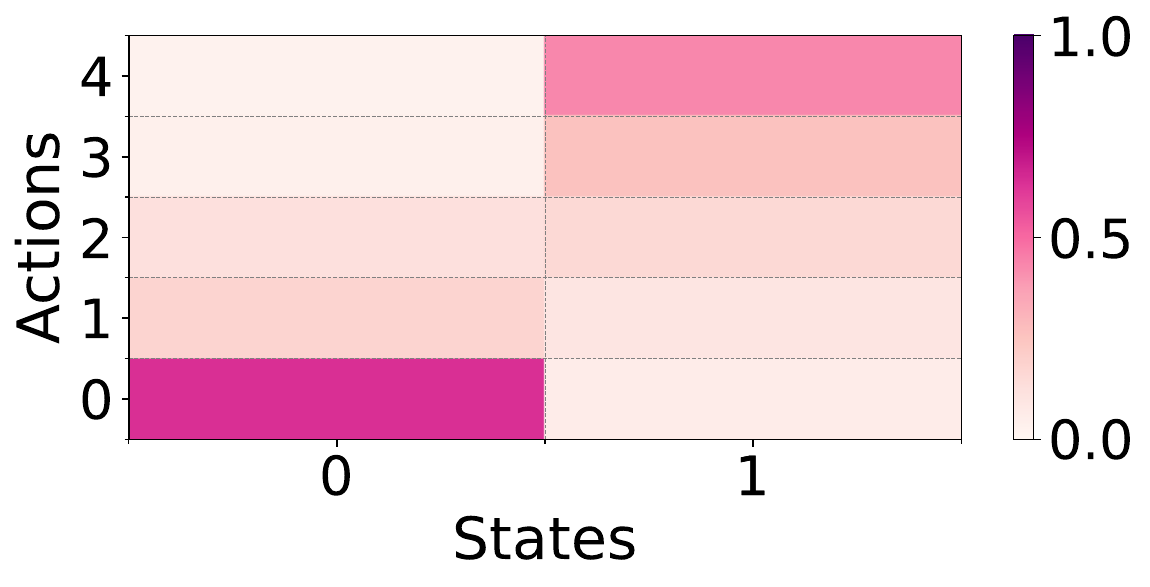}

    \vspace{1mm}
    \algolegend
    \caption{\small \textbf{DC-MFG}. (Top) \textit{Kinetic Congestion} $\tau=0.18, c_{move}=0.1, x_{target}=(4,4)$ Exploitabilities and Equilibrium $(\mu^*, \pi^*)$ for Boltzmann PI. Interestingly, examining the policy plot reveals how this dynamic leads to \textit{congestion} (every player at the equilibrium is at rest). (Bottom) \textit{SISEpidemic} $\beta=0.5, \nu=0.1, C=5.0$. Exploitabilities and equilibrium $(\mu^*, \pi^*)$ (PSO). }
    \label{fig_app: SIS-pi_mu}
\end{figure}
\vspace{-0.3cm}

\subsection{MF-Garnets.}
Experimentation for the MF-Garnet games was conducted across 10 different MFG instances, where each instance has a unique Garnet seed (0--9) to ensure reproducibility. All experiments are characterized by a common configuration: number of states, number of actions, branching factor, dynamics structures (additive or multiplicative), and reward structures (additive or multiplicative).  For each algorithm-instance combination, we extract the final exploitability value. Results are aggregated by computing the mean and standard deviation of final exploitability across all 10 instances for each algorithm for each type of Garnet game. The results are presented in Table~\ref{tab:garnet-results}. 

\begin{table*}[t]
\centering
\footnotesize
\begin{tabular}{lllll}
\toprule
\textbf{Algorithm} & 5x5x5 (A/M) & 5x5x5 (M/M) & 25x10x10 (A/A) & 25x10x10 (M/A) \\
\midrule
Fixed Point (FP) & 1.6537 $\pm$ \textbf{\color{salmon!75!black}3.2993} & 0.0731 $\pm$ 0.2193 & \textbf{\color{pastelblue!88!black}9.50e-04} $\pm$ 0.0017 & \textbf{\color{pastelblue!85!black}2.13e-04} $\pm$ 6.35e-04 \\
Damped FP & 1.4504 $\pm$ 2.8451 & \textbf{\color{pastelblue!88!black}0.0396} $\pm$ 0.1189 & \textbf{\color{pastelblue}8.51e-04} $\pm$ 0.0017 & \textbf{\color{pastelblue}3.43e-05} $\pm$ 9.83e-05 \\
Fictitious Play & \textbf{\color{pastelblue!88!black} 0.6235} $\pm$ 1.6547 & \textbf{\color{pastelblue}0.0118} $\pm$ 0.0236 & 0.0026 $\pm$ 0.0036 & 3.32e-04 $\pm$ 7.40e-04 \\

Boltzmann PI & \textbf{\color{pastelblue!72!black}1.0570} $\pm$ 1.8618 & 0.7435 $\pm$ \textbf{\color{salmon!88!black}0.8503} & 0.9508 $\pm$ \textbf{\color{salmon!75!black}0.2879} & 0.9322 $\pm$ \textbf{\color{salmon!75!black}0.3213} \\
Smooth PI & 2.3228 $\pm$ \textbf{\color{salmon}4.8665} & 0.3223 $\pm$ \textbf{\color{salmon}0.9670} & \textbf{\color{pastelblue!75!black}9.80e-04} $\pm$ 0.0026 & \textbf{\color{pastelblue!75!black}2.91e-04} $\pm$ 8.68e-04 \\
PI & 1.8707 $\pm$ \textbf{\color{salmon!88!black}3.5905} & \textbf{\color{pastelblue!75!black}0.0412} $\pm$ 0.1237 & 0.0030 $\pm$ 0.0046 & 0.0027 $\pm$ 0.0080 \\
OMD & 2.3106 $\pm$ 3.0778 & 1.0101 $\pm$ \textbf{\color{salmon!75!black}0.4339} & 1.4371 $\pm$ \textbf{\color{salmon!88!black}0.3936} & 1.4728 $\pm$ \textbf{\color{salmon!88!black}0.4549} \\
MF-PSO & \textbf{\color{pastelblue}0.2250} $\pm$ 0.2115 & 0.1844 $\pm$ 0.2449 & 3.8633 $\pm$ \textbf{\color{salmon}1.9042} &  4.0158 $\pm$ \textbf{\color{salmon}2.1616} \\
\bottomrule
\end{tabular}
\caption{\small \textbf{MF-Garnet} Results. Each MF-Garnet has $5$ parameters $|\states|\times|\actions|\times$Branching Factor (Dynamics/Reward Structure). We highlighted the three best \textbf{\color{pastelblue} exploitability} values and the worst three \textbf{\color{salmon} standard deviations} per game type. MF-Garnets show that different algorithms can perform significantly differently depending on the type of \textit{game structure} and the \textit{dimensionality}.}
\label{tab:garnet-results}
\vspace{-0.5cm}
\end{table*}
\vspace{-0.2cm}
\paragraph{Ablation and Scalability Analysis}\label{sec:pso-garnet-scalability}
We further evaluate MF-Garnet scalability by increasing the number of states from $10$ to $400$, with six actions and branching factor six, in the multiplicative-dynamics/additive-reward (M/A) setting. Figure~\ref{fig:garnet-scalability} reports final exploitability and wall-clock time across five seeds. At $400$ states, FP, Damped FP, Fictitious Play and PI variants other than Boltzmann PI retain mean exploitability below $10^{-3}$, whereas MF-PSO reaches approximately $5.3$. MF-PSO also requires approximately $432$ seconds, compared with $173$--$181$ seconds for the other methods except Fictitious Play ($274$ seconds). These results expose a limitation of population-based search as the policy dimension grows. The complementary additive-dynamics/multiplicative-reward (A/M) study is reported in Appx.~\ref{app:garnet-scalability-am}, and the particle-count ablation on CPU and GPU in Appx.~\ref{app:pso-scalability}.

\begin{figure}[h!]
\centering
\begingroup
\captionsetup{font=footnotesize, skip=3pt}
\begin{subfigure}[t]{0.43\linewidth}
\centering
\begin{tikzpicture}
\begin{axis}[width=\linewidth, height=3.2cm, xmode=log, ymode=log,
  log basis x=10, xmin=8, xmax=500, xtick={10,20,80,130,400},
  xticklabels={10,20,80,130,400}, xlabel={States $|\states|$},
  tick label style={font=\scriptsize, inner sep=1pt}, label style={font=\scriptsize, inner sep=1pt},
  grid=major, grid style={gray!20}, minor tick num=0,
  legend style={font=\scriptsize, draw=none, row sep=0pt, inner sep=1pt, /tikz/every even column/.append style={column sep=3pt}},
  legend columns=8,ymin=1e-6, ymax=20,ylabel={Final exploitability},legend to name=garnet-scaling-ma-legend,]
\addplot[draw=none, fill=blue!75!black, fill opacity=0.08, forget plot] coordinates {
(10,0.0713945) (20,2.40065e-06) (80,2.244754e-05) (130,0.002497547) (400,0.0002485984) (400,1e-06) (130,1e-06) (80,1e-06) (20,1e-06) (10,1e-06)
} \closedcycle;
\addplot[draw=none, fill=cyan!70!black, fill opacity=0.08, forget plot] coordinates {
(10,0.0307925) (20,1.862411e-06) (80,0.0002736493) (130,0.002468015) (400,0.000617281) (400,1e-06) (130,1e-06) (80,1e-06) (20,1e-06) (10,1e-06)
} \closedcycle;
\addplot[draw=none, fill=green!50!black, fill opacity=0.08, forget plot] coordinates {
(10,0.0298226) (20,0.001156478) (80,0.001300036) (130,0.000322447) (400,0.000308841) (400,1e-06) (130,1e-06) (80,1e-06) (20,1e-06) (10,1e-06)
} \closedcycle;
\addplot[draw=none, fill=violet, fill opacity=0.08, forget plot] coordinates {
(10,0.032882) (20,3.20412e-06) (80,0.00037092) (130,0.002800275) (400,8.24866e-05) (400,1e-06) (130,1e-06) (80,1e-06) (20,1e-06) (10,1e-06)
} \closedcycle;
\addplot[draw=none, fill=brown, fill opacity=0.08, forget plot] coordinates {
(10,0.0463227) (20,3.38195e-06) (80,8.85321e-05) (130,0.001523537) (400,0.0002762877) (400,1e-06) (130,1e-06) (80,1e-06) (20,1e-06) (10,1e-06)
} \closedcycle;
\addplot[draw=none, fill=magenta!70!black, fill opacity=0.08, forget plot] coordinates {
(10,0.906794) (20,0.875816) (80,0.820252) (130,0.873461) (400,0.872342) (400,0.295568) (130,0.310891) (80,0.314894) (20,0.184032) (10,0.320464)
} \closedcycle;
\addplot[draw=none, fill=orange!90!black, fill opacity=0.08, forget plot] coordinates {
(10,1.268121) (20,1.148453) (80,1.093966) (130,1.155182) (400,1.159163) (400,0.571207) (130,0.568414) (80,0.593942) (20,0.476467) (10,0.585493)
} \closedcycle;
\addplot[draw=none, fill=red!75!black, fill opacity=0.08, forget plot] coordinates {
(10,0.661481) (20,0.626622) (80,5.82519) (130,6.19885) (400,9.42454) (400,1.24456) (130,0.88663) (80,0.73623) (20,0.01134) (10,1e-06)
} \closedcycle;
\addplot+[blue!75!black, dashed, thick, mark=square*, mark size=1.7pt, mark options={solid, draw=blue!75!black, fill=blue!75!black, fill opacity=0.3}] coordinates {
(10,0.0271936)
(20,1.28746e-06)
(80,7.93934e-06)
(130,0.000819707)
(400,7.77364e-05)
};
\addlegendentry{FP}
\addplot+[cyan!70!black, densely dashed, thick, mark=*, mark size=1.7pt, mark options={solid, draw=cyan!70!black, fill=cyan!70!black, fill opacity=0.3}] coordinates {
(10,0.0129543)
(20,1e-06)
(80,8.57353e-05)
(130,0.000810075)
(400,0.000207639)
};
\addlegendentry{Damped FP}
\addplot+[green!50!black, dashdotted, thick, mark=triangle*, mark size=1.7pt, mark options={solid, draw=green!50!black, fill=green!50!black, fill opacity=0.3}] coordinates {
(10,0.0104142)
(20,0.000419545)
(80,0.000444591)
(130,0.000140858)
(400,0.000131047)
};
\addlegendentry{Fictitious Play}
\addplot+[violet, dotted, thick, mark=square*, mark size=1.7pt, mark options={solid, draw=violet, fill=violet, fill opacity=0.3}] coordinates {
(10,0.0138775)
(20,1.87159e-06)
(80,0.000115454)
(130,0.000913405)
(400,2.62499e-05)
};
\addlegendentry{PI}
\addplot+[brown, densely dotted, thick, mark=*, mark size=1.7pt, mark options={solid, draw=brown, fill=brown, fill opacity=0.3}] coordinates {
(10,0.0190351)
(20,1.91927e-06)
(80,2.80261e-05)
(130,0.000593448)
(400,8.61287e-05)
};
\addlegendentry{Smooth PI}
\addplot+[magenta!70!black, dotted, thick, mark=*, mark size=1.7pt, mark options={solid, draw=magenta!70!black, fill=magenta!70!black, fill opacity=0.3}] coordinates {
(10,0.613629)
(20,0.529924)
(80,0.567573)
(130,0.592176)
(400,0.583955)
};
\addlegendentry{Boltzmann PI}
\addplot+[orange!90!black, solid, thick, mark=square*, mark size=1.7pt, mark options={solid, draw=orange!90!black, fill=orange!90!black, fill opacity=0.3}] coordinates {
(10,0.926807)
(20,0.81246)
(80,0.843954)
(130,0.861798)
(400,0.865185)
};
\addlegendentry{OMD}
\addplot+[red!75!black, solid, thick, mark=*, mark size=1.7pt, mark options={solid, draw=red!75!black, fill=red!75!black, fill opacity=0.3}] coordinates {
(10,0.288509)
(20,0.318981)
(80,3.28071)
(130,3.54274)
(400,5.33455)
};
\addlegendentry{MF-PSO}
\addplot[gray, dashdotted, forget plot, domain=8:500] {1e-6};
\end{axis}
\end{tikzpicture}
\end{subfigure}
\hspace{0.015\linewidth}
\begin{subfigure}[t]{0.43\linewidth}
\centering
\begin{tikzpicture}
\begin{axis}[width=\linewidth, height=3.2cm, xmode=log, ymode=log,
  log basis x=10, xmin=8, xmax=500, xtick={10,20,80,130,400},
  xticklabels={10,20,80,130,400}, xlabel={States $|\states|$},
  tick label style={font=\scriptsize, inner sep=1pt}, label style={font=\scriptsize, inner sep=1pt},
  grid=major, grid style={gray!20}, minor tick num=0,
  legend style={font=\scriptsize, draw=none, row sep=0pt, inner sep=1pt},
  legend columns=4,ymin=1, ymax=800,ylabel={Wall-clock time (s)},]
\addplot+[blue!75!black, dashed, thick, mark=square*, mark size=1.7pt,
 mark options={solid, draw=blue!75!black, fill=blue!75!black, fill opacity=0.3},
 error bars/.cd, y dir=both, y explicit] coordinates {
(10,2.051) += (0,0.114) -= (0,0.114)
(20,2.442) += (0,0.205) -= (0,0.205)
(80,3.269) += (0,0.307) -= (0,0.307)
(130,8.301) += (0,0.688) -= (0,0.688)
(400,180.502) += (0,8.563) -= (0,8.563)
};
\addplot+[cyan!70!black, densely dashed, thick, mark=*, mark size=1.7pt,
 mark options={solid, draw=cyan!70!black, fill=cyan!70!black, fill opacity=0.3},
 error bars/.cd, y dir=both, y explicit] coordinates {
(10,2.121) += (0,0.095) -= (0,0.095)
(20,2.491) += (0,0.182) -= (0,0.182)
(80,3.361) += (0,0.101) -= (0,0.101)
(130,8.191) += (0,0.31) -= (0,0.31)
(400,176.192) += (0,8.012) -= (0,8.012)
};
\addplot+[green!50!black, dashdotted, thick, mark=triangle*, mark size=1.7pt,
 mark options={solid, draw=green!50!black, fill=green!50!black, fill opacity=0.3},
 error bars/.cd, y dir=both, y explicit] coordinates {
(10,84.114) += (0,0.59) -= (0,0.59)
(20,84.129) += (0,1.697) -= (0,1.697)
(80,89.901) += (0,0.754) -= (0,0.754)
(130,98.823) += (0,2.206) -= (0,2.206)
(400,274.257) += (0,4.953) -= (0,4.953)
};
\addplot+[violet, dotted, thick, mark=square*, mark size=1.7pt,
 mark options={solid, draw=violet, fill=violet, fill opacity=0.3},
 error bars/.cd, y dir=both, y explicit] coordinates {
(10,1.871) += (0,0.127) -= (0,0.127)
(20,2.261) += (0,0.133) -= (0,0.133)
(80,3.144) += (0,0.233) -= (0,0.233)
(130,7.828) += (0,0.266) -= (0,0.266)
(400,173.909) += (0,3.906) -= (0,3.906)
};
\addplot+[brown, densely dotted, thick, mark=*, mark size=1.7pt,
 mark options={solid, draw=brown, fill=brown, fill opacity=0.3},
 error bars/.cd, y dir=both, y explicit] coordinates {
(10,2.288) += (0,0.149) -= (0,0.149)
(20,2.612) += (0,0.162) -= (0,0.162)
(80,3.667) += (0,0.268) -= (0,0.268)
(130,7.773) += (0,0.355) -= (0,0.355)
(400,178.109) += (0,5.725) -= (0,5.725)
};
\addplot+[magenta!70!black, dotted, thick, mark=*, mark size=1.7pt,
 mark options={solid, draw=magenta!70!black, fill=magenta!70!black, fill opacity=0.3},
 error bars/.cd, y dir=both, y explicit] coordinates {
(10,2.275) += (0,0.182) -= (0,0.182)
(20,2.638) += (0,0.07) -= (0,0.07)
(80,3.422) += (0,0.202) -= (0,0.202)
(130,8.136) += (0,0.219) -= (0,0.219)
(400,173.463) += (0,5.427) -= (0,5.427)
};
\addplot+[orange!90!black, solid, thick, mark=square*, mark size=1.7pt,
 mark options={solid, draw=orange!90!black, fill=orange!90!black, fill opacity=0.3},
 error bars/.cd, y dir=both, y explicit] coordinates {
(10,2.409) += (0,0.19) -= (0,0.19)
(20,2.77) += (0,0.179) -= (0,0.179)
(80,3.846) += (0,0.215) -= (0,0.215)
(130,8.687) += (0,0.481) -= (0,0.481)
(400,177.039) += (0,3.924) -= (0,3.924)
};
\addplot+[red!75!black, solid, thick, mark=*, mark size=1.7pt,
 mark options={solid, draw=red!75!black, fill=red!75!black, fill opacity=0.3},
 error bars/.cd, y dir=both, y explicit] coordinates {
(10,5.756) += (0,0.274) -= (0,0.274)
(20,6.819) += (0,0.296) -= (0,0.296)
(80,20.656) += (0,1.273) -= (0,1.273)
(130,41.846) += (0,3.215) -= (0,3.215)
(400,432.321) += (0,23.593) -= (0,23.593)
};
\end{axis}
\end{tikzpicture}
\end{subfigure}
\par\vspace{1pt}
\pgfplotslegendfromname{garnet-scaling-ma-legend}
\endgroup
\caption{\small \textbf{MF-Garnet scalability (M/A: multiplicative dynamics, additive rewards).} Six actions, branching factor six; NVIDIA L40 GPU. Means over five seeds.}
\label{fig:garnet-scalability}
\end{figure}

\vspace{-0.5cm}
\section{Discussion}

\noindent\textbf{Findings.} Fictitious Play and Online Mirror Descent (OMD) are widely used in the community, yet our results refine their ideal use cases. Surprisingly, Fictitious Play demonstrates robustness by 
\vspace{-0.15mm}
\setlength{\belowcaptionskip}{-0.2cm}
\begin{wrapfigure}{R}{0.30\textwidth}
  \centering
\begin{tikzpicture}\label{fig:runtime_SIS}
    \node[anchor=south west, inner sep=0] (image) at (0,0) {
      \includegraphics[width=\linewidth]{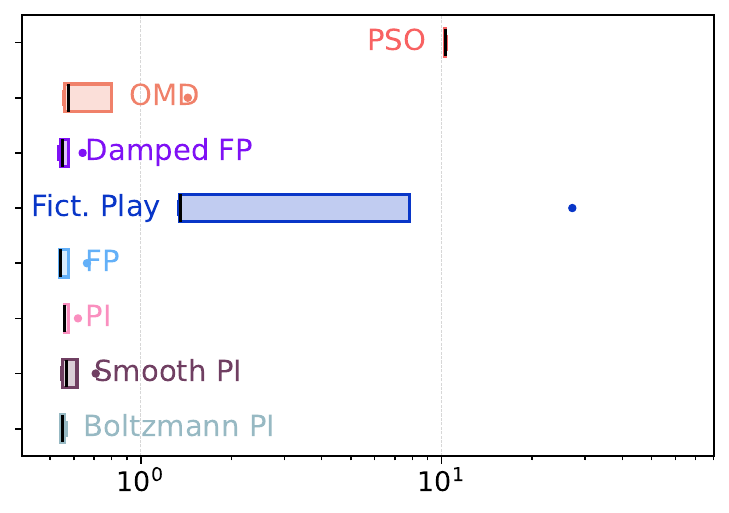}
    };
    \begin{scope}[x={(image.south east)}, y={(image.north west)}]
      \node[font=\footnotesize, text=black] at (0.5, -0.05) {Wall-clock runtime (s) - SIS};
    \end{scope}
  \end{tikzpicture}
  \vspace{-8mm}
\end{wrapfigure}
converging even in \textit{non-monotone} settings where theoretical guarantees are absent. 
In contrast, OMD often reaches low exploitability on the handcrafted environments, particularly in cyclic games, but\textit{ requires delicate hyperparameter tuning} due to high sensitivity (Appx.~\ref{app: sweep}), and it is among the weakest methods on MF-Garnets (Table~\ref{tab:garnet-results}). 
Finally, MF-PSO outperforms Fictitious Play on several handcrafted benchmarks, but is less effective on larger MF-Garnets (Section~\ref{sec:pso-garnet-scalability}). Its computational cost increases with particle count, with the observed scaling depending on the hardware (Appx.~\ref{app:pso-scalability}; see also the SIS runtime inset).
Furthermore, we observe that algorithms achieving comparable (near-optimal) exploitability \textbf{may still converge to different equilibrium representations}, which is why it remains important to explicitly display the equilibrium $(\mu^*,\pi^*)$. Lastly, from Figs. \ref{fig:wall-clock1}–\ref{fig:wall-clock4}, it can be seen that MF-PSO (a sample-based method) exhibits greater variability in convergence runtime, depending on the specific properties of the problem.

\noindent\textbf{Guidelines and Outlook.} These findings suggest that reliable evaluation of MFG learning algorithms requires both standard machine learning practices and MFG-specific checks.
The following points are particularly relevant for the MFG learning algorithms: 1) \textit{Fixed Point Sanity Check:} Always compare against the simple Fixed Point (FP) method. If FP solves the MFG in just a few iterations, applying more sophisticated methods might introduce unnecessary complexity. 2) \textit{Solvers:} Include \textbf{MF-PSO} as a model-agnostic, exploitability-driven baseline: it is competitive on several small handcrafted environments, although its performance degrades as the policy dimension grows (Section~\ref{sec:pso-garnet-scalability}); 3) \textit{Regime Testing:} Do not rely on a single game type. We proposed MF-Garnet to test the performance of SOTA algorithms on settings with different properties; 4) \textit{Environment Parameters:} Even in a single environment, the parameter values can have a huge influence on how difficult it is to compute an equilibrium (see Appx.~\ref{app:env_sweep}).
\vspace{-0.2cm}
\paragraph{RL Comparison}
\begin{wrapfigure}{R}{0.44\textwidth}
\centering
\captionsetup{font=footnotesize}
\begin{tikzpicture}
\begin{axis}[
  width=\linewidth, height=2.9cm,
  ymode=log, xmin=0, xmax=150, ymin=0.35, ymax=450,
  xtick={0,30,60,90,120,150},
  xlabel={FP iteration}, ylabel={Exploitability},
  tick label style={font=\scriptsize}, label style={font=\scriptsize},
  grid=major, grid style={gray!20},
  legend columns=3,
  legend style={font=\scriptsize, at={(0.5,1.08)}, anchor=south, draw=none, fill=none},
]
\addplot[draw=none, fill=blue!70!black, fill opacity=0.13, forget plot] coordinates {
(0,325.669) (0,127.409) (10,10.4111) (20,1.38069) (30,2.32922) (40,3.66634) (50,3.17269) (60,3.79611) (70,4.37614) (80,4.05852) (90,4.61704) (100,4.27772) (110,4.26448) (120,4.6878) (130,4.56253) (140,4.6423) (150,4.92619) (150,5.43763) (140,4.83703) (130,5.5235) (120,5.30503) (110,4.63987) (100,5.28072) (90,5.73451) (80,4.74901) (70,5.19047) (60,5.62736) (50,4.35232) (40,5.16175) (30,4.79481) (20,4.95879) (10,16.7381) (0,325.669)
} \closedcycle;
\addplot[draw=none, fill=teal!80!black, fill opacity=0.13, forget plot] coordinates {
(0,325.669) (0,127.409) (10,5.11488) (20,0.509478) (30,1.67196) (40,2.98748) (50,3.24659) (60,4.09707) (70,3.14749) (80,3.78356) (90,4.46028) (100,4.73343) (110,4.15182) (120,4.32302) (130,4.41043) (140,4.79217) (150,4.93848) (150,6.16925) (140,5.94812) (130,6.15319) (120,5.90972) (110,6.11794) (100,5.30939) (90,6.48891) (80,5.96578) (70,5.86218) (60,6.09794) (50,6.38794) (40,7.45913) (30,5.79565) (20,7.81167) (10,17.0844) (0,325.669)
} \closedcycle;
\addplot[draw=none, fill=violet, fill opacity=0.13, forget plot] coordinates {
(0,325.669) (0,127.409) (10,13.6041) (20,11.1099) (30,7.10358) (40,3.31855) (50,3.05627) (60,1.83024) (70,3.0185) (80,4.28559) (90,3.75295) (100,3.55183) (110,3.54204) (120,4.2573) (130,5.03643) (140,3.90975) (150,4.58822) (150,6.86798) (140,6.095) (130,5.89177) (120,5.31747) (110,6.56872) (100,6.51919) (90,7.96317) (80,7.22939) (70,7.71118) (60,8.35047) (50,7.69387) (40,6.94137) (30,9.50045) (20,19.4147) (10,26.4289) (0,325.669)
} \closedcycle;
\addplot[blue!70!black, thick, mark=*, mark size=1.1pt, mark options={solid, draw=blue!70!black, fill=white}] coordinates {
(0,226.539) (10,13.5746) (20,3.16974) (30,3.56201) (40,4.41404) (50,3.76251) (60,4.71174) (70,4.78331) (80,4.40376) (90,5.17577) (100,4.77921) (110,4.45217) (120,4.99641) (130,5.04301) (140,4.73966) (150,5.18191)
};
\addlegendentry{Exact}
\addplot[teal!80!black, thick, mark=square*, mark size=1.1pt, mark options={solid, draw=teal!80!black, fill=white}] coordinates {
(0,226.539) (10,11.0997) (20,4.16057) (30,3.73381) (40,5.22331) (50,4.81726) (60,5.0975) (70,4.50484) (80,4.87467) (90,5.4746) (100,5.02141) (110,5.13488) (120,5.11637) (130,5.28181) (140,5.37015) (150,5.55386)
};
\addlegendentry{PPO}
\addplot[violet, thick, mark=triangle*, mark size=1.1pt, mark options={solid, draw=violet, fill=white}] coordinates {
(0,226.539) (10,20.0165) (20,15.2623) (30,8.30202) (40,5.12996) (50,5.37507) (60,5.09036) (70,5.36484) (80,5.75749) (90,5.85806) (100,5.03551) (110,5.05538) (120,4.78738) (130,5.4641) (140,5.00238) (150,5.7281)
};
\addlegendentry{DQN}
\end{axis}
\end{tikzpicture}
\caption{\textbf{RL comparison on Beach Bar (LL-MFG).} Exact best responses provide a reference for PPO and DQN within Fictitious Play. Model-based evaluations have been used for all methods.}
\label{fig:rl-comparison}
\end{wrapfigure}

Bench-MFG provides a baseline for studying RL approximation algorithms in mean-field games. Exact best responses supply a reference for assessing the additional sampling and function-approximation errors of RL solvers~\citep{perrin2020fictitious}. This helps determine when approximation is needed, for instance when the model is unavailable to the learner or exact computation becomes impractical, and how well it performs in exploitability, convergence and computational cost. Our Gym-style interface replaces the inner best-response solver while retaining the same outer algorithm. Figure~\ref{fig:rl-comparison} compares PPO, DQN and exact best responses within Fictitious Play on Beach Bar (LL-MFG), using model-based exploitability evaluation for all methods. Similar final exploitability but different transient behavior and variability illustrate the value of this reference. 
\vspace{-0.3cm}
\section{Conclusion and Future Work}
\vspace{-0.3cm}
In this work, we addressed the lack of standardized evaluation in MFGs by proposing a unified benchmark suite implemented in JAX for efficiency and pure Python for reproducibility. We established a \textbf{taxonomy} of problem classes (from Contractive and Monotone games to complex Dynamics-Coupled environments) and \textbf{introduced MF-Garnets} for the procedural generation of random instances. By benchmarking a variety of solvers and our exploitability minimizer (\textbf{MF-PSO}), we demonstrated the necessity of testing across different regimes and synthesized our findings into a set of best practices for reproducible research.

\textbf{Limitations.} The present work is limited to finite spaces and stationary problems. Our main evaluation focuses on algorithms with exact computations based on the model; the PPO and DQN comparison is a preliminary study of model-free best-response approximation, and exploitability evaluation still requires the model. Broader RL evaluation and extensions to continuous spaces remain directions for future work.

\noindent\textbf{Impact statement.} This paper presents work whose goal is to advance the field of Machine Learning. It aims at improving the quality of experimental results related to learning methods for MFGs, making them clearer and more robust. Moreover, since our framework computes equilibrium policies with exact, model-based solvers, it can serve as an expert to train imitation learning (IL) policies in the mean field setting, a paradigm that holds great potential for real-world applications.

\textbf {Acknowledgement} The experiments in this paper are partially supported by NYU Torch and the NYUSH HPC cluster. Jiacheng Shen is supported by Earth Species Project. During the main development of the work, Lorenzo Magnino was at NYU Shanghai. 

\noindent\textbf{Competing interests.} The authors declare no competing interests.

\bibliography{mfg-num-bib2024}
\bibliographystyle{plainnat}

\newpage
\appendix

\section{Algorithms Details}\label{app:algos}

\subsection{Pseudocodes}

\begin{algorithm}[H]
\small
\caption{Fictitious Play}
\label{algo:fp}
\begin{algorithmic}[1]
    \REQUIRE Initial logits' distribution $\nu\in\Delta_{\mathbb{R}^{|\states|\times|\actions|}}$; Number of iterations $K$; Mean field transition steps $N$.
    \STATE \textbf{Initialize:} Generate initial policy $\bar\pi_0$ using Boltzmann transformation and initialize $\bar\mu_0=\mathbf{M}^N(\bar\pi_0)$.
    \FOR{iteration $k = 1$ to $K$}
        \STATE Find the best response $\pi_k^*$ such that: $\pi_k^* = \arg\max_{\pi} J_{\gamma} (\pi;\bar\mu_{k-1})$
        \STATE Compute the stationary mean field distribution: $\mu^{\pi^*_k} = \mathbf{M}^N(\pi^*_k)$.
        \STATE Compute the average mean field: $\bar\mu_k = \frac{k}{k+1}\bar\mu_{k-1} + \frac{1}{k+1}\mu^{\pi^*_k}$.
        \STATE Compute $\bar\pi_k$ as the average of $(\bar\pi_0,\pi^*_1, \dots,\pi^*_k)$.
\ENDFOR
\ENSURE $(\bar\pi_K,\bar\mu_K)$
\end{algorithmic}
\end{algorithm}

\begin{algorithm}[H]
\small
\caption{Damped Fictitious Play}
\label{algo:dampedfp}
\begin{algorithmic}[1]
    \REQUIRE Initial logits' distribution $\nu\in\Delta_{\mathbb{R}^{|\states|\times|\actions|}}$; Number of iterations $K$; Damping constant $\lambda \in (0,1]$; Number of mean field transition steps $N$.
    \STATE \textbf{Initialize:} Generate initial policy $\pi_0$ using Boltzmann transformation and initialize $\mu_0=\mathbf{M}^N(\pi_0)$.
    \FOR{iteration $k = 1$ to $K$}
        \STATE Find the best response $\pi_k^*$ such that: $\pi_k^* = \arg\max_{\pi} J_{\gamma} (\pi;\mu_{k-1})$
        \STATE Compute the stationary mean field distribution: $\mu^{\pi^*_k} = \mathbf{M}^N(\pi^*_k)$.
        \STATE Compute the damped mean field: $\mu_k = (1-\lambda)\mu_{k-1} + \lambda\mu^{\pi^*_k}$.
        \STATE Set $\pi_k = \pi_k^*$.
    \ENDFOR
    \ENSURE $(\pi_K,\mu_K)$
\end{algorithmic}
\end{algorithm}

\begin{algorithm}[H]
\small
\caption{Pure Fixed Point}
\label{algo:purefp}
\begin{algorithmic}[1]
    \REQUIRE Initial logits' distribution $\nu\in\Delta_{\mathbb{R}^{|\states|\times|\actions|}}$; Number of iterations $K$; Number of mean field transition steps $N$.
    \STATE \textbf{Initialize:} Generate initial policy $\pi_0$ using Boltzmann transformation and initialize $\mu_0=\mathbf{M}^N(\pi_0)$.
    \FOR{iteration $k = 1$ to $K$}
        \STATE Find the best response $\pi_k^*$ such that: $\pi_k^* = \arg\max_{\pi} J_{\gamma} (\pi;\mu_{k-1})$
        \STATE Compute the stationary mean field distribution: $\mu_k = \mathbf{M}^N(\pi^*_k)$.
        \STATE Set $\pi_k = \pi_k^*$.
    \ENDFOR
    \ENSURE $(\pi_K,\mu_K)$
\end{algorithmic}
\end{algorithm}

\begin{algorithm}[H]
\small
\caption{Policy Iteration}
\label{algo:pi}
\begin{algorithmic}[1]
    \REQUIRE Initial logits' distribution $\nu\in\Delta_{\mathbb{R}^{|\states|\times|\actions|}}$; Number of iterations $K$; Number of mean field transition steps $N$.
    \STATE \textbf{Initialize:} Generate initial policy $\pi_0$ using Boltzmann transformation and initialize $\mu_0=\mathbf{M}^N(\pi_0)$.
    \STATE Compute initial Q-values: $Q_0 = Q^{\pi_0, \mu_0}$.
    \FOR{iteration $k = 1$ to $K$}
        \STATE Update policy greedily: $\pi_k = \arg\max_{a \in \actions} Q_{k-1}(\cdot, a)$.
        \STATE Compute the stationary mean field distribution: $\mu_k = \mathbf{M}^N(\pi_k)$.
        \STATE Compute state-action value function: $Q_k = Q^{\pi_k, \mu_k}$.
    \ENDFOR
    \ENSURE $(\pi_K, \mu_K)$.
\end{algorithmic}
\end{algorithm}

\begin{algorithm}[H]
\small
\caption{Smooth Policy Iteration}
\label{algo:smoothpi}
\begin{algorithmic}[1]
    \REQUIRE Initial logits' distribution $\nu\in\Delta_{\mathbb{R}^{|\states|\times|\actions|}}$; Number of iterations $K$; Damping constant $\lambda \in (0,1]$ (or $\lambda_k = 1/(k+1)$); Number of mean field transition steps $N$.
    \STATE \textbf{Initialize:} Generate initial policy $\pi_0$ using Boltzmann transformation and initialize $\mu_0=\mathbf{M}^N(\pi_0)$.
    \STATE Compute initial Q-values: $Q_0 = Q^{\pi_0, \mu_0}$.
    \FOR{iteration $k = 1$ to $K$}
        \STATE Update policy greedily: $\pi_k = \arg\max_{a \in \actions} Q_{k-1}(\cdot, a)$.
        \STATE Compute the new stationary mean field distribution: $\mu^{\pi_k} = \mathbf{M}^N(\pi_k)$.
        \STATE Compute damped mean field: $\mu_k = \lambda_k \mu^{\pi_k} + (1-\lambda_k)\mu_{k-1}$ (where $\lambda_k = \lambda$ if constant, or $\lambda_k = 1/(k+1)$ if decreasing).
        \STATE Compute state-action value function: $Q_k = Q^{\pi_k, \mu_k}$.
    \ENDFOR
    \ENSURE $(\pi_K, \mu_K)$.
\end{algorithmic}
\end{algorithm}

\begin{algorithm}[H]
\small
\caption{Boltzmann Policy Iteration}
\label{algo:boltzmannpi}
\begin{algorithmic}[1]
    \REQUIRE Initial logits' distribution $\nu\in\Delta_{\mathbb{R}^{|\states|\times|\actions|}}$; Number of iterations $K$; Softmax temperature $\tau$; Number of mean field transition steps $N$.
    \STATE \textbf{Initialize:} Generate initial policy $\pi_0$ using Boltzmann transformation and initialize $\mu_0=\mathbf{M}^N(\pi_0)$.
    \STATE Compute initial Q-values: $Q_0 = Q^{\pi_0, \mu_0}$.
    \FOR{iteration $k = 1$ to $K$}
        \STATE Update policy using softmax: $\pi_k = \varphi_{\tau}(Q_{k-1})$.
        \STATE Compute the stationary mean field distribution: $\mu_k = \mathbf{M}^N(\pi_k)$.
        \STATE Compute state-action value function: $Q_k = Q^{\pi_k, \mu_k}$.
    \ENDFOR
    \ENSURE $(\pi_K, \mu_K)$.
\end{algorithmic}
\end{algorithm}

\begin{algorithm}[H]
\small
\caption{Online Mirror Descent (OMD)}
\label{algo:omd}
\begin{algorithmic}[1]
    \REQUIRE Initial logits' distribution $\nu\in\Delta_{\mathbb{R}^{|\states|\times|\actions|}}$; Number of iterations $K$; Learning rate $\alpha$; Softmax temperature $\tau$; Number of mean field transition steps $N$.
    \STATE \textbf{Initialize:} Generate initial policy $\bar\pi_0$ using Boltzmann transformation and initialize $\bar\mu_0=\mathbf{M}^N(\bar\pi_0)$. Initialize $\tilde Q_0 = Q^{\bar\pi_0, \bar\mu_0}$.
    \FOR{iteration $k = 1$ to $K$}
        \STATE Update policy:
            $\pi_{k} = \varphi_{\tau}(\tilde{Q}_{k-1})$
        \STATE Compute the new stationary mean field distribution:
            $\mu_{k} = \mathbf{M}^N(\pi_{k})$
        \STATE Compute state-action value function $Q_{k} = Q^{\pi_k, \mu_k}$.
        \STATE Update mirrored state-action values:
            $\tilde{Q}_{k} = \tilde{Q}_{k-1} + \alpha Q_{k}$
    \ENDFOR
    \ENSURE $(\pi_K, \mu_K)$.
\end{algorithmic}
\end{algorithm}

\begin{algorithm}[H]
\small
\caption{Particle Swarm Optimization (PSO) for Exploitability Minimization}
\label{algo:pso}
\begin{algorithmic}[1]
    \REQUIRE Initial logits' distribution $\nu\in\Delta_{\mathbb{R}^{|\states|\times|\actions|}}$; Number of particles $\mathrm{P}$, Number of iterations $K$. Inertia weight $w$, cognitive weight $c_1$, social weight $c_2$. Softmax temperature $\tau$.
    \STATE \textbf{Initialize:}
    \FOR{each particle $i=1, \dots, \mathrm{P}$}
        \STATE Initialize position (logits) $\theta_i\sim\nu$ .
        \STATE Initialize velocity $v_i$ to zero.
        \STATE Set personal best position $p_i \gets \theta_i$.
    \ENDFOR
    \STATE Find the global best position $p_g \gets \argmin_{\theta_i} \exploitability(\varphi_{\tau}(\theta_i))$.

    \FOR{iteration $k = 1$ to $K$}
        \FOR{each particle $i=1, \dots, \mathrm{P}$}
            \STATE Generate random vectors $r_1, r_2 \sim U(0, 1)$.
            \STATE Update velocity: $v_i \gets w v_i + c_1 r_1 \cdot (p_i - \theta_i) + c_2 r_2 \cdot (p_g - \theta_i)$.
            \STATE Update position: $\theta_i \gets \theta_i + v_i$.
            \STATE Calculate new exploitability $e_i = \exploitability(\varphi_{\tau}(\theta_i))$.
            \IF{$e_i < \exploitability(\varphi_{\tau}(p_i))$}
                \STATE Update personal best: $p_i \gets \theta_i$.
            \ENDIF
            \IF{$e_i < \exploitability(\varphi_{\tau}(p_g))$}
                \STATE Update global best: $p_g \gets \theta_i$.
            \ENDIF
        \ENDFOR
    \ENDFOR
    \ENSURE The global best policy $\varphi_{\tau}(p_g)$ and the stationary distribution generated by that $\mathbf{M}^N(\varphi_{\tau}(p_g))$.
\end{algorithmic}
\end{algorithm}

\subsection{Average Policy}\label{app:average_policy}
We average the $k+1$ policies $(\pi_0, \pi_1, \dots, \pi_k)$ using the weighted average:
\[
\quad \bar{\pi}_k(a|x) = \frac{\sum_{i=0}^k \mu^{\pi_i}(x)\pi_i(a|x)}{\sum_{i=0}^k \mu^{\pi_i}(x)}.
\]
\subsection{Best Response Computation via Backward Induction}\label{app:br}
\textbf{Backward Induction:}
Initialize $V^0(x) = 0$ for all states $x$. For $k = 0, 1, \ldots, T-1$:

\begin{equation}
V^{k+1}(x) = \max_{a \in \actions} \left[ r(x, a, \mu) + \gamma \sum_{x'} p(x'|x, a, \mu) V^k\left(x'\right) \right]
\end{equation}
In the end, we obtain an approximation of the optimal value function $V^{T}\approx V^*$.

\textbf{Optimal Policy Extraction:}
From the final value function $V^* = V^{T}$, compute action values:

\begin{equation}
Q^*(x, a) = r(x, a, \mu) + \gamma \sum_{x'} p(x'|x, a, \mu) V^*\left(x'\right)
\end{equation}

The optimal policy is:
\begin{equation}
\pi^*(a|x) = \begin{cases} 
1 & \text{if } a = a^*(x) \\
0 & \text{otherwise}
\end{cases}
\end{equation}

where $a^*(x)$ is the first element of $\arg\max_{a} Q^*(x, a)$, the set of optimal actions at state $x$ (ties are broken by choosing the first action).

\newpage
\section{Proofs}
\textbf{Proof of Prop.~\ref{prop: potential_characterization}}\label{proof: potential_characterization}
\begin{proof}
{\color{black} Choose affine coordinates $\mu=\mu_0+Bz$ on the simplex, where the columns of $B$ span the zero-sum tangent space. The reward one-form in these coordinates is $B^\top g(\mu_0+Bz)$, whose Jacobian is $B^\top D_\mu g(\mu)B$. If a potential exists, this Jacobian is symmetric by equality of mixed derivatives. Conversely, the stated tangent condition makes this Jacobian symmetric on the convex relative interior, so integration yields a potential $G$. Continuous differentiability up to the boundary extends the identity there. Adding $\sum_x\psi(x,a)\mu(x)$ yields the reward potential for each action. This is the variational-derivative convention of Appx.~\ref{app: functional_derivative}.}
\end{proof}

\section{Technical Notes}

\subsection{\textcolor{black}{MF-Garnet Transition Normalization}}\label{app:garnet-transitions}
{\color{black} In both cases, let $w_{x'}=\tilde P(x'|x,a,\mu)+\delta$ and $S=\sum_z w_z$, where $\delta\geq0$ is an optional intensity floor (zero in the benchmark). The implementation computes
\[
 q_{x'}=\frac{w_{x'}}{S+\epsilon},\qquad
 \widehat p_{x'}=\frac{q_{x'}}{\max\{\sum_zq_z,\epsilon\}},\qquad \epsilon>0.
\]
For $\sum_zq_z\geq\epsilon$ this equals $w_{x'}/S$. Transitions use inverse-CDF sampling with uniformly weighted midpoints $u_n=(n+1/2)/N_{\rm noise}$; any residual mass is assigned to the last state. In particular, $S=0$ sends every agent to the last state. Thus the implemented transition law is the distribution induced by this finite noise grid. The dynamics are parameterized by $(P^0,\Omega,c_p,\rho_p,\epsilon,\delta)$ and the choice of additive or multiplicative intensity.}

\subsection{Definition of derivative with respect to a measure}\label{app: functional_derivative}
Let $\mathcal V: \mathcal D \to \mathbb{R}$ be a functional on the space of probability distributions over our finite state space $\states$.
We say that the map $g: \states \times \mathcal D \to \mathbb{R}$ is the \textbf{variational derivative} of $\mathcal V$ at $\mu$, denoted $g(x, \mu) = {\nabla \mathcal V}(x, \mu)$, if for any two distributions $\mu, \mu' \in \mathcal D$, the following limit holds:
\[
\lim_{s \to 0} \frac{\mathcal V((1-s)\mu + s\mu') - \mathcal V(\mu)}{s} = \sum_{x \in \states} \nabla \mathcal V(x, \mu) \left( \mu'(x) - \mu(x) \right) = \sum_{x \in \states} g(x, \mu) \left( \mu'(x) - \mu(x) \right)
\]
This expresses the first-order change in the functional $\mathcal V$ in response to a small perturbation of the measure $\mu$ in the direction of $\mu'$.

\noindent\textbf{Integral Representation.} 
Assuming the derivative map $g(x, \mu)$ is continuous with respect to $\mu$, the fundamental theorem of calculus allows us to write the total change in $\mathcal V$ as an integral over the path between $\mu$ and $\mu'$:
\[
\mathcal V(\mu') - \mathcal V(\mu) = \int_0^1 \sum_{x \in \states} g(x, (1-s)\mu + s\mu') \left( \mu'(x) - \mu(x) \right) ds
\]

\noindent\textbf{Normalization.} 
As noted in the literature \citep{cardaliaguet2017learning}, this derivative is typically defined only up to an additive constant. To ensure uniqueness, we impose the following normalization condition:
\[
\sum_{x \in \states} g(x, \mu) \mu(x) = C, \quad \forall \mu \in \mathcal D, \quad C\ge 0
\]
This means that the expected value of the derivative under the measure itself is $C$.

\section{Details on Contractive MFG} 

{\color{black}
\subsection{Theoretical limitation of Contractive MFGs in discrete spaces}
\label{sec:contractive-limitation}
Before defining the benchmark example, we establish a formal limitation of deterministic best responses in discrete action spaces. This lemma justifies why ``contractive'' examples in this setting must necessarily involve policies that are locally constant with respect to the mean field.

\begin{lemma}[Constant Best Response]
\label{lem:contraction-constant}
    Let $\states$ and $\actions$ be finite sets. Let $\Delta_\states$ be the simplex of population distributions. Assume the best response map $\BR: \Delta_\states \to (\Delta_\actions)^\states$ is:
    \begin{enumerate}
        \item \textbf{Single-valued:} For every $\mu$, there is a unique optimal policy $\pi^*_\mu$.
        \item \textbf{Continuous:} The map $\mu \mapsto \pi^*_\mu$ is continuous.
    \end{enumerate}
    Then, the best response map is \textbf{constant} on $\Delta_\states$. That is, there exists a single deterministic policy $\pi$ such that $\BR(\mu) = \pi$ for all $\mu \in \Delta_\states$.
\end{lemma}

\begin{proof}
    Consider the standard (unregularized) MDP objective. The set of deterministic stationary policies, denoted by $\Pi_{det} \subset (\Delta_\actions)^\states$, is finite, containing $|\actions|^{|\states|}$ elements (the vertices of the policy polytope).
    
    Since the action space is finite and the objective is linear in policy occupancy measures, optimal policies always lie on the faces of the optimization polytope. If the solution is unique (first point in the statement), it must be a vertex, meaning $\BR(\mu) \in \Pi_{det}$.
    
    We are given that $\BR: \Delta_\states \to (\Delta_\actions)^\states$ is a continuous map. The domain $\Delta_\states$ is a connected topological space (a convex simplex). The image of the map, $\BR(\Delta_\states)$, must therefore be a connected subset of the codomain.
    
    However, the image is constrained to be a subset of $\Pi_{det}$, which is a finite discrete set. The only connected subsets of a finite discrete set are singletons. Thus, the image of $\BR$ consists of a single policy $\pi$.
\end{proof}

\subsection{{\color{black} Contractive Example: Coordination Game}}
\label{sec:contractive-ex-proof}

We now construct a concrete MFG example satisfying the conditions of the lemma. In this game, mean field interactions are present in the reward function (distinguishing it from a No-Interaction game), but they are sufficiently weak relative to the ``intrinsic'' costs that they never trigger a change in the optimal strategy. {\color{black} The uniform resets select the unique stationary population $q$, yielding a contraction with coefficient zero.}

\subsubsection{Details on Example~\ref{ex:contractive-example}}
\label{sec:contractive-ex-analysis}

We show that if the switching cost $C$ dominates the maximum possible gain from avoiding congestion, the unique best response is to always ``Stay''.

\begin{proposition}[Dominated Interaction Condition]
\label{prop:contractive-ex-property}
    If the parameters satisfy the condition:
    \(
        C > \frac{\alpha}{1-\gamma},
    \)
    then the Best Response is unique and constant: $\pi^*(a=\text{Stay}|x) = 1$ for all $\mu$. {\color{black} With uniform resets, the unique stationary population is $q=(1/2,1/2)$, so $\Gamma(\mu)=q$ and the unique equilibrium is $(\pi_{stay},q)$.}
\end{proposition}

\begin{proof}
{\color{black}
Fix $\mu\in\Delta_\states$. Since all rewards are nonpositive, $V^*(x)\leq0$. Always staying incurs at most $\alpha$ per step, even with resets, so $V^*(x)\geq V_{stay}(x)\geq-\alpha/(1-\gamma)$. Consequently,
\[
Q^*(x,\text{Stay})\geq-\frac{\alpha}{1-\gamma},
\qquad Q^*(x,\text{Switch})\leq-C-\alpha\mu(x)\leq-C.
\]
Thus $C>\alpha/(1-\gamma)$ makes Stay strictly optimal in both states for every $\mu$. Under this policy, population evolution is $\nu_{t+1}=(1-\eta)\nu_t+\eta q$, which contracts distances by $1-\eta<1$ and has unique stationary distribution $q$. Therefore $\Gamma(\mu)=q$ for all $\mu$, giving contraction coefficient zero and the unique equilibrium $(\pi_{stay},q)$.
}
\end{proof}

}

\section{Wall-clock Runtime (s)}
\begin{figure}[h!]
    \centering
    \includegraphics[width=0.42\linewidth]{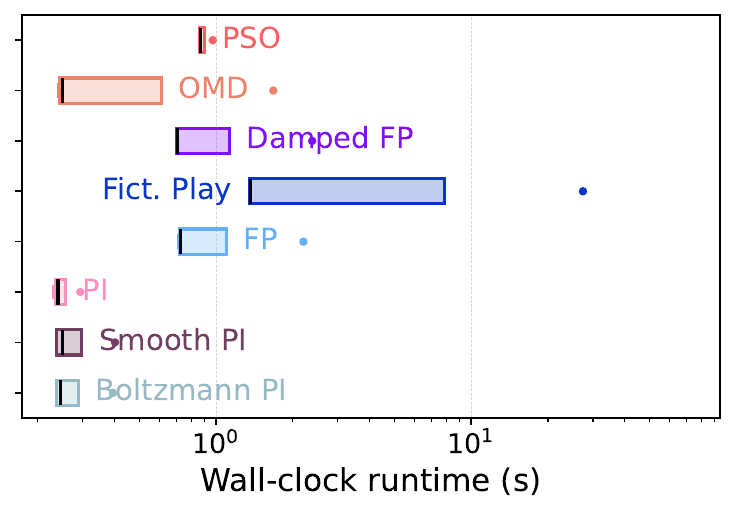}
    \includegraphics[width=0.42\linewidth]{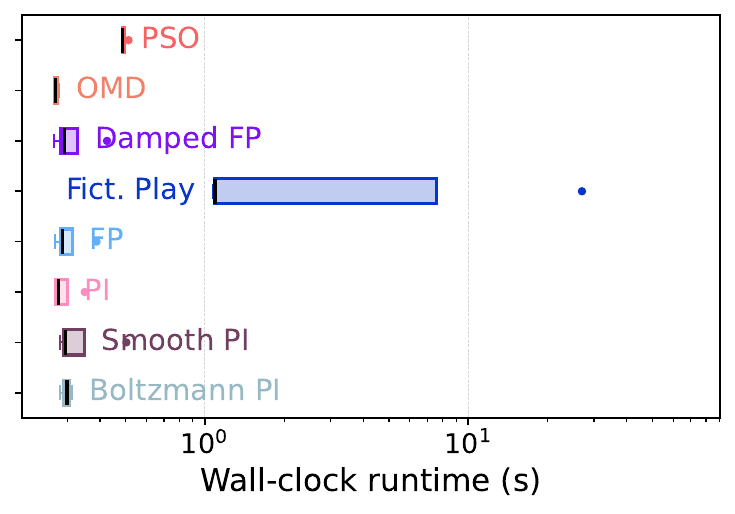}
    \caption{NI-MFG, {\color{black} Coordination Game}}
    \label{fig:wall-clock1}
\end{figure}
\begin{figure}[h!]
    \centering
    \includegraphics[width=0.42\linewidth]{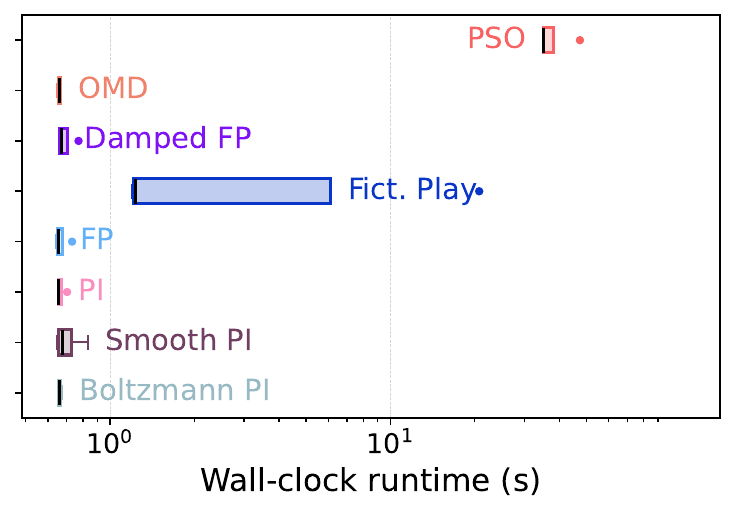}
    \includegraphics[width=0.42\linewidth]{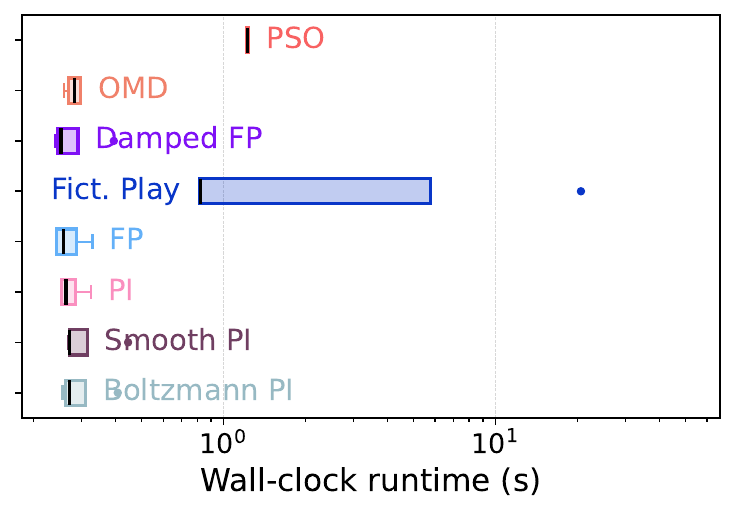}
    \caption{P-MFG, LL-MFG}
    \label{fig:wall-clock2}
\end{figure}
\begin{figure}[h!]
    \centering
    \includegraphics[width=0.42\linewidth]{figures/SISEpidemic/runtime_best_versions.pdf}
    \includegraphics[width=0.42\linewidth]{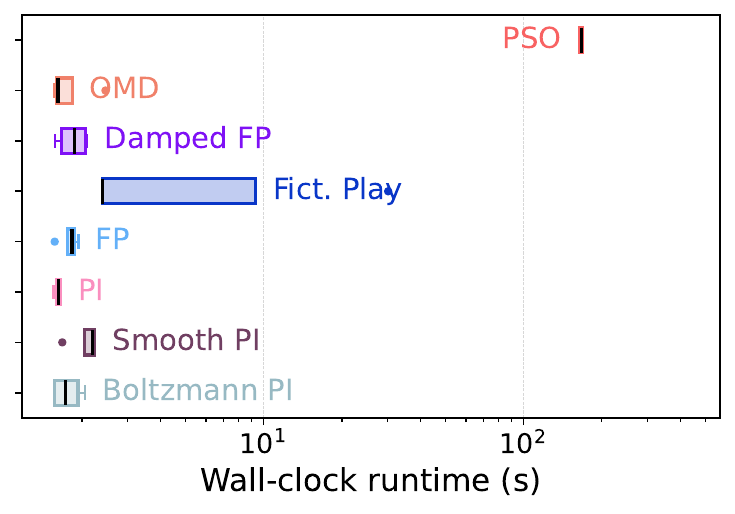}
    \caption{DC-MFG: \textit{SISEpidemic}, \textit{KineticCongestion}}
    \label{fig:wall-clock3}
\end{figure}
\begin{figure}[h!]
    \centering
    \includegraphics[width=0.42\linewidth]{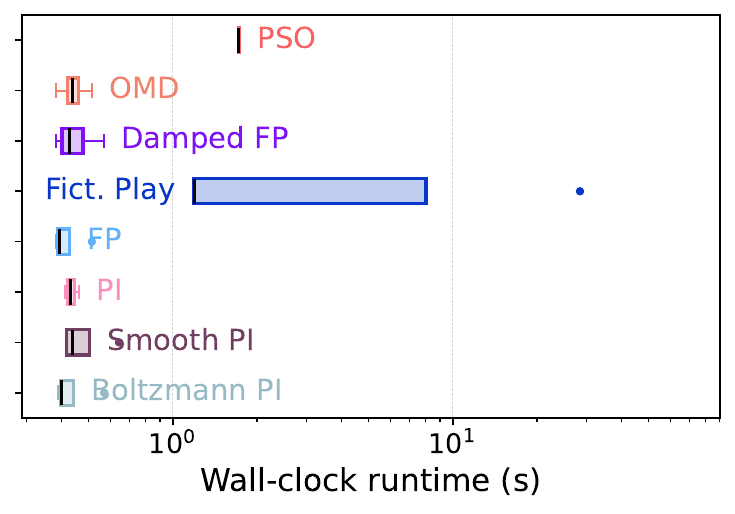}
    \includegraphics[width=0.42\linewidth]{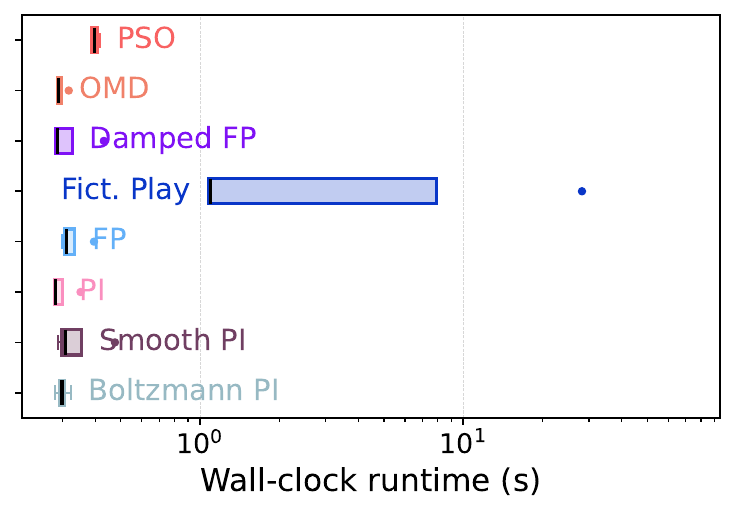}
    \caption{Multiple Equilibria, Cyclic Game}
    \label{fig:wall-clock4}
\end{figure}
\clearpage
\section{MF-PSO Particle Scalability}\label{app:pso-scalability}
We vary the number of MF-PSO particles $\mathrm{P}\in\{50,100,200,400,800\}$ on Four Rooms and Two Beach Bars (Example~\ref{ex:two-beach-bar}). Tables~\ref{tab:pso-particles-four-rooms} and~\ref{tab:pso-particles-beach-bar} report final exploitability and wall-clock time as mean $\pm$ standard deviation over four runs. The Four Rooms CPU is an Intel Xeon Platinum 8592, the Two Beach Bars CPU is an Intel Xeon Gold 6338 at 2.00 GHz, and both GPU experiments use an NVIDIA GeForce RTX 4090 D. These timings therefore complement the NVIDIA L40 results in the main experiments rather than providing a comparison on identical hardware.

On Four Rooms, increasing $\mathrm{P}$ from $50$ to $800$ reduces mean exploitability from $0.985$ to $0.642$ on CPU and from $0.852$ to $0.680$ on GPU, while runtime increases from $6.3$ to $67.1$ seconds and from $2.85$ to $5.68$ seconds, respectively. On Two Beach Bars, exploitability is already of order $10^{-5}$ at $50$ particles, so larger swarms offer little additional accuracy. CPU cost grows with particle count, whereas GPU batching substantially moderates the observed runtime increase. More particles can improve search quality, but the benefit is environment-dependent and must be assessed against wall-clock cost.

\begin{table}[htbp]
\centering
\small
\caption{\textbf{MF-PSO particle ablation on Four Rooms.} Mean $\pm$ standard deviation over four runs per configuration.}
\label{tab:pso-particles-four-rooms}
\begin{tabular}{rcccc}
\toprule
 & \multicolumn{2}{c}{CPU: Xeon Platinum 8592} & \multicolumn{2}{c}{GPU: RTX 4090 D} \\
\cmidrule(lr){2-3}\cmidrule(lr){4-5}
$\mathrm{P}$ & Final exploitability & Runtime (s) & Final exploitability & Runtime (s) \\
\midrule
50  & $0.985\pm0.038$ & $6.3\pm0.2$  & $0.852\pm0.116$ & $2.85\pm0.03$ \\
100 & $0.816\pm0.054$ & $10.9\pm0.0$ & $0.767\pm0.088$ & $3.05\pm0.03$ \\
200 & $0.757\pm0.063$ & $18.3\pm0.9$ & $0.763\pm0.071$ & $3.33\pm0.05$ \\
400 & $0.690\pm0.045$ & $33.0\pm0.4$ & $0.709\pm0.053$ & $4.36\pm0.02$ \\
800 & $0.642\pm0.058$ & $67.1\pm5.7$ & $0.680\pm0.089$ & $5.68\pm0.05$ \\
\bottomrule
\end{tabular}
\end{table}

\begin{table}[htbp]
\centering
\small
\caption{\textbf{MF-PSO particle ablation on Two Beach Bars.} Mean $\pm$ standard deviation over four runs per configuration; exploitability entries are in units of $10^{-6}$.}
\label{tab:pso-particles-beach-bar}
\begin{tabular}{rcccc}
\toprule
 & \multicolumn{2}{c}{CPU: Xeon Gold 6338} & \multicolumn{2}{c}{GPU: RTX 4090 D} \\
\cmidrule(lr){2-3}\cmidrule(lr){4-5}
$\mathrm{P}$ & Exploitability ($10^{-6}$) & Runtime (s) & Exploitability ($10^{-6}$) & Runtime (s) \\
\midrule
50  & $9.537\pm1.907$ & $0.912\pm0.014$ & $9.537\pm3.304$ & $2.228\pm0.026$ \\
100 & $9.537\pm1.907$ & $1.109\pm0.013$ & $7.629\pm0$     & $2.267\pm0.011$ \\
200 & $7.629\pm0$     & $1.895\pm0.024$ & $7.629\pm0$     & $2.270\pm0.007$ \\
400 & $7.629\pm0$     & $3.166\pm0.009$ & $7.629\pm0$     & $2.395\pm0.085$ \\
800 & $8.583\pm1.652$ & $5.810\pm0.003$ & $7.629\pm0$     & $2.394\pm0.076$ \\
\bottomrule
\end{tabular}
\end{table}
\clearpage
\section{MF-Garnet Scalability: A/M}\label{app:garnet-scalability-am}
The additive-dynamics/multiplicative-reward (A/M) setting complements the M/A study in Section~\ref{sec:pso-garnet-scalability}. We use the same state counts, six actions, branching factor six and five seeds per setting. Figure~\ref{fig:garnet-scalability-am} reports final exploitability and wall-clock time. At $400$ states, FP, Damped FP, Fictitious Play and PI variants other than Boltzmann PI retain mean exploitability below $0.02$, whereas MF-PSO reaches approximately $5.4$. MF-PSO takes approximately $422$ seconds, compared with $175$--$181$ seconds for the other methods except Fictitious Play ($277$ seconds). The scalability limitation of MF-PSO is therefore also present in this second regime.

\begin{figure}[htbp]
\centering
\begingroup
\begin{subfigure}[t]{0.49\linewidth}
\centering
\begin{tikzpicture}
\begin{axis}[width=\linewidth, height=4.15cm, xmode=log, ymode=log,
  log basis x=10, xmin=8, xmax=500, xtick={10,20,80,130,400},
  xticklabels={10,20,80,130,400}, xlabel={Number of states $|\states|$},
  tick label style={font=\scriptsize}, label style={font=\small},
  grid=major, grid style={gray!20}, minor tick num=0,
  legend style={font=\scriptsize, draw=none},
  legend columns=4,ymin=1e-6, ymax=20,ylabel={Final exploitability},legend to name=garnet-scaling-am-legend,]
\addplot+[blue!75!black, dashed, thick, mark=triangle*, mark size=1.7pt,
 error bars/.cd, y dir=both, y explicit] coordinates {
(10,0.00606747) += (0,0.0125409) -= (0,0.00606647)
(20,0.000906599) += (0,0.00156302) -= (0,0.000905599)
(80,0.00704374) += (0,0.00468042) -= (0,0.00468042)
(130,0.00966048) += (0,0.00915803) -= (0,0.00915803)
(400,0.0147514) += (0,0.0137644) -= (0,0.0137644)
};
\addlegendentry{FP}
\addplot+[cyan!70!black, densely dashed, thick, mark=*, mark size=1.7pt, mark options={draw=cyan!70!black, fill=cyan!70!black},
 error bars/.cd, y dir=both, y explicit] coordinates {
(10,0.0436184) += (0,0.09106) -= (0,0.0436174)
(20,1e-06) += (0,0) -= (0,0)
(80,0.00294679) += (0,0.00331618) -= (0,0.00294579)
(130,0.00647513) += (0,0.00573392) -= (0,0.00573392)
(400,0.00842253) += (0,0.00622867) -= (0,0.00622867)
};
\addlegendentry{Damped FP}
\addplot+[green!50!black, dashdotted, thick, mark=diamond*, mark size=1.7pt,
 error bars/.cd, y dir=both, y explicit] coordinates {
(10,0.00729733) += (0,0.0136728) -= (0,0.00729633)
(20,0.00374103) += (0,0.00515119) -= (0,0.00374003)
(80,0.00161573) += (0,0.00136297) -= (0,0.00136297)
(130,0.00646074) += (0,0.00623934) -= (0,0.00623934)
(400,0.00986813) += (0,0.00748221) -= (0,0.00748221)
};
\addlegendentry{Fictitious Play}
\addplot+[violet, dotted, thick, mark=*, mark size=1.7pt, mark options={draw=violet, fill=violet},
 error bars/.cd, y dir=both, y explicit] coordinates {
(10,0.0756387) += (0,0.161999) -= (0,0.0756377)
(20,0.00576452) += (0,0.0125451) -= (0,0.00576352)
(80,0.00454777) += (0,0.00445212) -= (0,0.00445212)
(130,0.00845829) += (0,0.00849602) -= (0,0.00845729)
(400,0.0196889) += (0,0.014977) -= (0,0.014977)
};
\addlegendentry{PI}
\addplot+[brown, densely dotted, thick, mark=*, mark size=1.7pt, mark options={draw=brown, fill=brown},
 error bars/.cd, y dir=both, y explicit] coordinates {
(10,0.0092597) += (0,0.0196204) -= (0,0.0092587)
(20,0.000604582) += (0,0.00135255) -= (0,0.000603582)
(80,0.00131841) += (0,0.00186621) -= (0,0.00131741)
(130,0.00909827) += (0,0.00879222) -= (0,0.00879222)
(400,0.00815877) += (0,0.00613255) -= (0,0.00613255)
};
\addlegendentry{Smooth PI}
\addplot+[magenta!70!black, dotted, thick, mark=star, mark size=1.7pt,
 error bars/.cd, y dir=both, y explicit] coordinates {
(10,0.618112) += (0,0.273146) -= (0,0.273146)
(20,0.557052) += (0,0.241105) -= (0,0.241105)
(80,0.58735) += (0,0.174131) -= (0,0.174131)
(130,0.585907) += (0,0.199226) -= (0,0.199226)
(400,0.617998) += (0,0.232178) -= (0,0.232178)
};
\addlegendentry{Boltzmann PI}
\addplot+[orange!90!black, solid, thick, mark=square*, mark size=1.7pt, mark options={draw=orange!90!black, fill=white},
 error bars/.cd, y dir=both, y explicit] coordinates {
(10,0.950878) += (0,0.399232) -= (0,0.399232)
(20,0.818714) += (0,0.259054) -= (0,0.259054)
(80,0.860555) += (0,0.227239) -= (0,0.227239)
(130,0.841501) += (0,0.242947) -= (0,0.242947)
(400,0.893079) += (0,0.272592) -= (0,0.272592)
};
\addlegendentry{OMD}
\addplot+[red!75!black, solid, thick, mark=*, mark size=1.7pt, mark options={draw=red!75!black, fill=white},
 error bars/.cd, y dir=both, y explicit] coordinates {
(10,0.415826) += (0,0.337958) -= (0,0.337958)
(20,1.08492) += (0,0.920579) -= (0,0.920579)
(80,3.53446) += (0,2.80921) -= (0,2.80921)
(130,3.98609) += (0,3.09675) -= (0,3.09675)
(400,5.3703) += (0,4.22447) -= (0,4.22447)
};
\addlegendentry{MF-PSO}
\addplot[gray, dashdotted, forget plot, domain=8:500] {1e-6};
\end{axis}
\end{tikzpicture}
\caption{Final exploitability}
\end{subfigure}
\hfill
\begin{subfigure}[t]{0.49\linewidth}
\centering
\begin{tikzpicture}
\begin{axis}[width=\linewidth, height=4.15cm, xmode=log, ymode=log,
  log basis x=10, xmin=8, xmax=500, xtick={10,20,80,130,400},
  xticklabels={10,20,80,130,400}, xlabel={Number of states $|\states|$},
  tick label style={font=\scriptsize}, label style={font=\small},
  grid=major, grid style={gray!20}, minor tick num=0,
  legend style={font=\scriptsize, draw=none},
  legend columns=4,ymin=1, ymax=800,ylabel={Wall-clock time (s)},]
\addplot+[blue!75!black, dashed, thick, mark=triangle*, mark size=1.7pt,
 error bars/.cd, y dir=both, y explicit] coordinates {
(10,2.056) += (0,0.16) -= (0,0.16)
(20,2.306) += (0,0.116) -= (0,0.116)
(80,3.288) += (0,0.15) -= (0,0.15)
(130,8.426) += (0,0.585) -= (0,0.585)
(400,180.453) += (0,8.474) -= (0,8.474)
};
\addplot+[cyan!70!black, densely dashed, thick, mark=*, mark size=1.7pt, mark options={draw=cyan!70!black, fill=cyan!70!black},
 error bars/.cd, y dir=both, y explicit] coordinates {
(10,2.134) += (0,0.139) -= (0,0.139)
(20,2.564) += (0,0.113) -= (0,0.113)
(80,3.385) += (0,0.316) -= (0,0.316)
(130,7.88) += (0,0.205) -= (0,0.205)
(400,177.684) += (0,9.234) -= (0,9.234)
};
\addplot+[green!50!black, dashdotted, thick, mark=diamond*, mark size=1.7pt,
 error bars/.cd, y dir=both, y explicit] coordinates {
(10,83.937) += (0,1.757) -= (0,1.757)
(20,87.59) += (0,3.104) -= (0,3.104)
(80,91.359) += (0,2.638) -= (0,2.638)
(130,97.692) += (0,1.312) -= (0,1.312)
(400,277.044) += (0,4.189) -= (0,4.189)
};
\addplot+[violet, dotted, thick, mark=*, mark size=1.7pt, mark options={draw=violet, fill=violet},
 error bars/.cd, y dir=both, y explicit] coordinates {
(10,1.835) += (0,0.073) -= (0,0.073)
(20,2.168) += (0,0.251) -= (0,0.251)
(80,3.037) += (0,0.209) -= (0,0.209)
(130,7.73) += (0,0.779) -= (0,0.779)
(400,175.1) += (0,6.833) -= (0,6.833)
};
\addplot+[brown, densely dotted, thick, mark=*, mark size=1.7pt, mark options={draw=brown, fill=brown},
 error bars/.cd, y dir=both, y explicit] coordinates {
(10,2.191) += (0,0.095) -= (0,0.095)
(20,2.587) += (0,0.199) -= (0,0.199)
(80,3.447) += (0,0.176) -= (0,0.176)
(130,8.042) += (0,0.517) -= (0,0.517)
(400,181.432) += (0,10.206) -= (0,10.206)
};
\addplot+[magenta!70!black, dotted, thick, mark=star, mark size=1.7pt,
 error bars/.cd, y dir=both, y explicit] coordinates {
(10,2.171) += (0,0.11) -= (0,0.11)
(20,2.547) += (0,0.22) -= (0,0.22)
(80,3.461) += (0,0.357) -= (0,0.357)
(130,8.576) += (0,0.671) -= (0,0.671)
(400,178.281) += (0,2.975) -= (0,2.975)
};
\addplot+[orange!90!black, solid, thick, mark=square*, mark size=1.7pt, mark options={draw=orange!90!black, fill=white},
 error bars/.cd, y dir=both, y explicit] coordinates {
(10,2.352) += (0,0.236) -= (0,0.236)
(20,2.775) += (0,0.265) -= (0,0.265)
(80,3.866) += (0,0.194) -= (0,0.194)
(130,8.618) += (0,0.405) -= (0,0.405)
(400,181.005) += (0,6.099) -= (0,6.099)
};
\addplot+[red!75!black, solid, thick, mark=*, mark size=1.7pt, mark options={draw=red!75!black, fill=white},
 error bars/.cd, y dir=both, y explicit] coordinates {
(10,5.584) += (0,0.135) -= (0,0.135)
(20,6.469) += (0,0.212) -= (0,0.212)
(80,18.457) += (0,1.214) -= (0,1.214)
(130,37.167) += (0,1.974) -= (0,1.974)
(400,422.394) += (0,27.545) -= (0,27.545)
};
\end{axis}
\end{tikzpicture}
\caption{Wall-clock time}
\end{subfigure}
\par\medskip
\pgfplotslegendfromname{garnet-scaling-am-legend}
\endgroup
\caption{\textbf{MF-Garnet scalability (A/M).} Additive dynamics and multiplicative rewards (A/M). Final exploitability and wall-clock time on an NVIDIA L40 GPU, with six actions and branching factor six. Points and error bars show mean and standard deviation across five seeds. Exploitability means below $10^{-6}$ are displayed at $10^{-6}$; error-bar endpoints below this plotting floor are clipped to it.}
\label{fig:garnet-scalability-am}
\end{figure}

\clearpage
\section{Hyperparameter Sweep}\label{app: sweep}
\centering
\begin{tabular}{ll}
\toprule
Algorithm & Hyperparameter Space \\
\midrule
PSO & $w \in \mathbb{R}^+,\ c_1 \in \mathbb{R}^+,\ c_2 \in \mathbb{R}^+,\ \tau \in \mathbb{R}^+$ \\
OMD & $\alpha \in \mathbb{R}^+,\ \tau \in \mathbb{R}^+$ \\
Damped FP & $\lambda \in (0,1]$ \\
Fict. Play & --- \\
FP & --- \\
PI & --- \\
Smooth PI & $\lambda \in (0,1]$ \\
Boltzmann PI & $\tau \in \mathbb{R}^+$ \\
\bottomrule
\end{tabular}

\subsection{No Interaction MFG (NI-MFG)}
\begin{figure}[h!]
    \centering
    \includegraphics[width=0.49\linewidth]{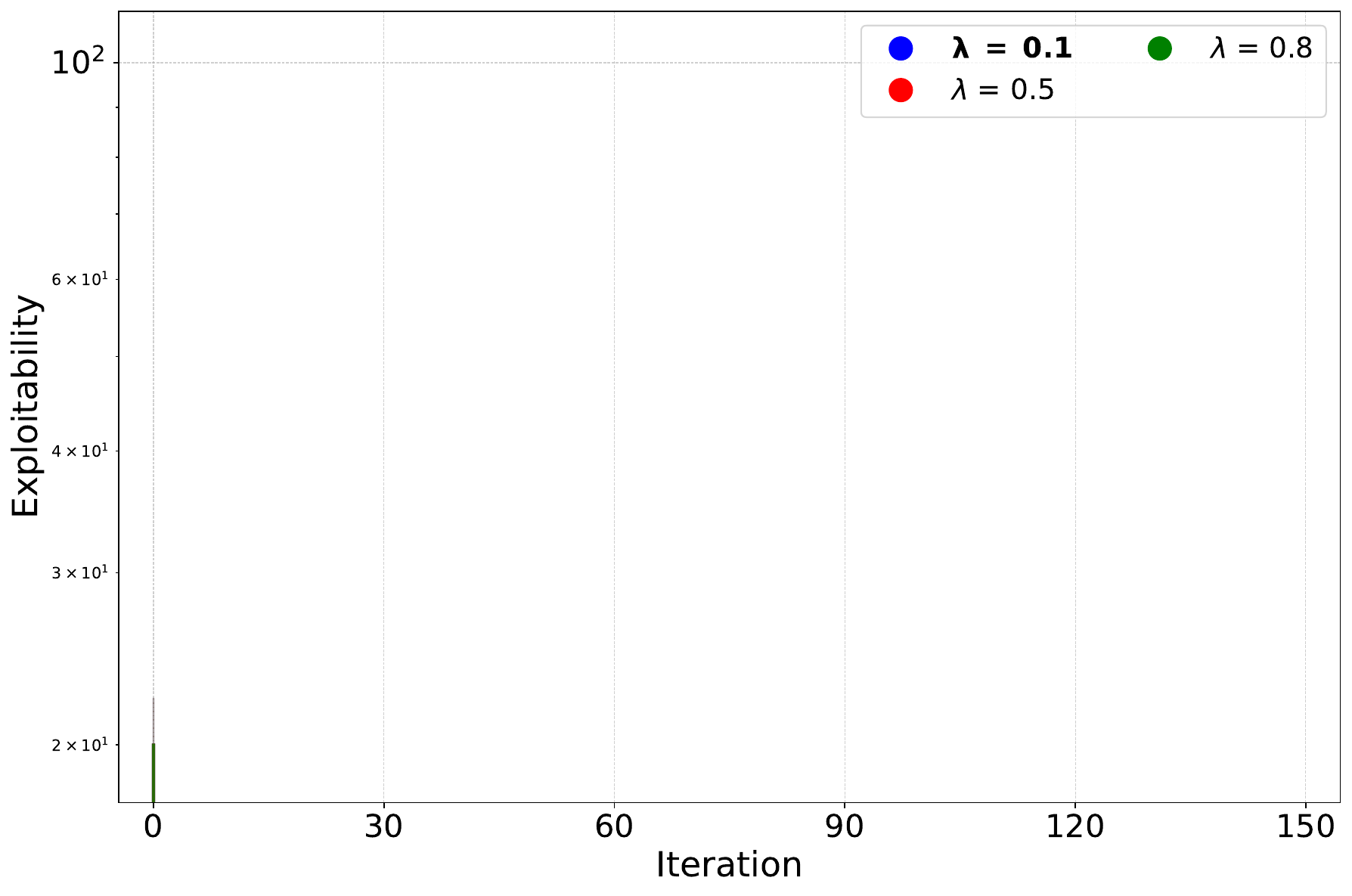}
    \includegraphics[width=0.49\linewidth]{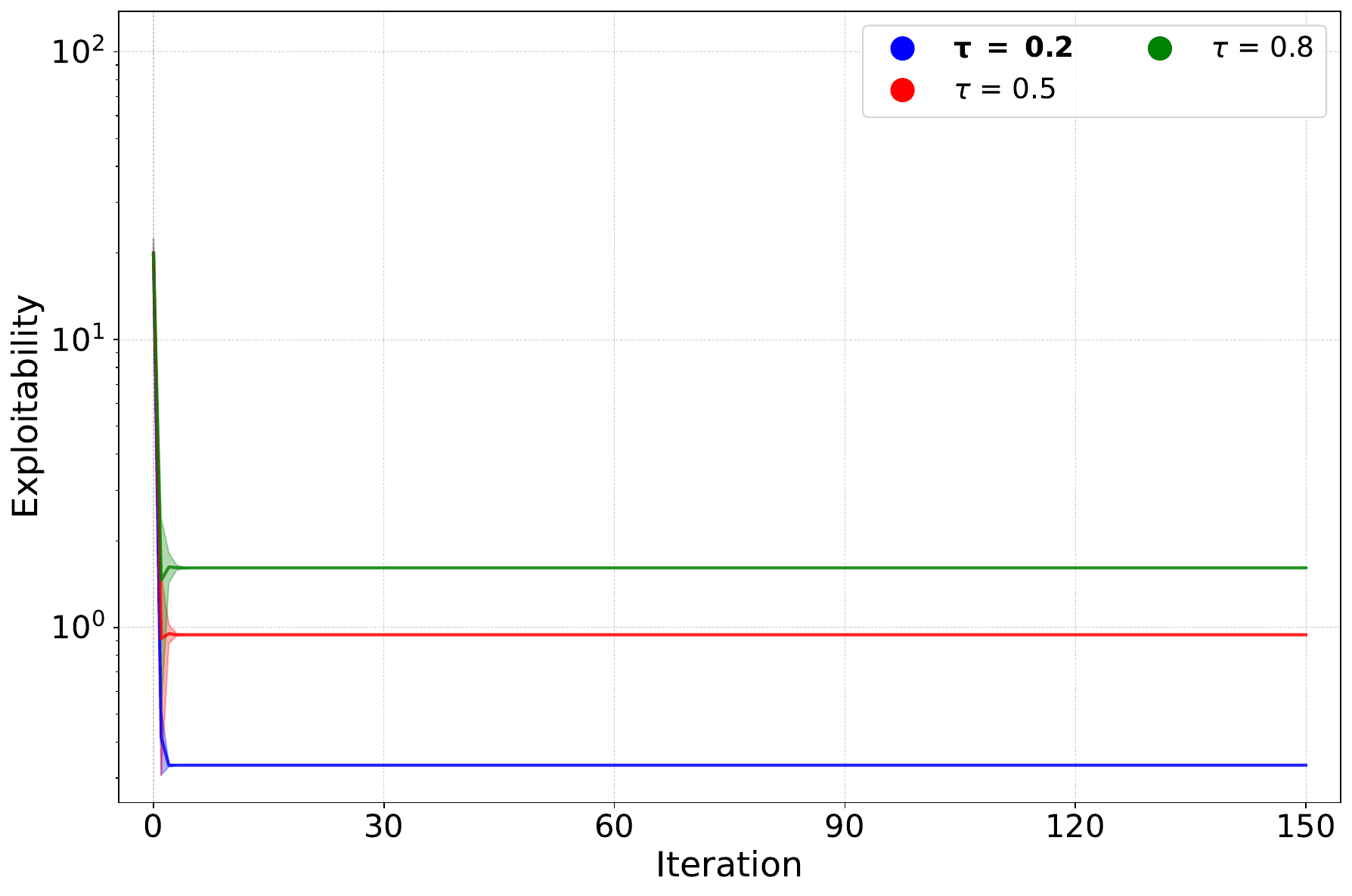}
    \includegraphics[width=0.49\linewidth]{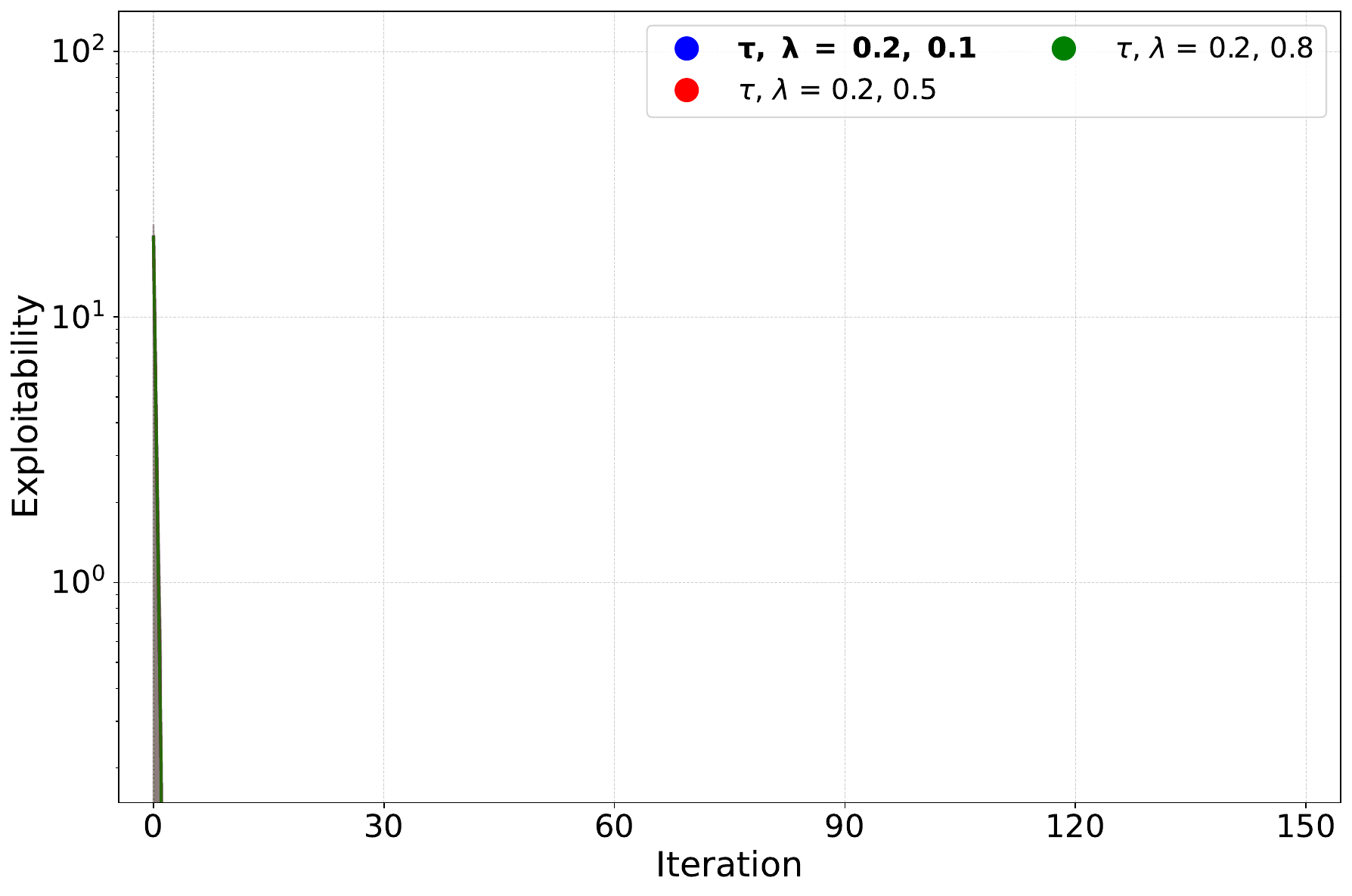}
    \includegraphics[width=0.49\linewidth]{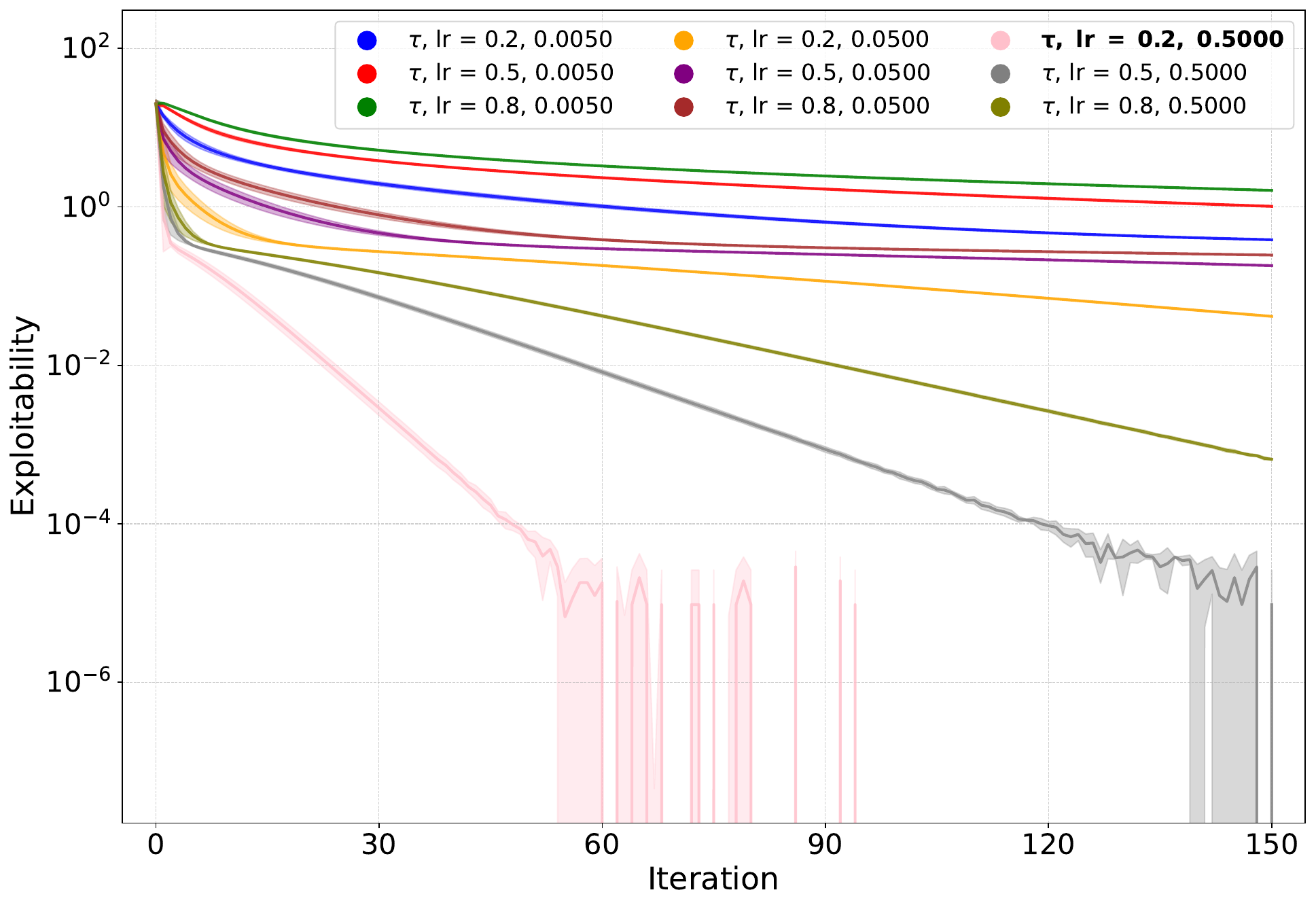}
    \includegraphics[width=0.49\linewidth]{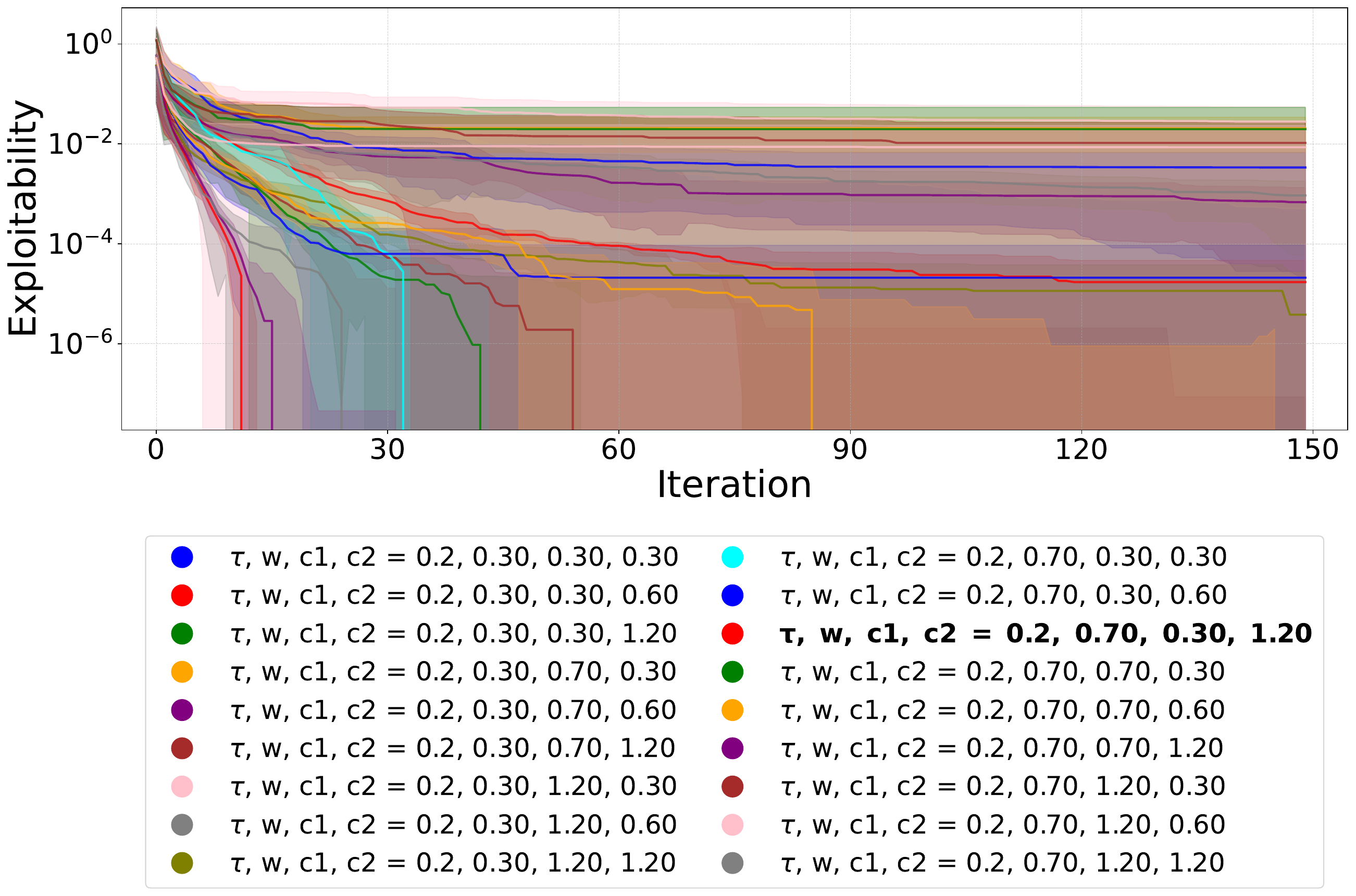}
    \includegraphics[width=0.49\linewidth]{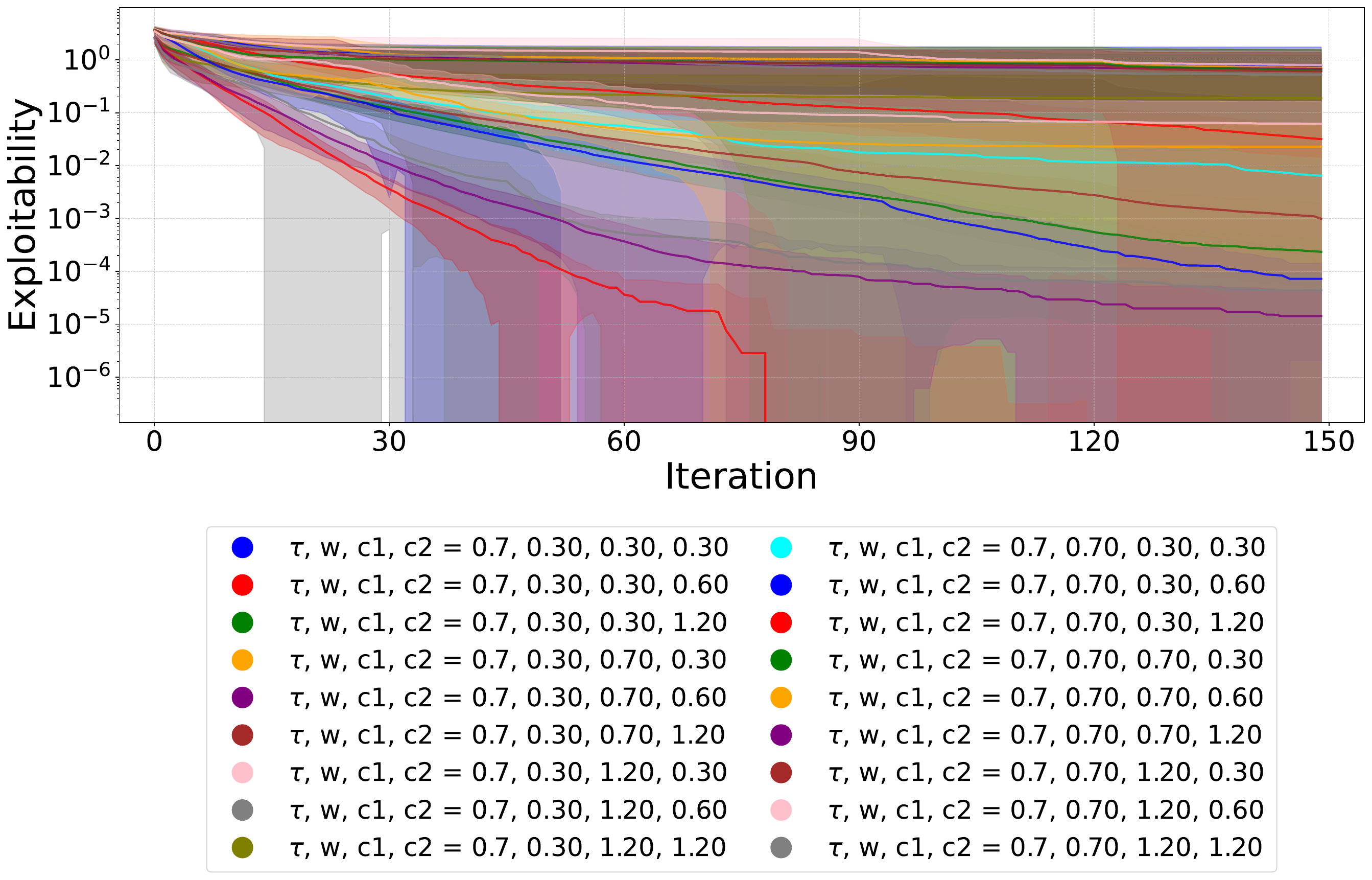}
    \caption{\textbf{NI-MFG}. Sensitivity of the algorithms wrt the hyperparameters}
    \label{fig:NI-MFG sweep}
\end{figure}

\newpage
\subsection{{\color{black} Coordination Game Benchmark}}
\begin{figure}[h!]
    \centering
    \includegraphics[width=0.49\linewidth]{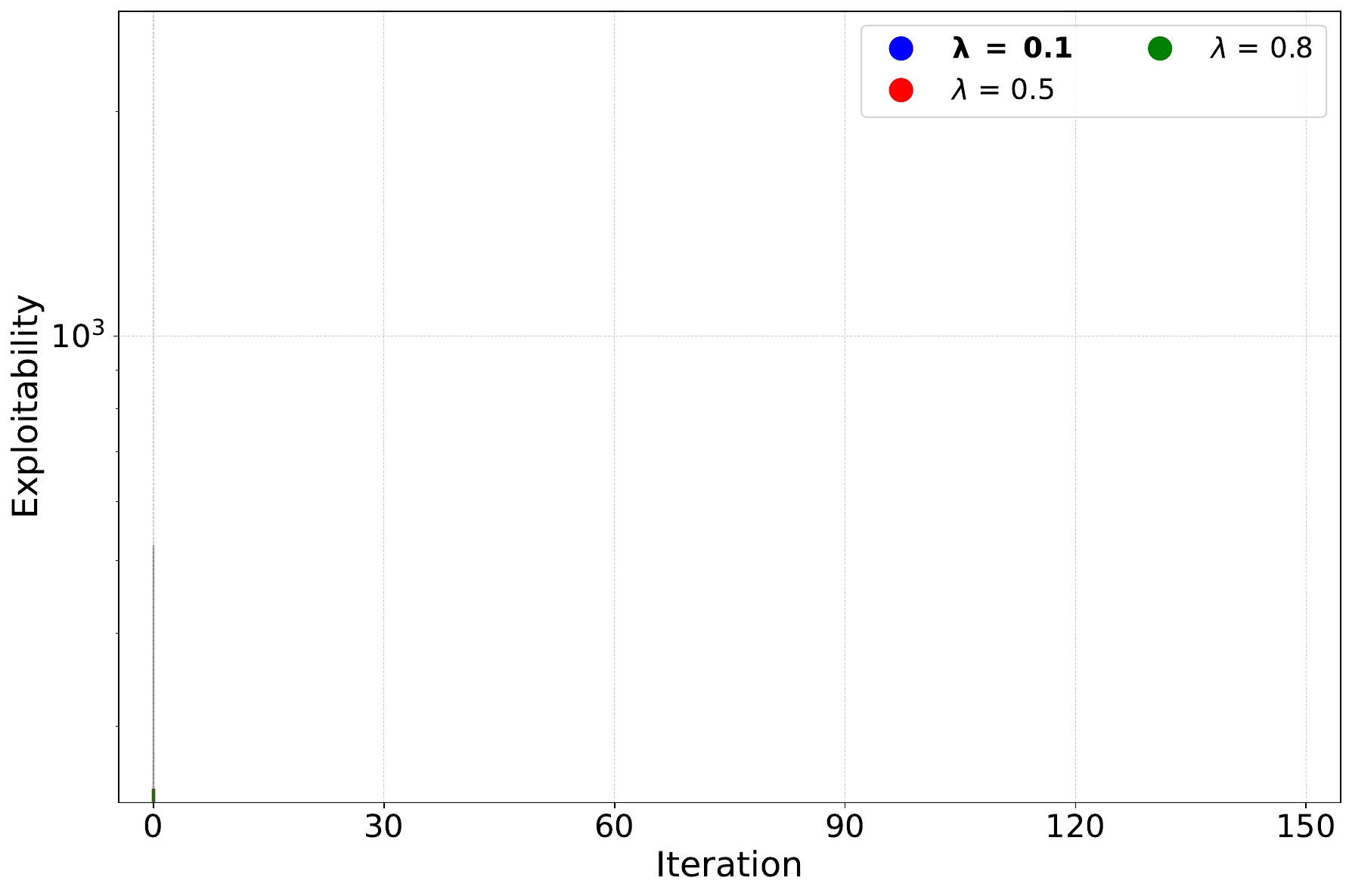}
    \includegraphics[width=0.49\linewidth]{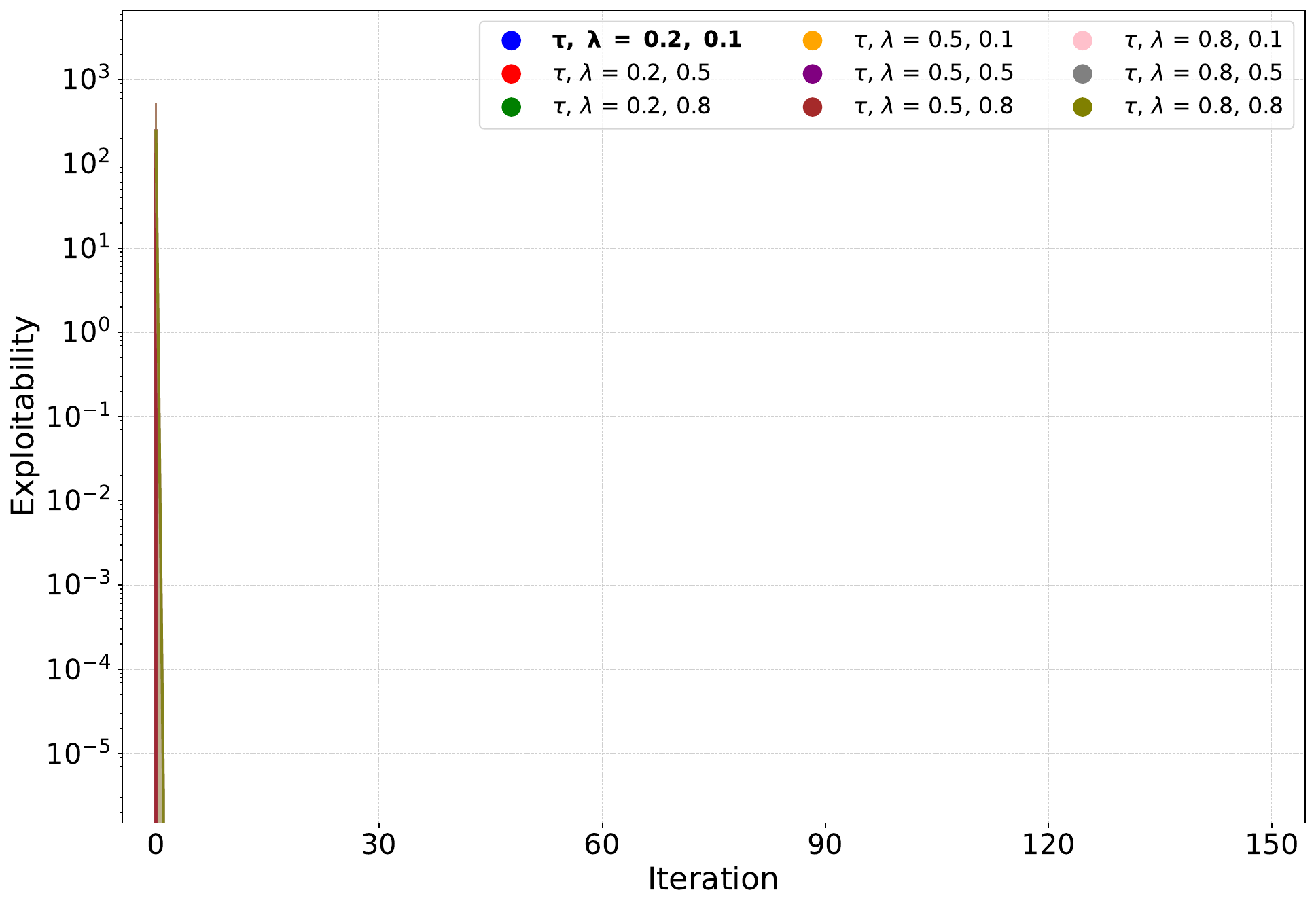}
    \includegraphics[width=0.49\linewidth]{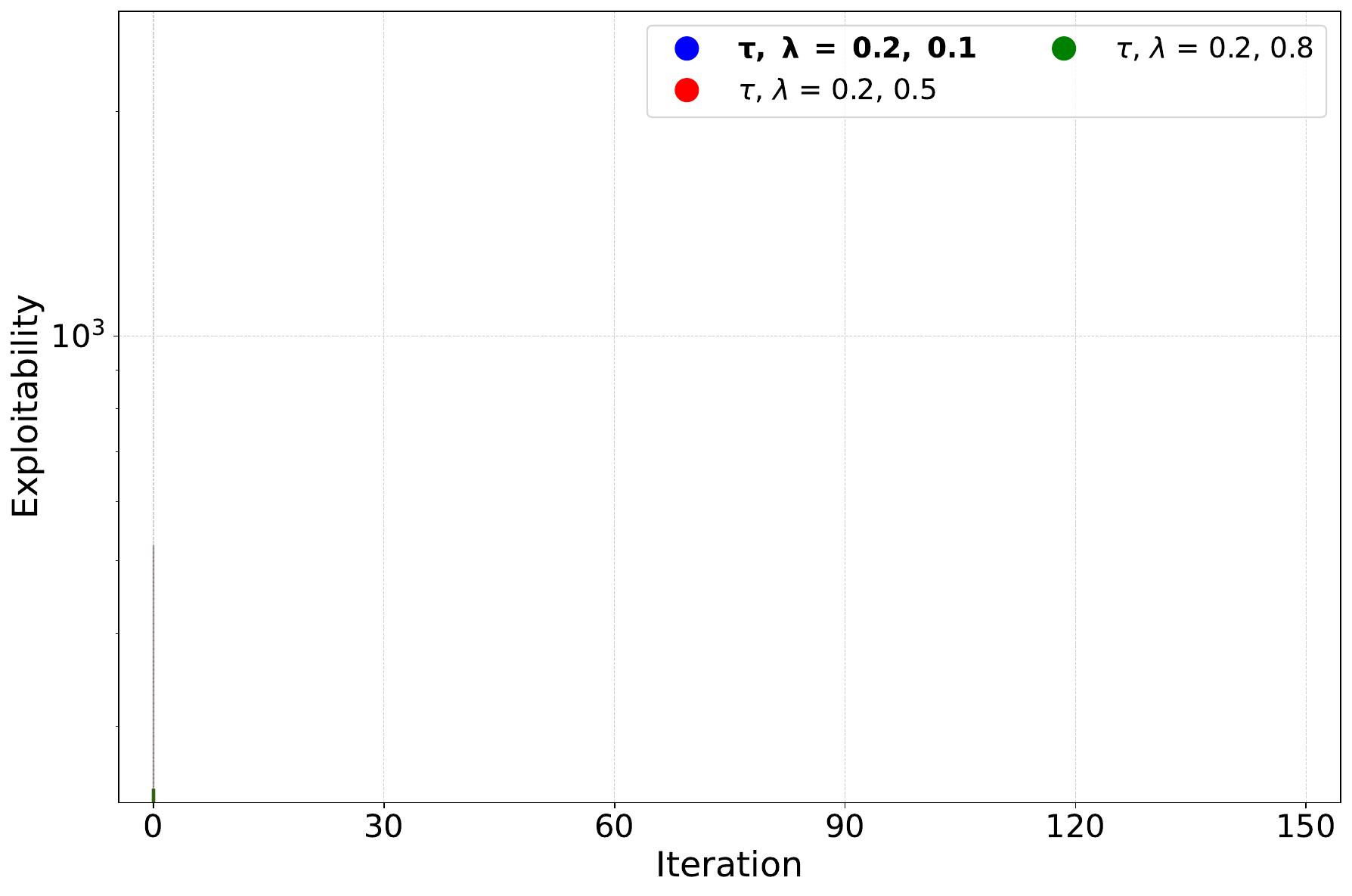}
    \includegraphics[width=0.49\linewidth]{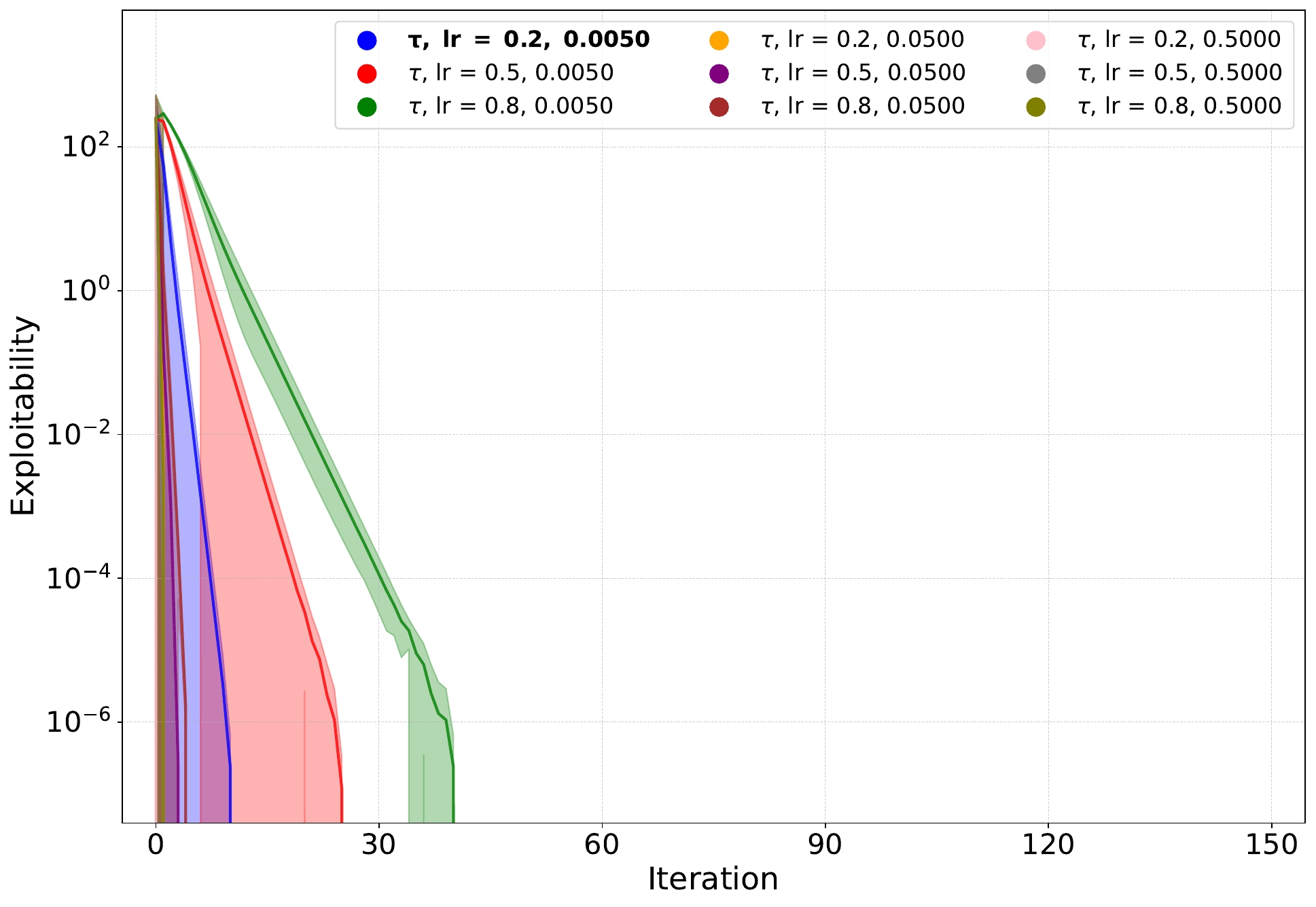}
    \includegraphics[width=0.49\linewidth]{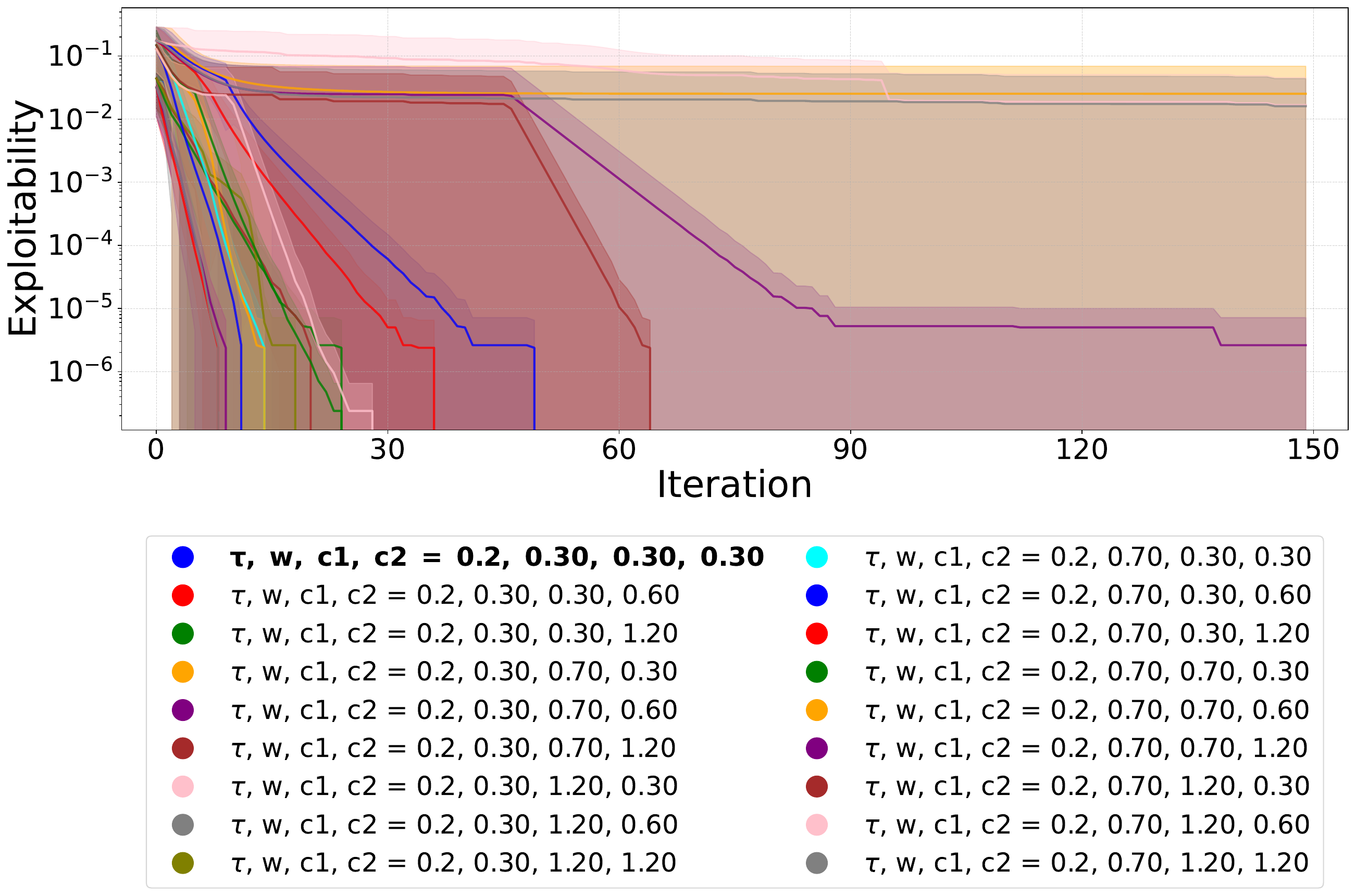}
    \includegraphics[width=0.49\linewidth]{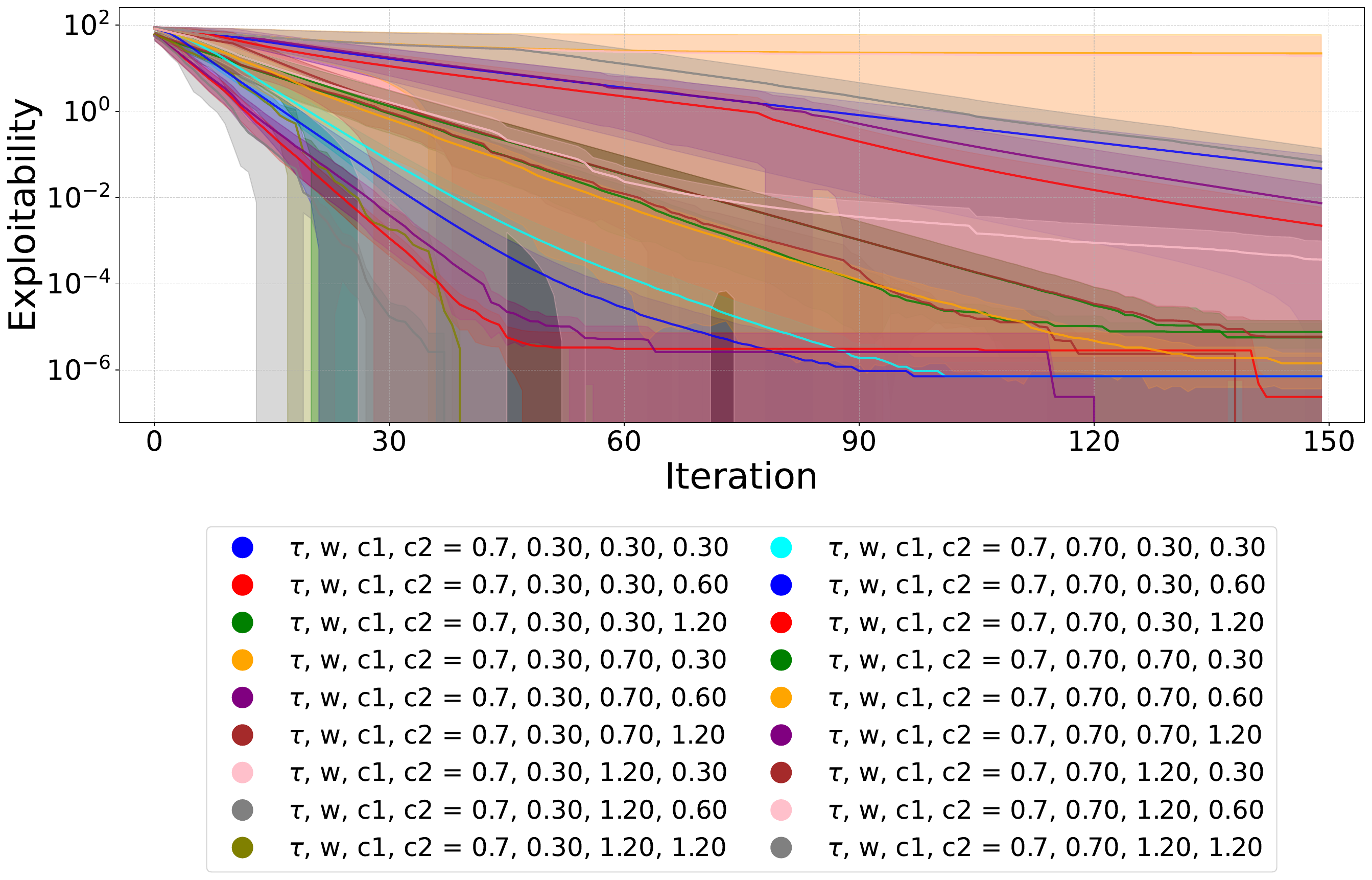}
    \caption{{\color{black}\textbf{Coordination Game}}. Sensitivity of the algorithms wrt the hyperparameters}
    \label{fig:C-MFG sweep}
\end{figure}
\newpage
\subsection{Lasry-Lions MFG (LL-MFG)}
\begin{figure}[h!]
    \centering
    \includegraphics[width=0.49\linewidth]{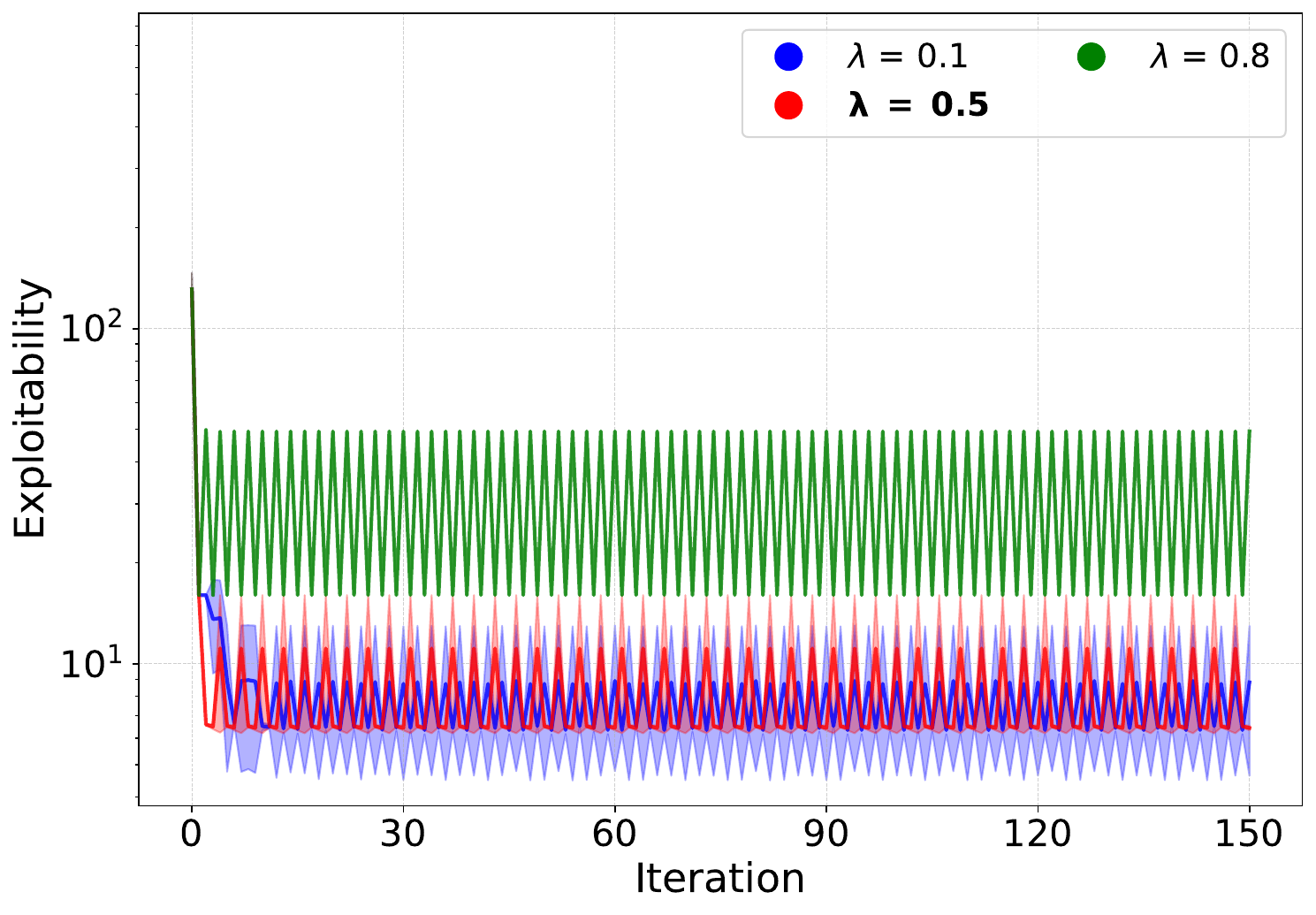}
    \includegraphics[width=0.49\linewidth]{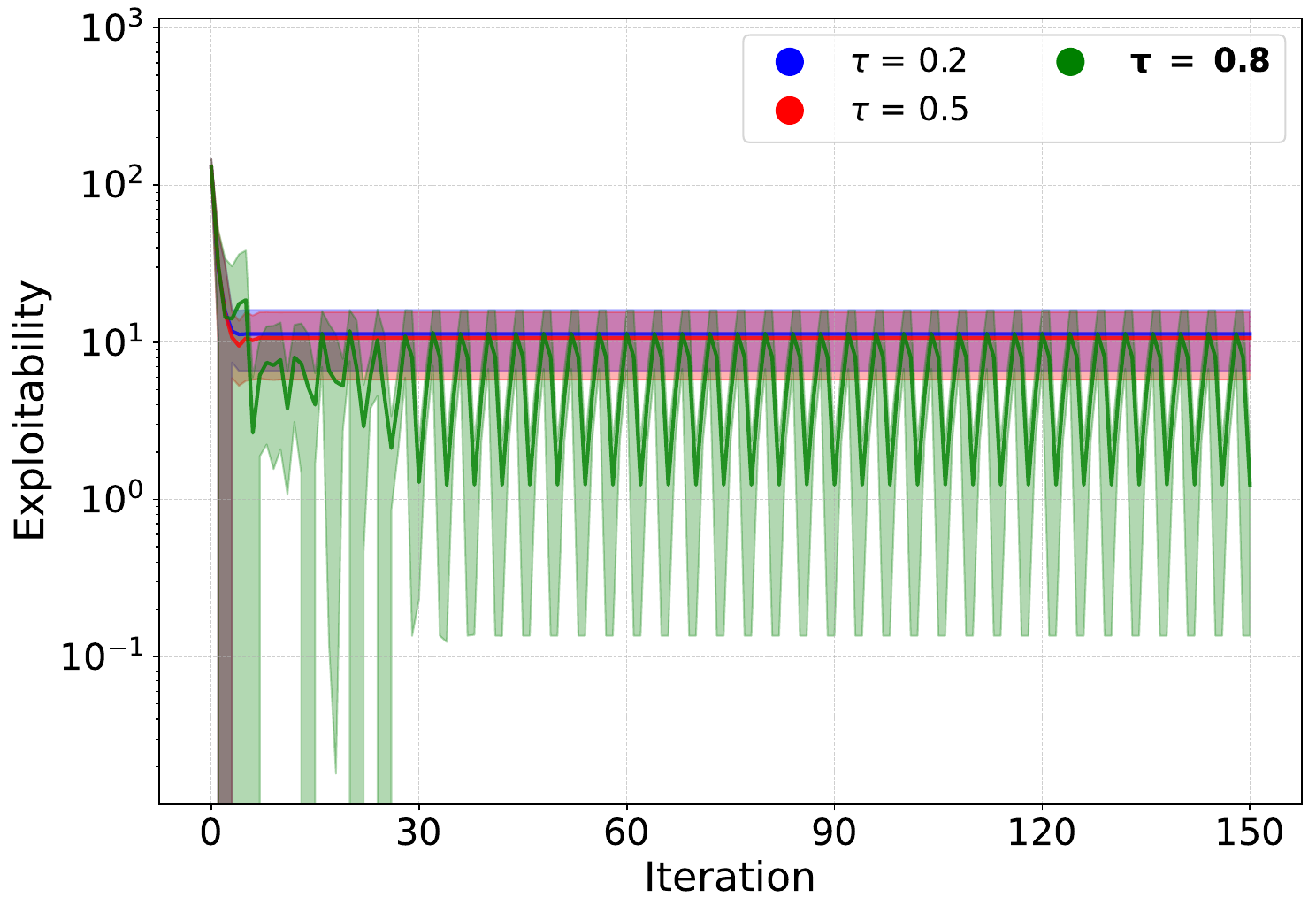}
    \includegraphics[width=0.49\linewidth]{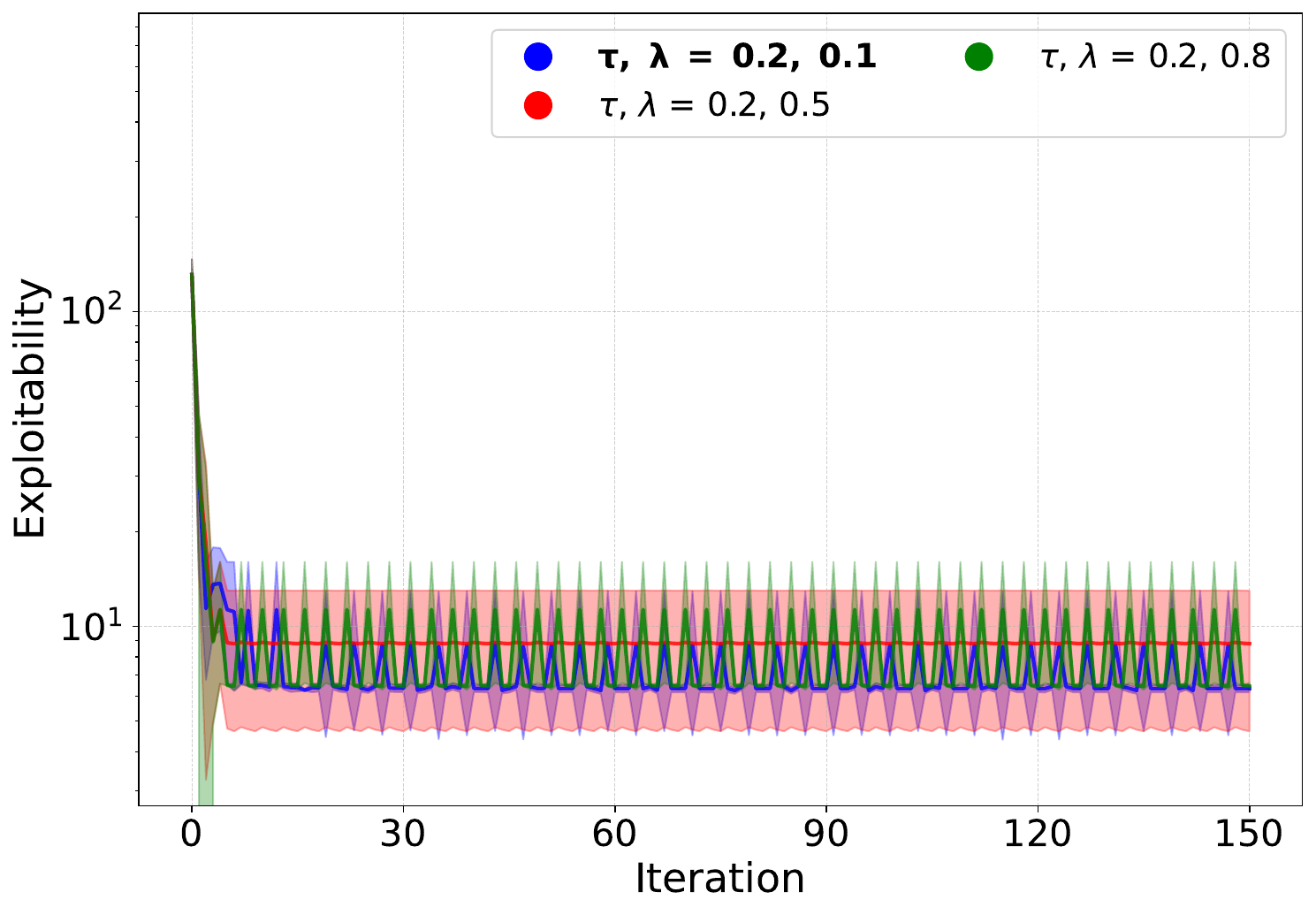}
    \includegraphics[width=0.49\linewidth]{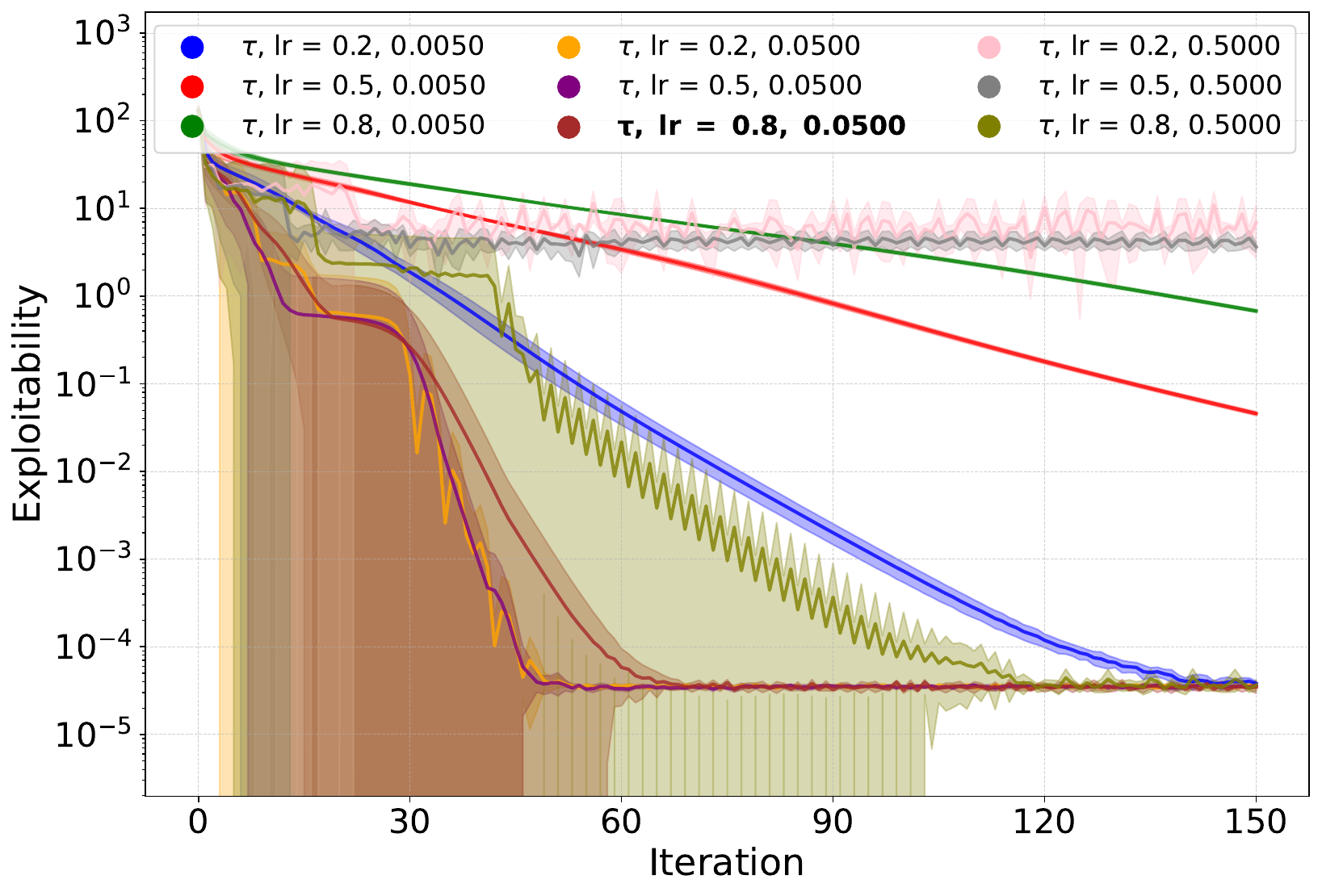}
    \includegraphics[width=0.49\linewidth]{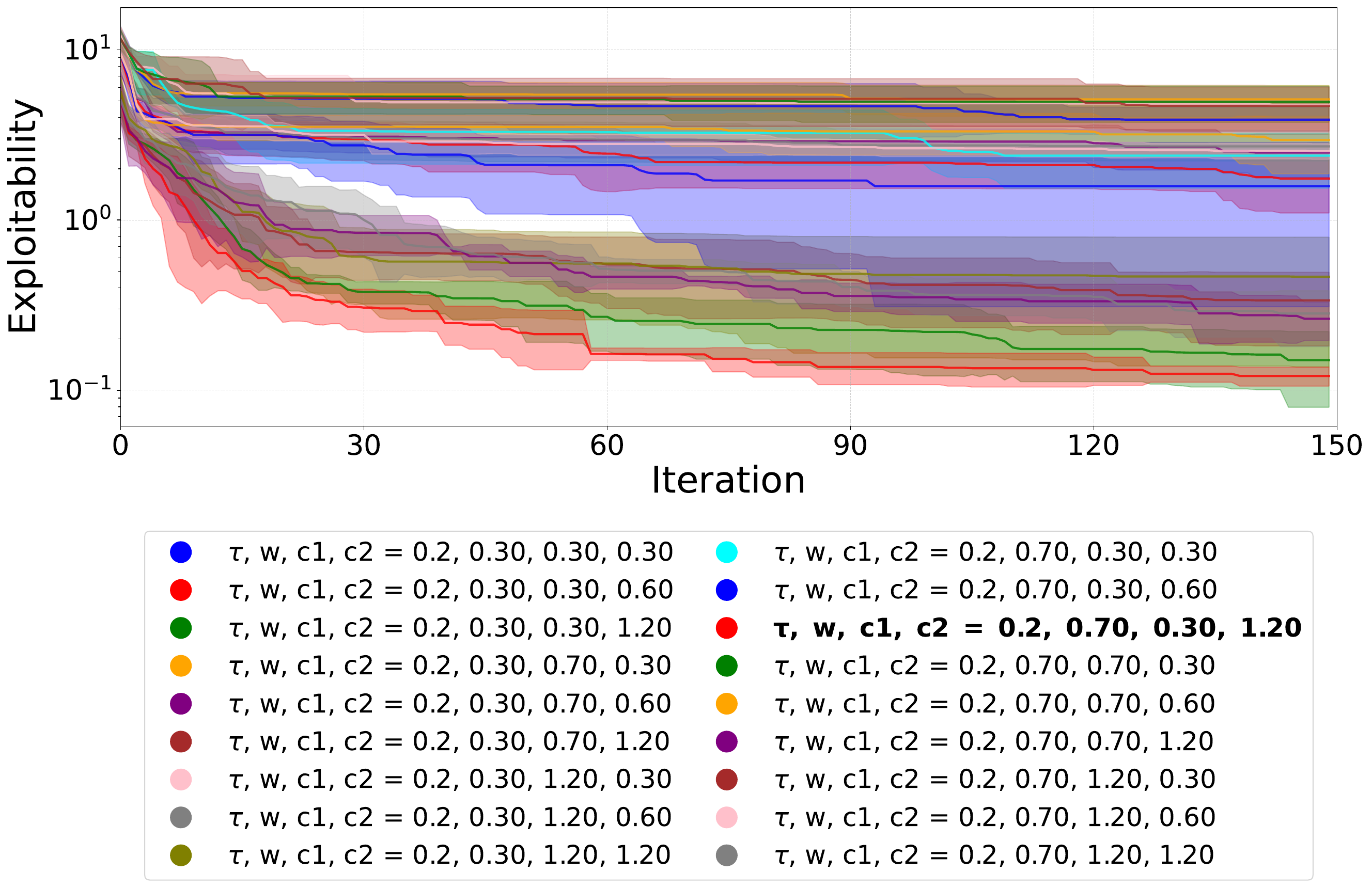}
    \includegraphics[width=0.49\linewidth]{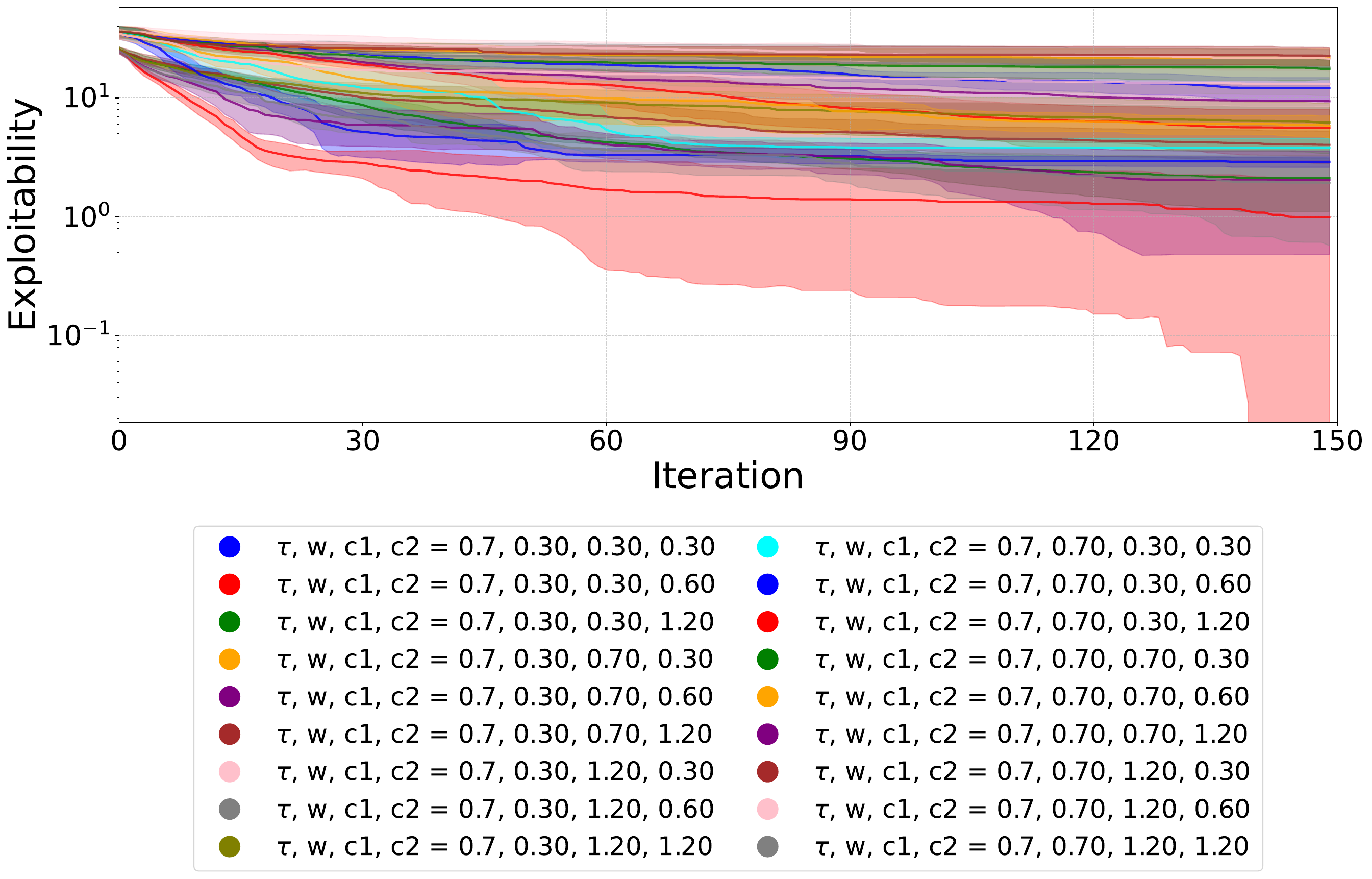}
    \caption{\textbf{LL-MFG}. Sensitivity of the algorithms wrt the hyperparameters}
    \label{fig:LL-MFG sweep}
\end{figure}

\newpage
\textbf{Multiple-Equilibria (violating LL-MFG)}
\begin{figure}[h!]
    \centering
    \includegraphics[width=0.49\linewidth]{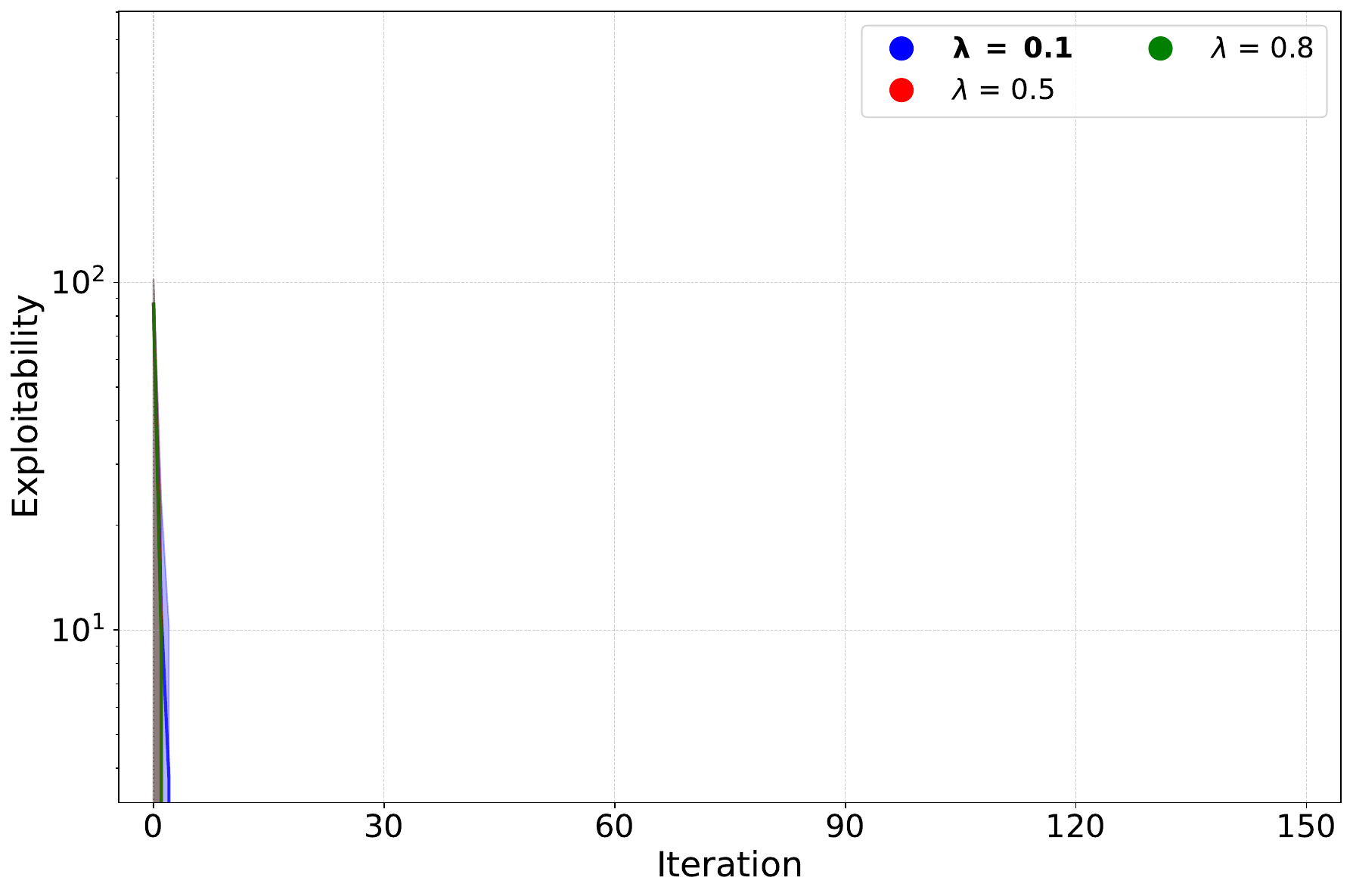}
    \includegraphics[width=0.49\linewidth]{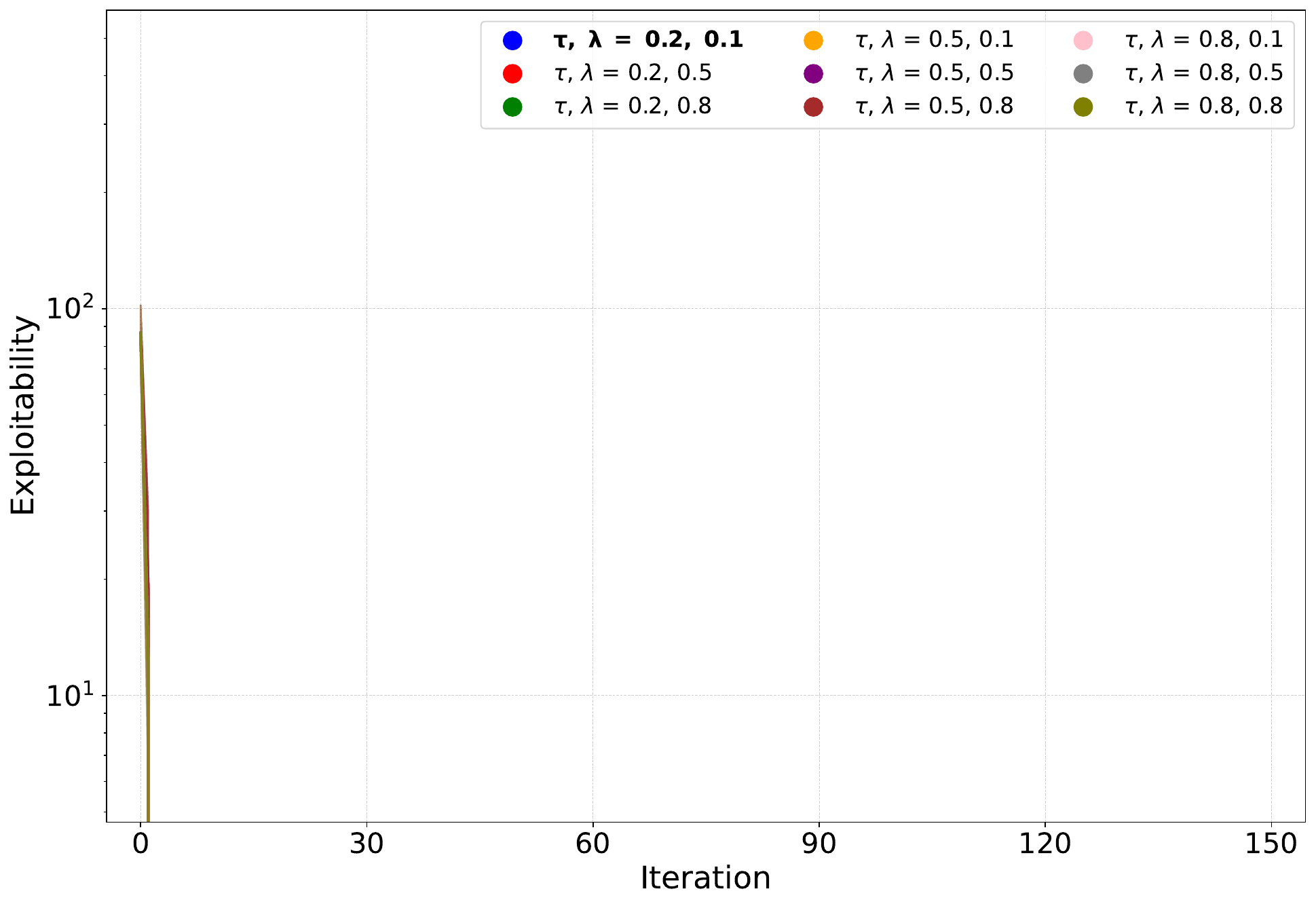}
    \includegraphics[width=0.49\linewidth]{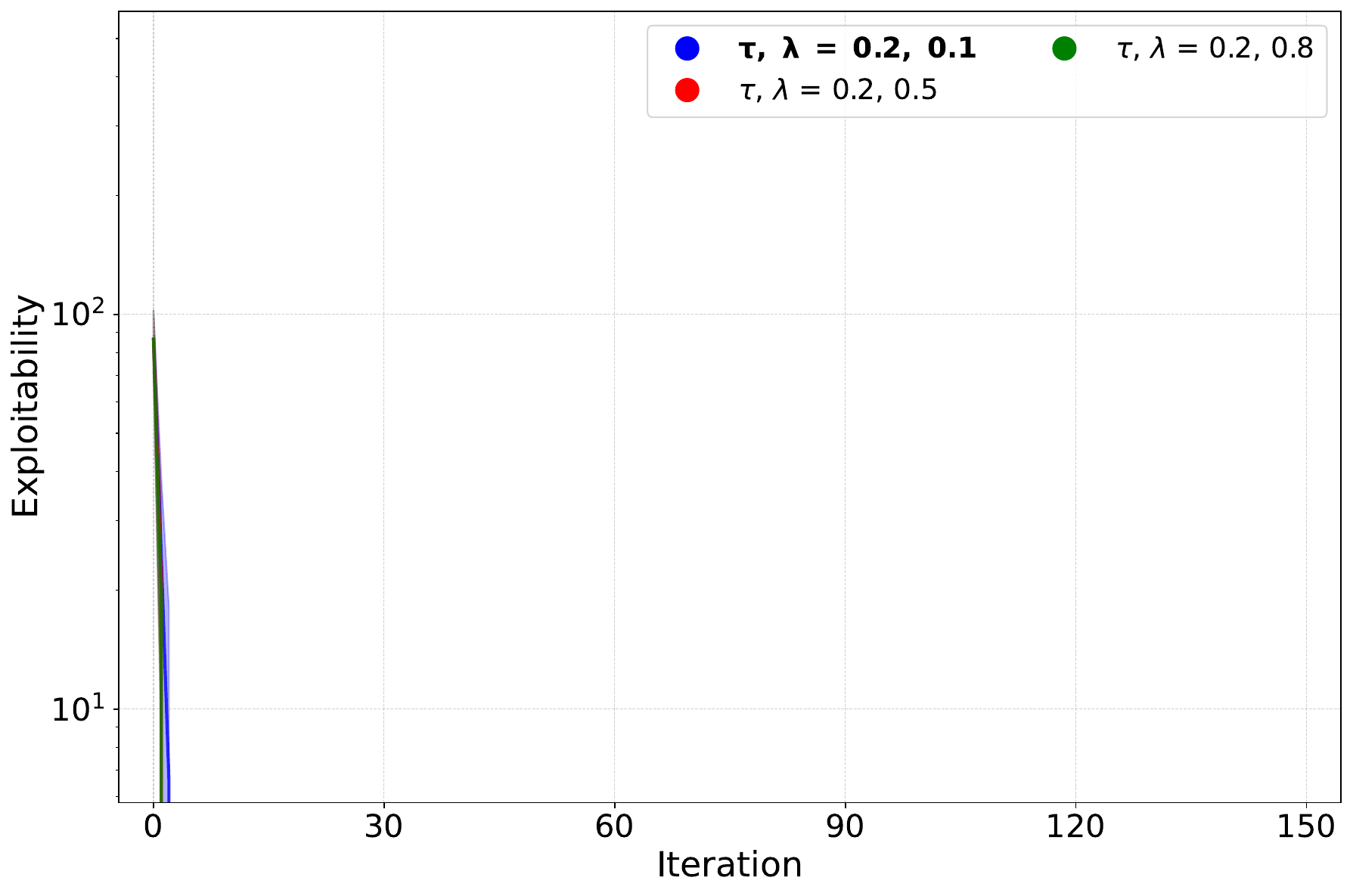}
    \includegraphics[width=0.49\linewidth]{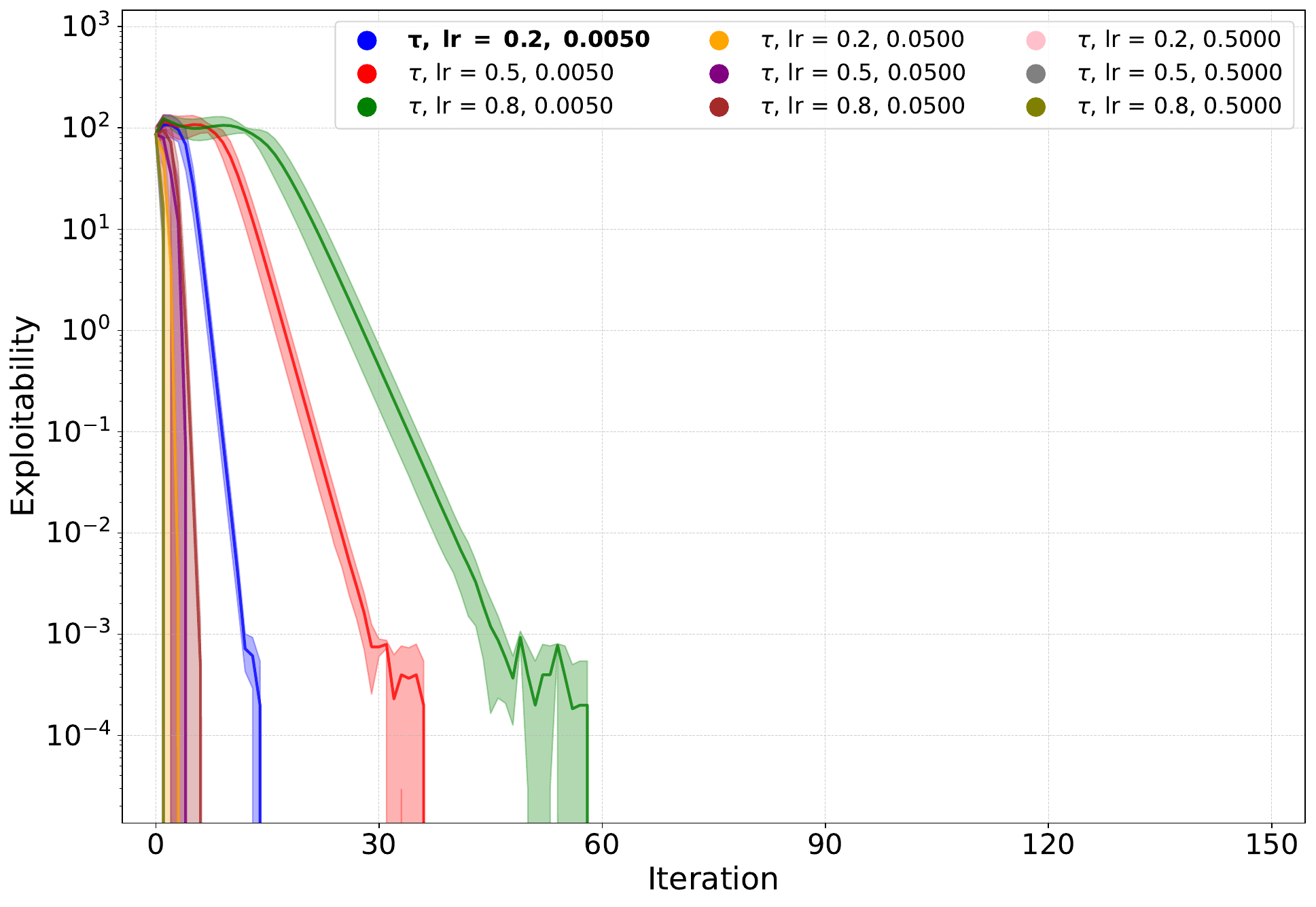}
    \includegraphics[width=0.49\linewidth]{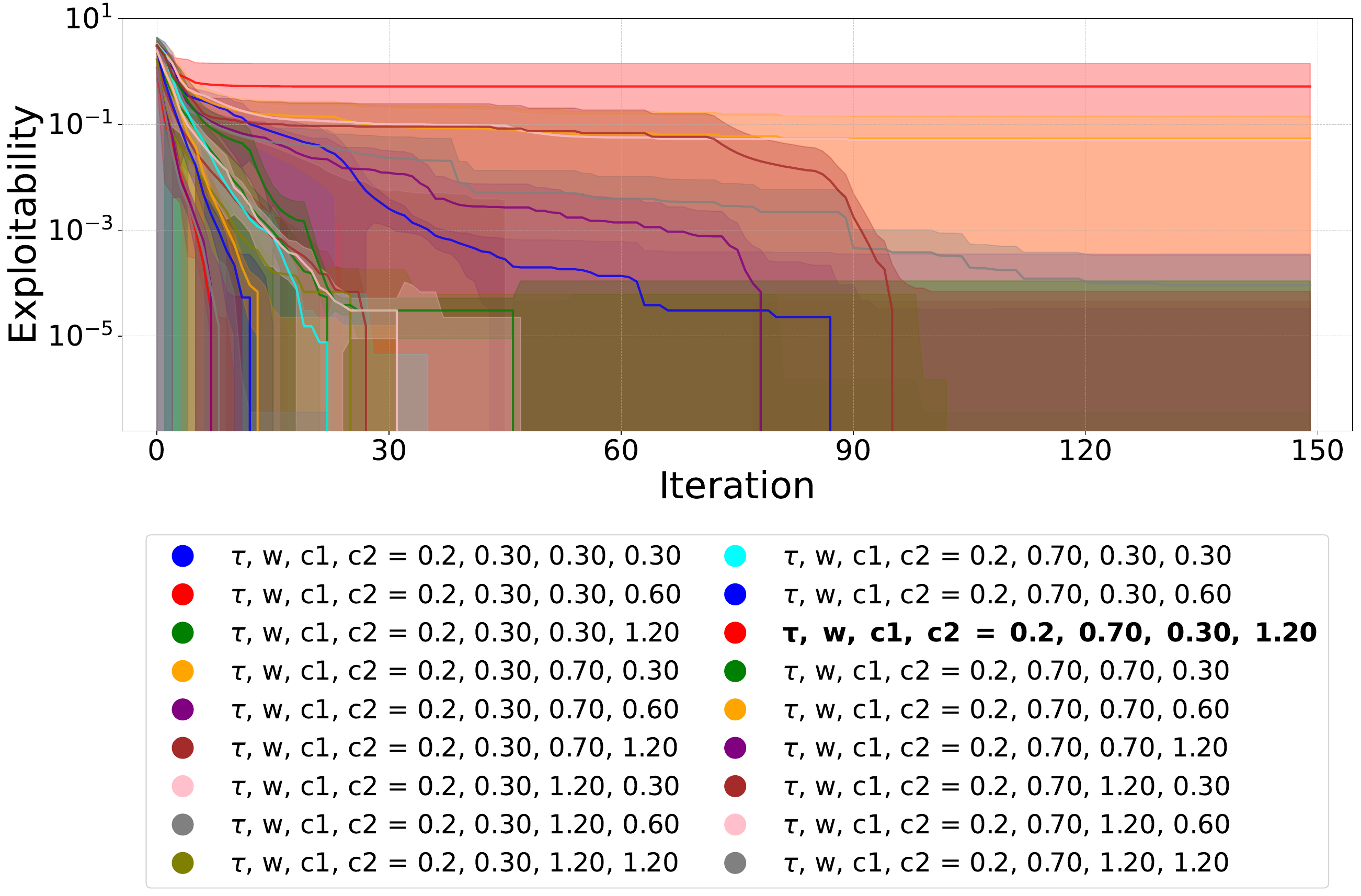}
    \includegraphics[width=0.49\linewidth]{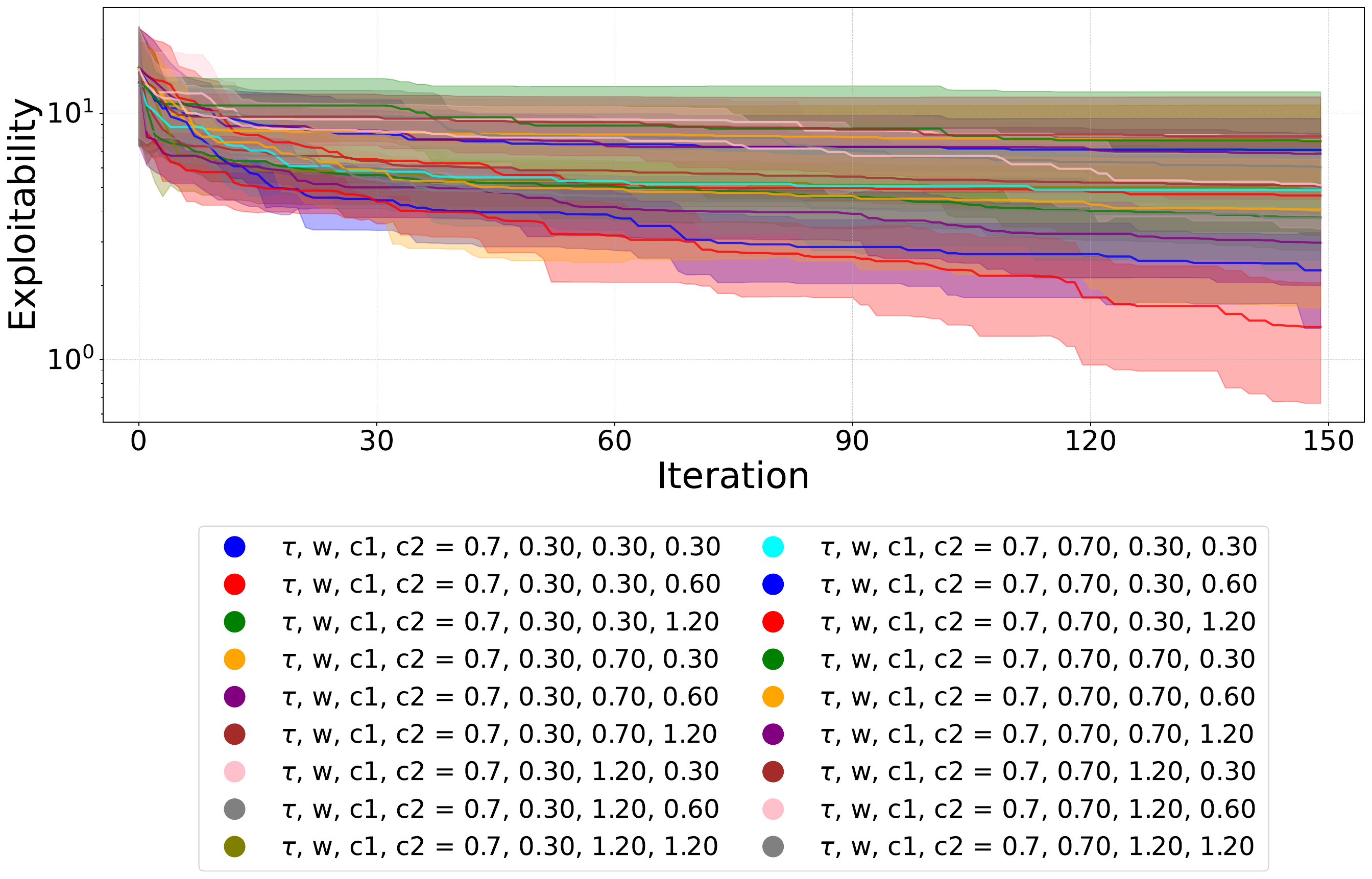}
    \caption{\textbf{Two Beach Bars}. Sensitivity of the algorithms wrt the hyperparameters}
\end{figure}

\newpage
\subsection{Potential Game (P-MFG)}
\begin{figure}[h!]
    \centering
    \includegraphics[width=0.49\linewidth]{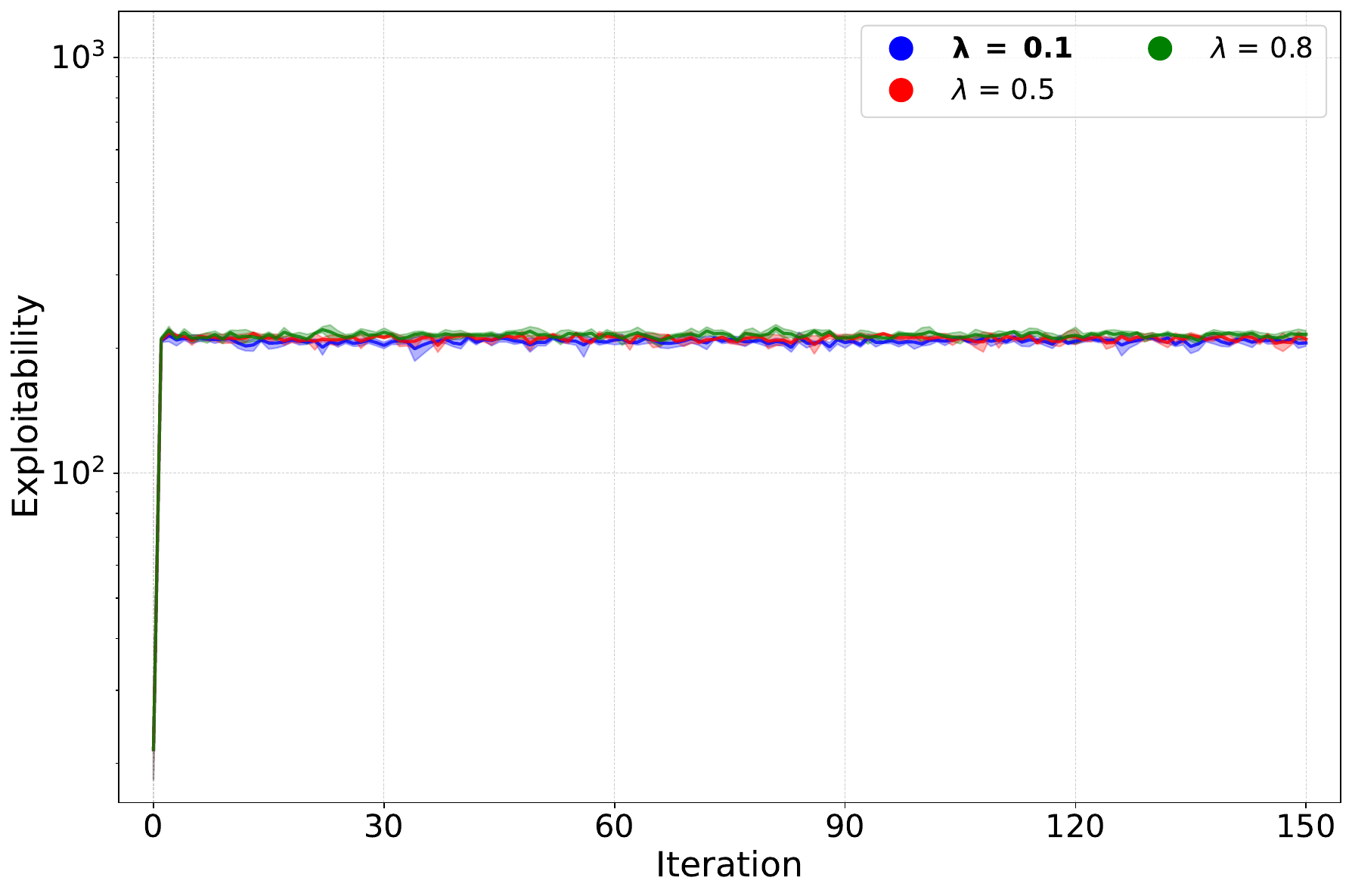}
    \includegraphics[width=0.49\linewidth]{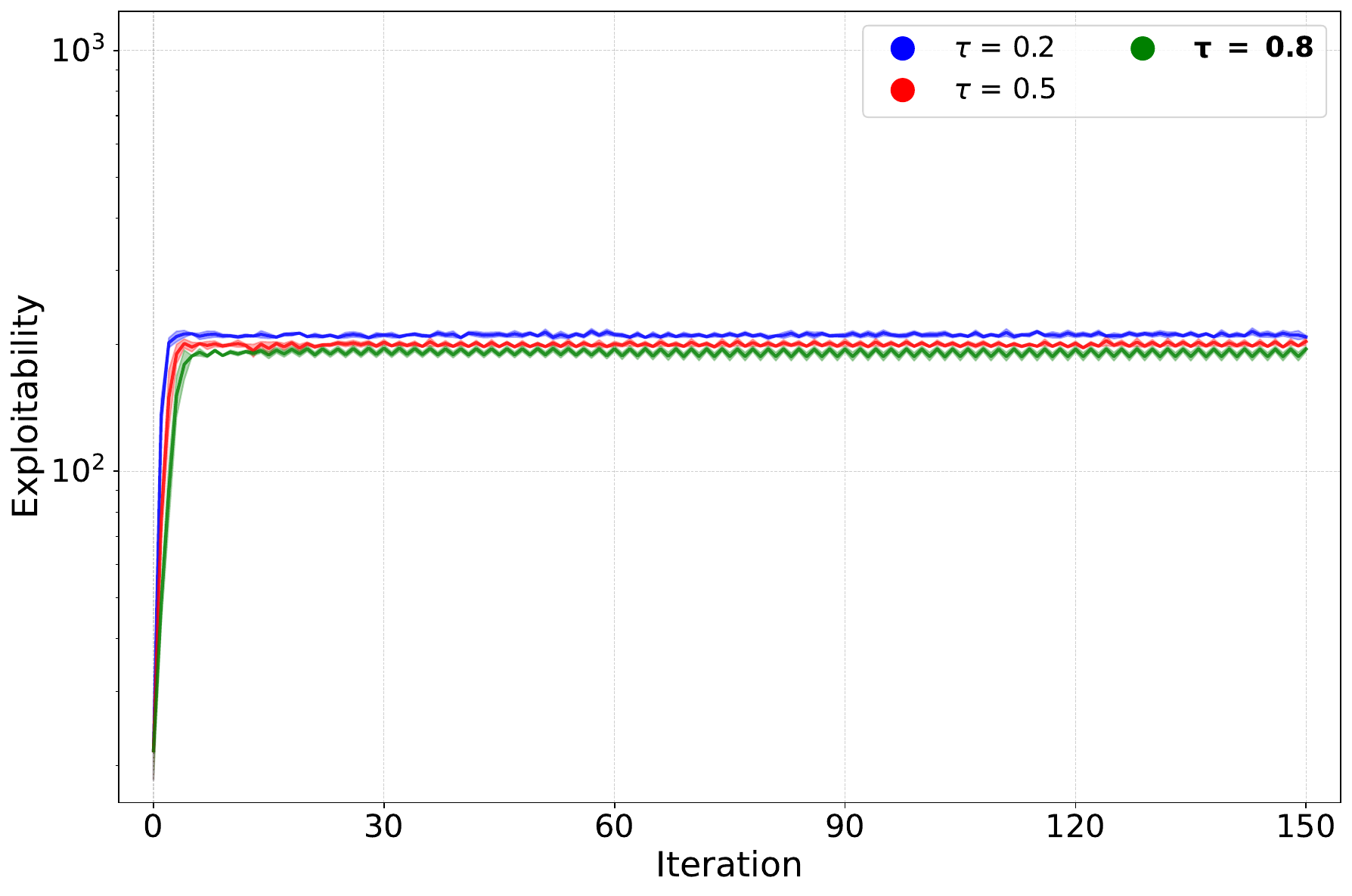}
    \includegraphics[width=0.49\linewidth]{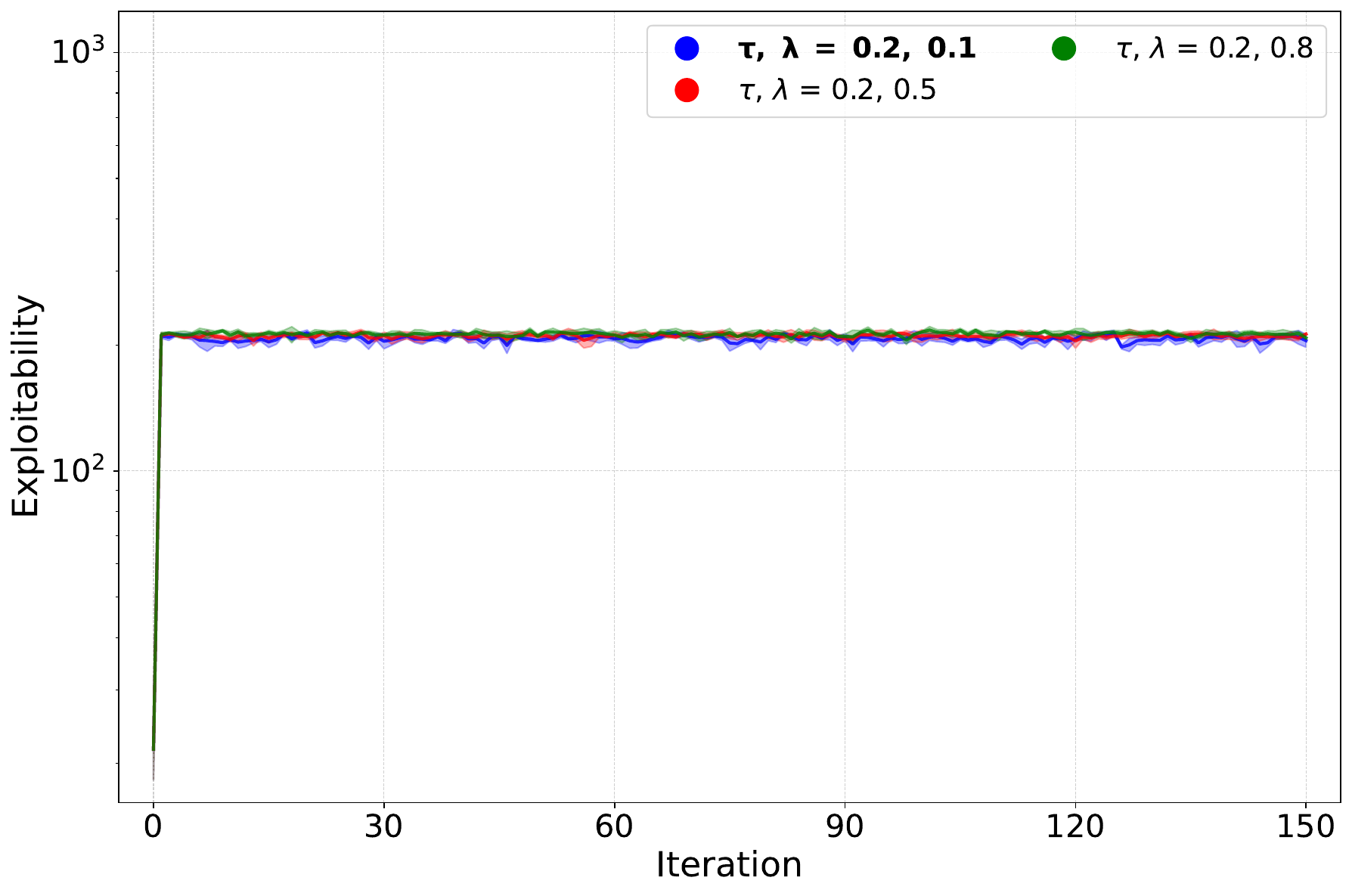}
    \includegraphics[width=0.49\linewidth]{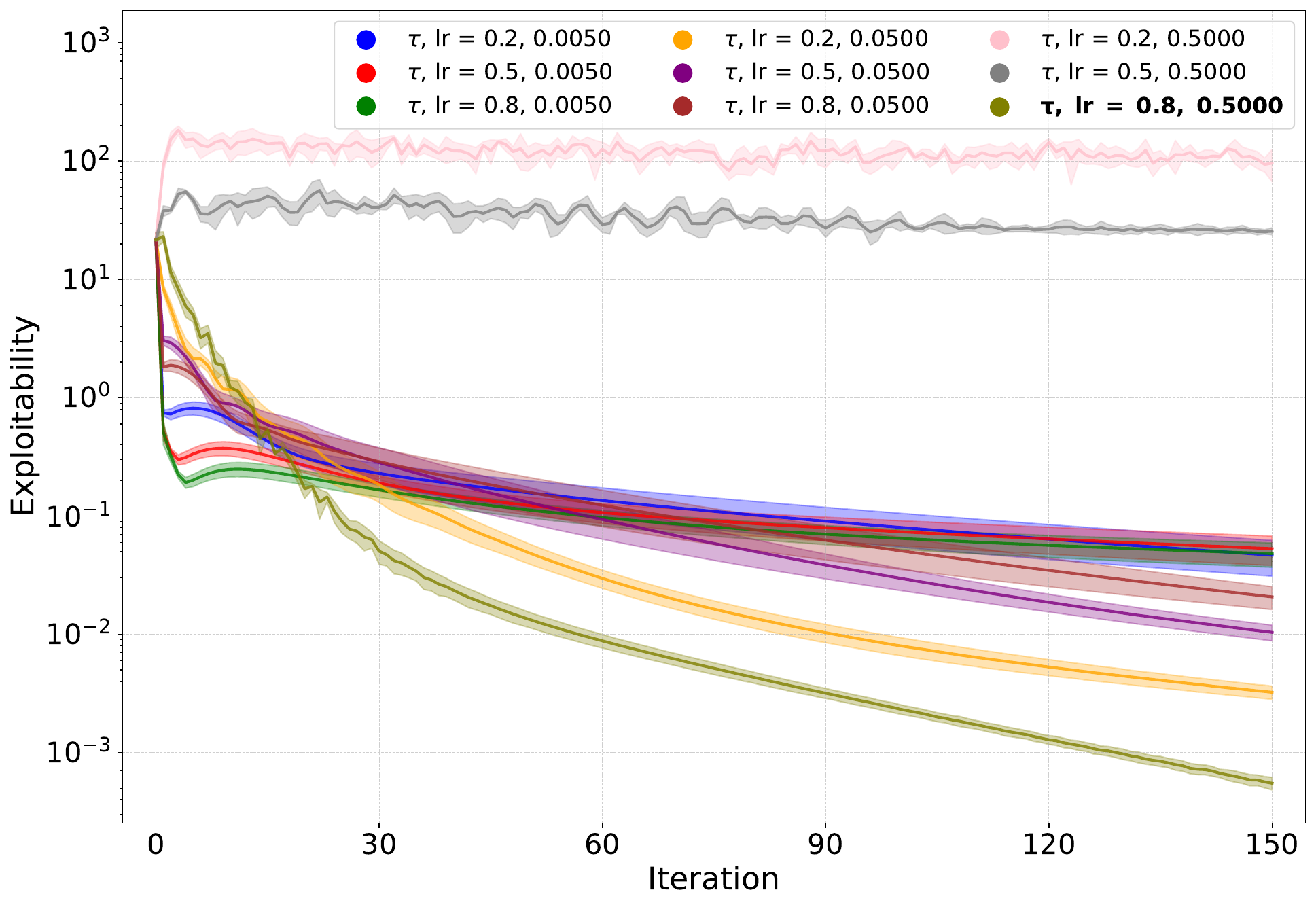}
    \includegraphics[width=0.49\linewidth]{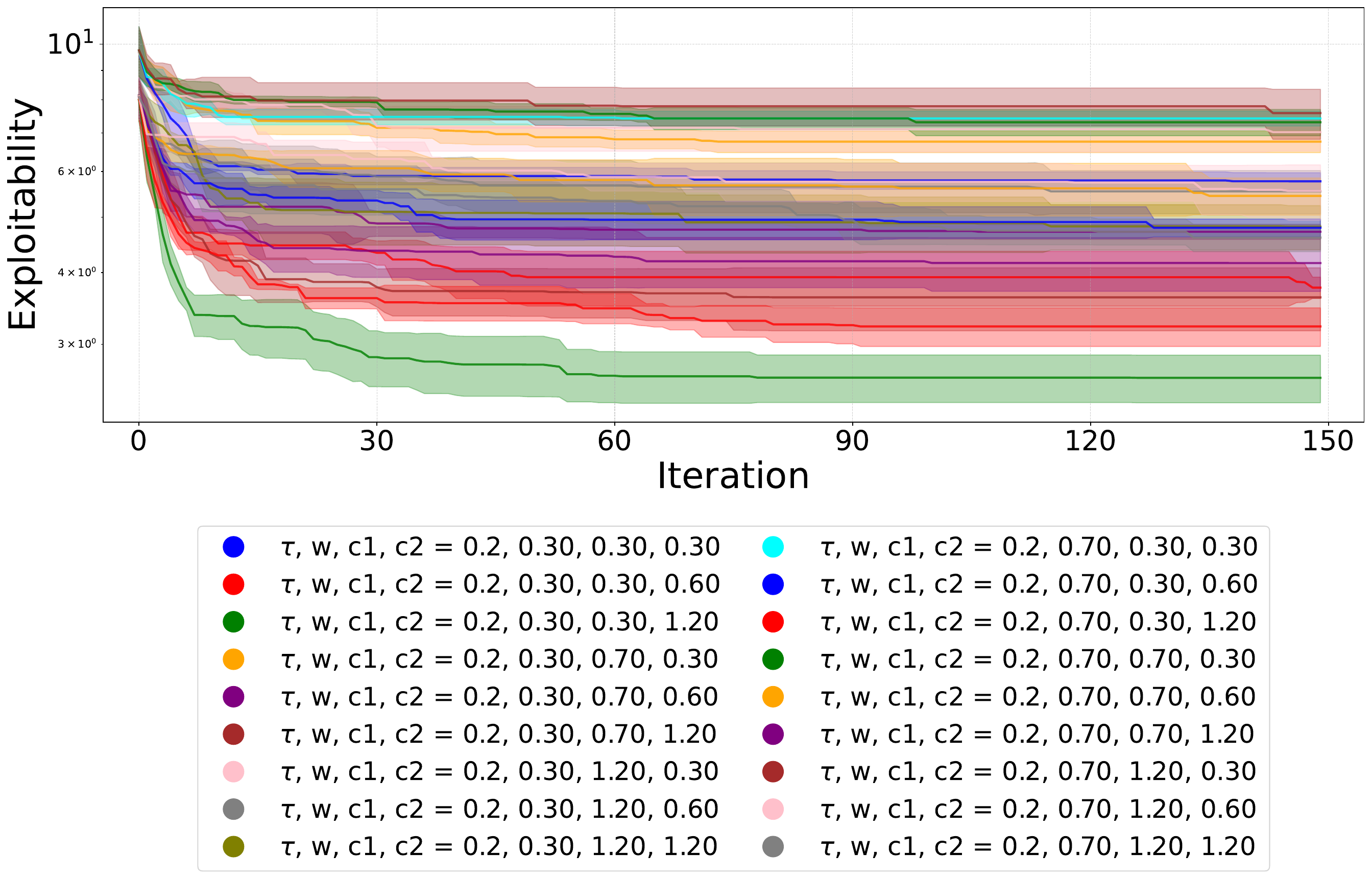}
    \includegraphics[width=0.49\linewidth]{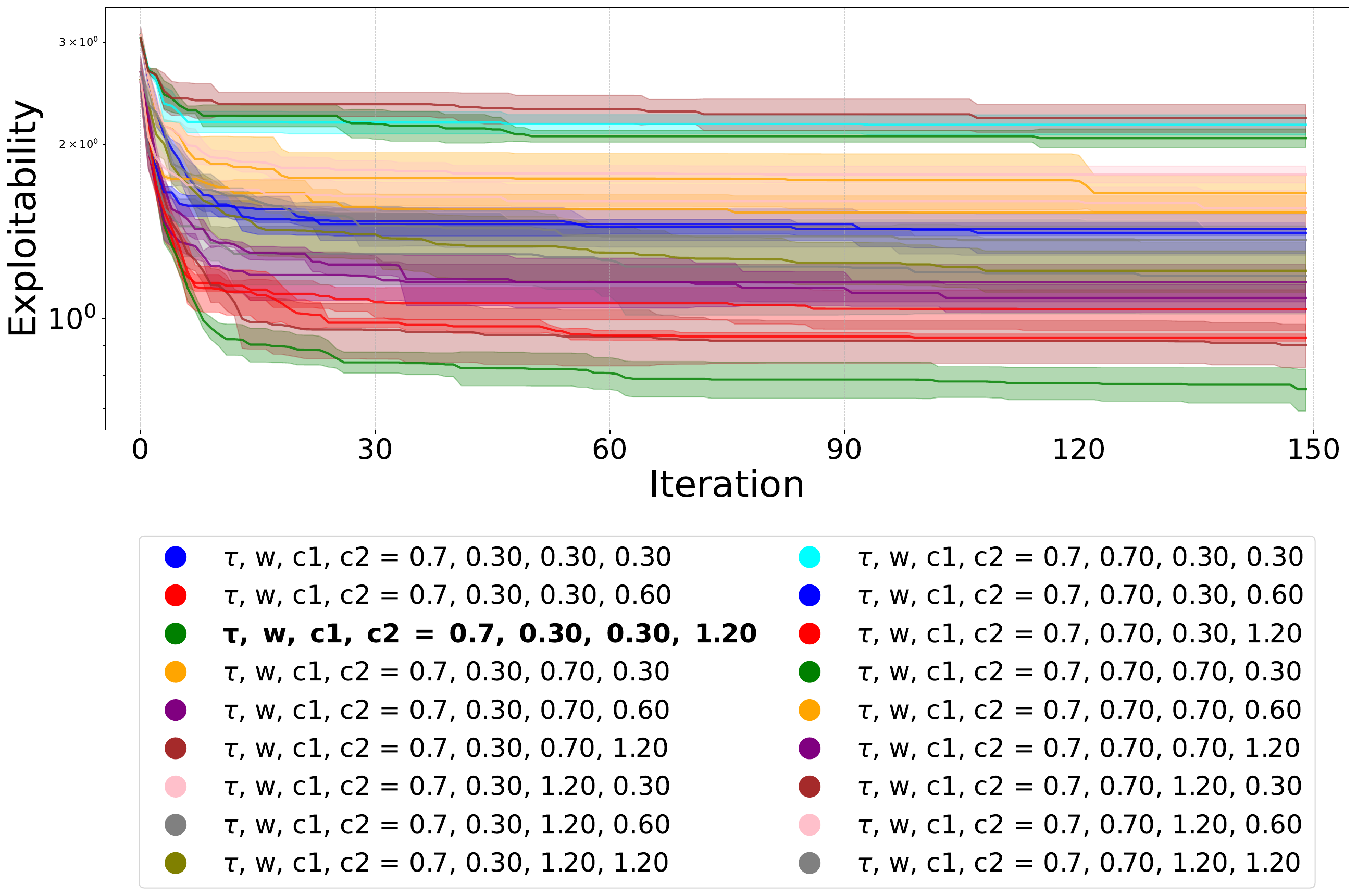}
    \caption{\textbf{4 Rooms Exploration}. Sensitivity of the algorithms wrt the hyperparameters}
    \label{fig:P-MFG sweep}
\end{figure}

\newpage
\textbf{Cyclic Game (non-Potential)}
\begin{figure}[h!]
    \centering
    \includegraphics[width=0.49\linewidth]{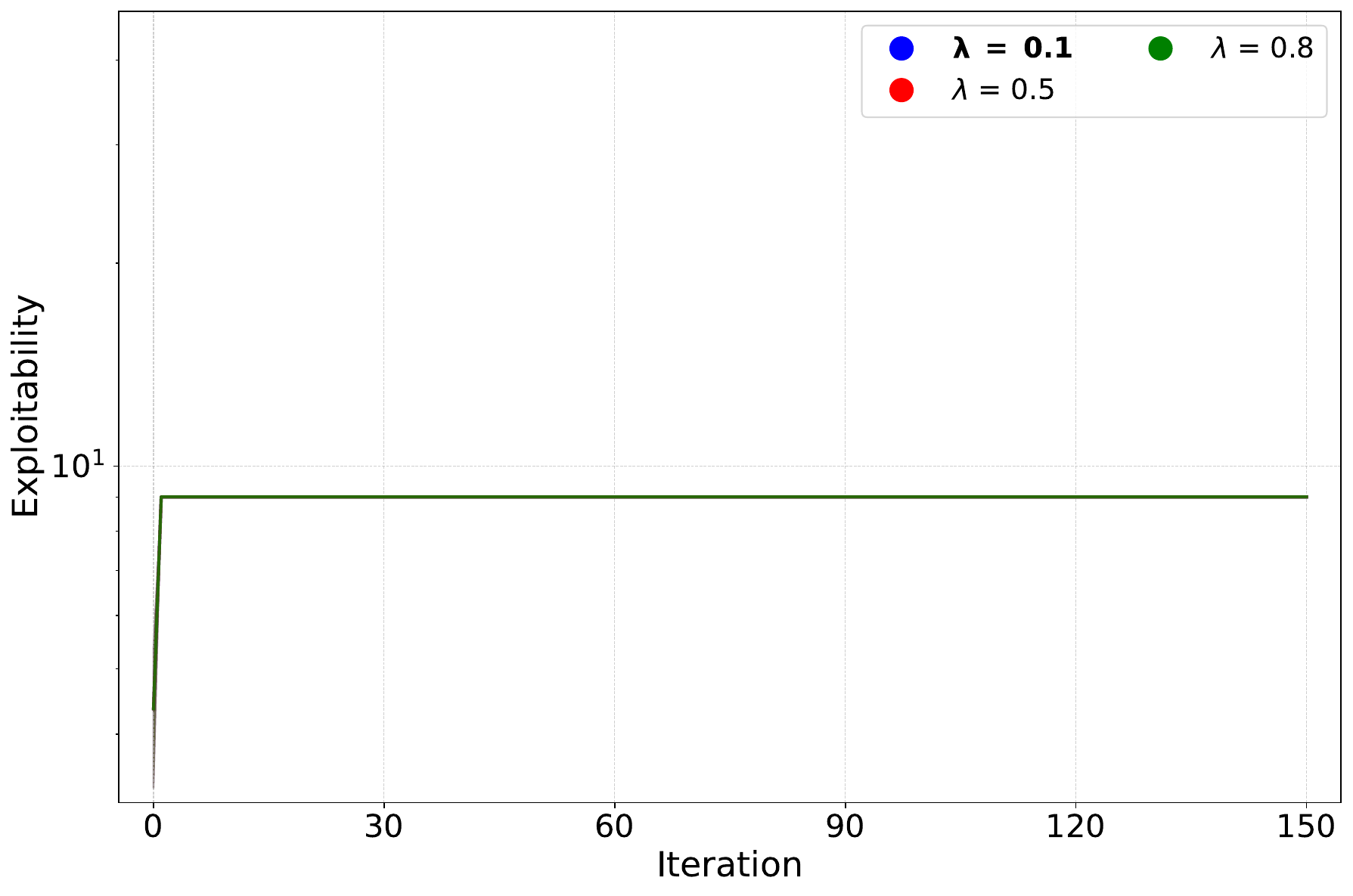}
    \includegraphics[width=0.49\linewidth]{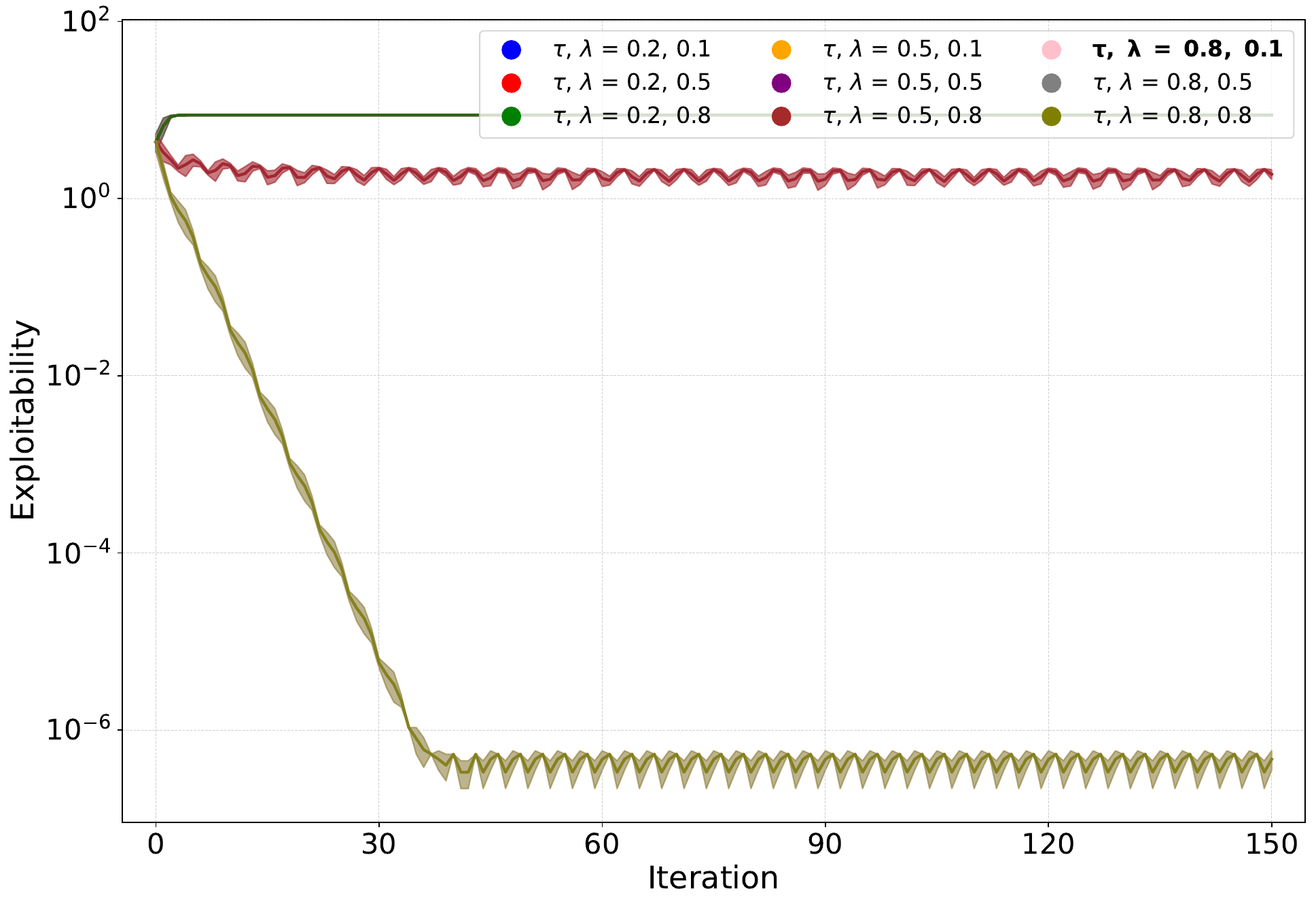}
    \includegraphics[width=0.49\linewidth]{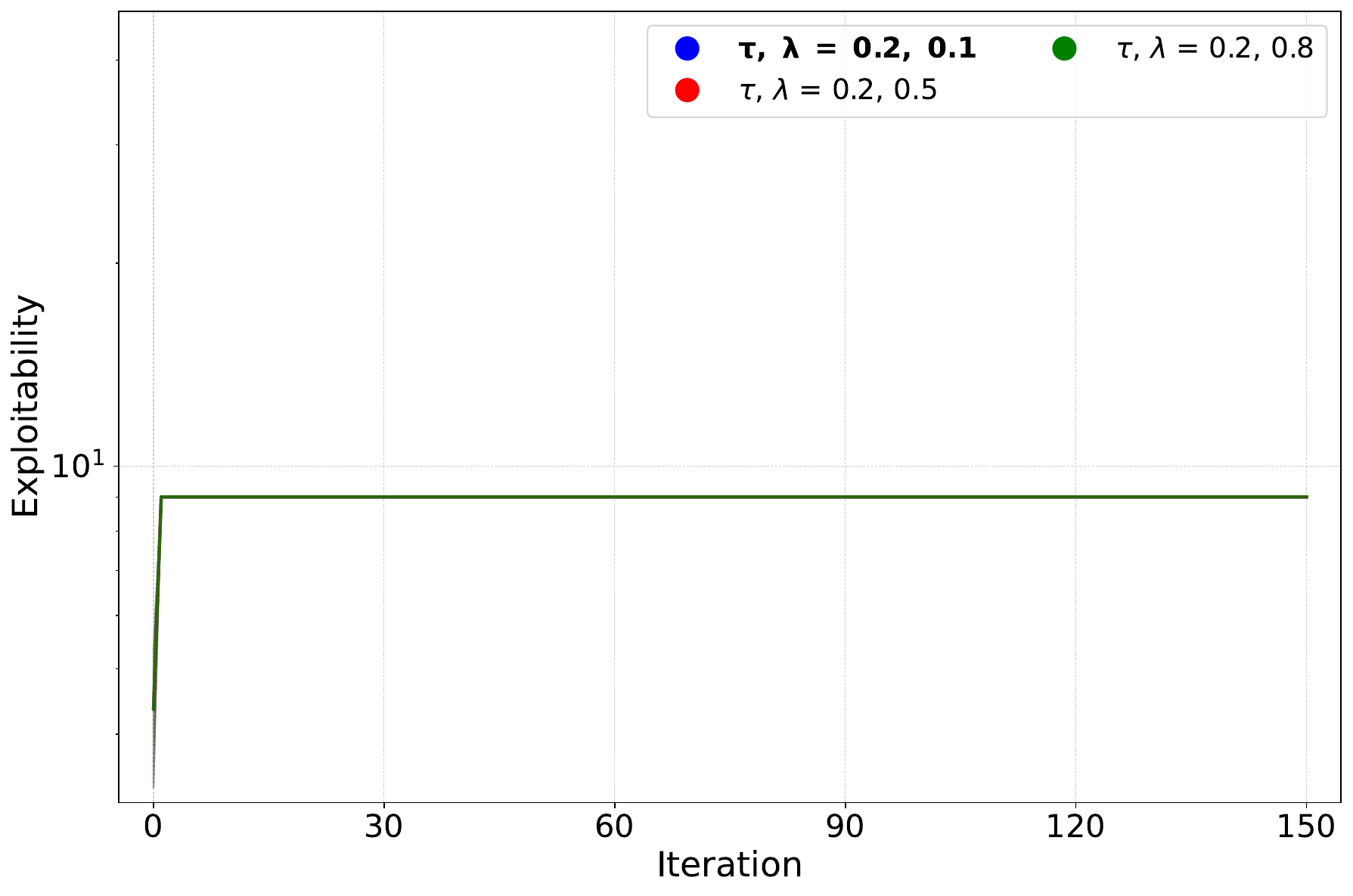}
    \includegraphics[width=0.49\linewidth]{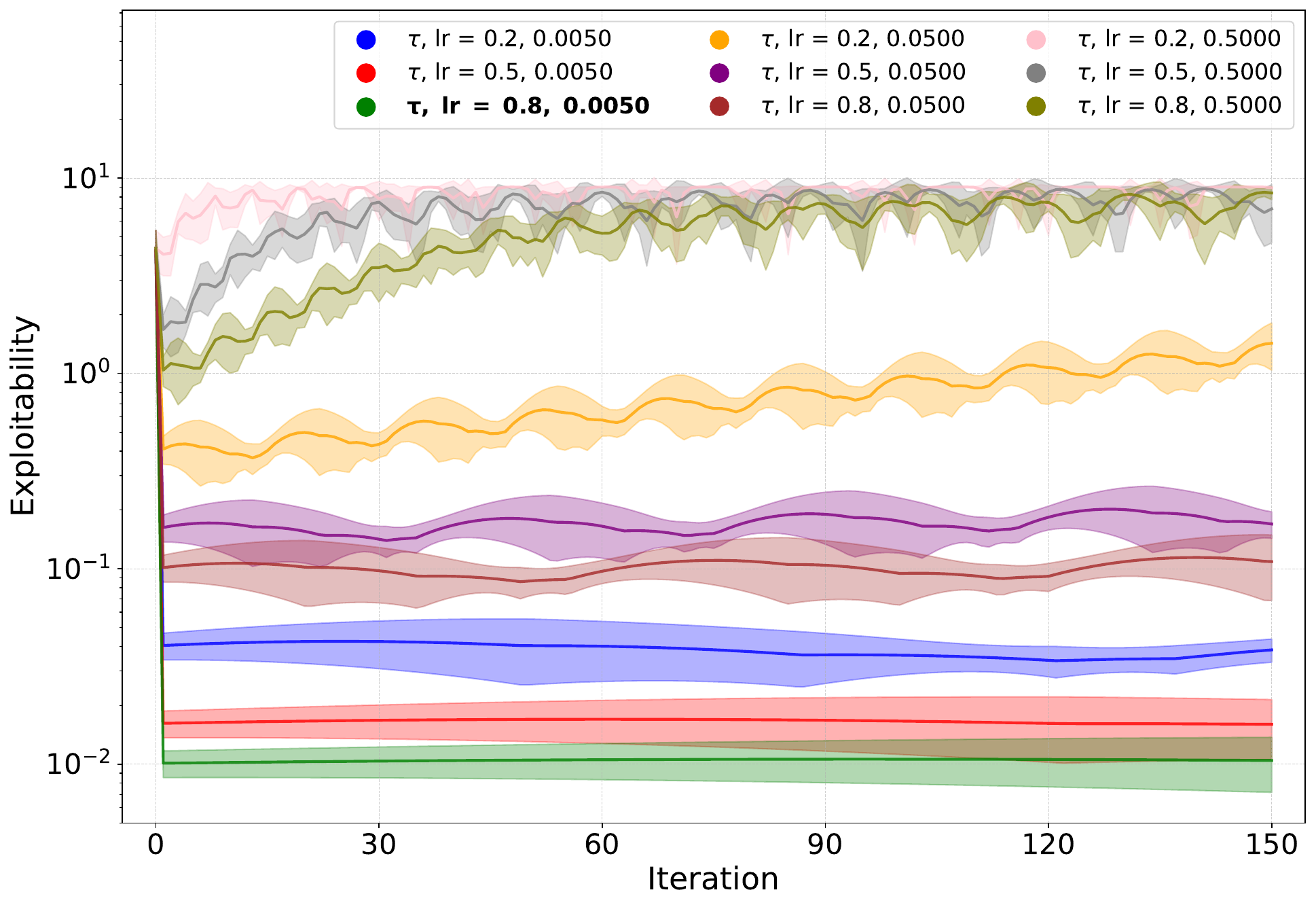}
    \includegraphics[width=0.49\linewidth]{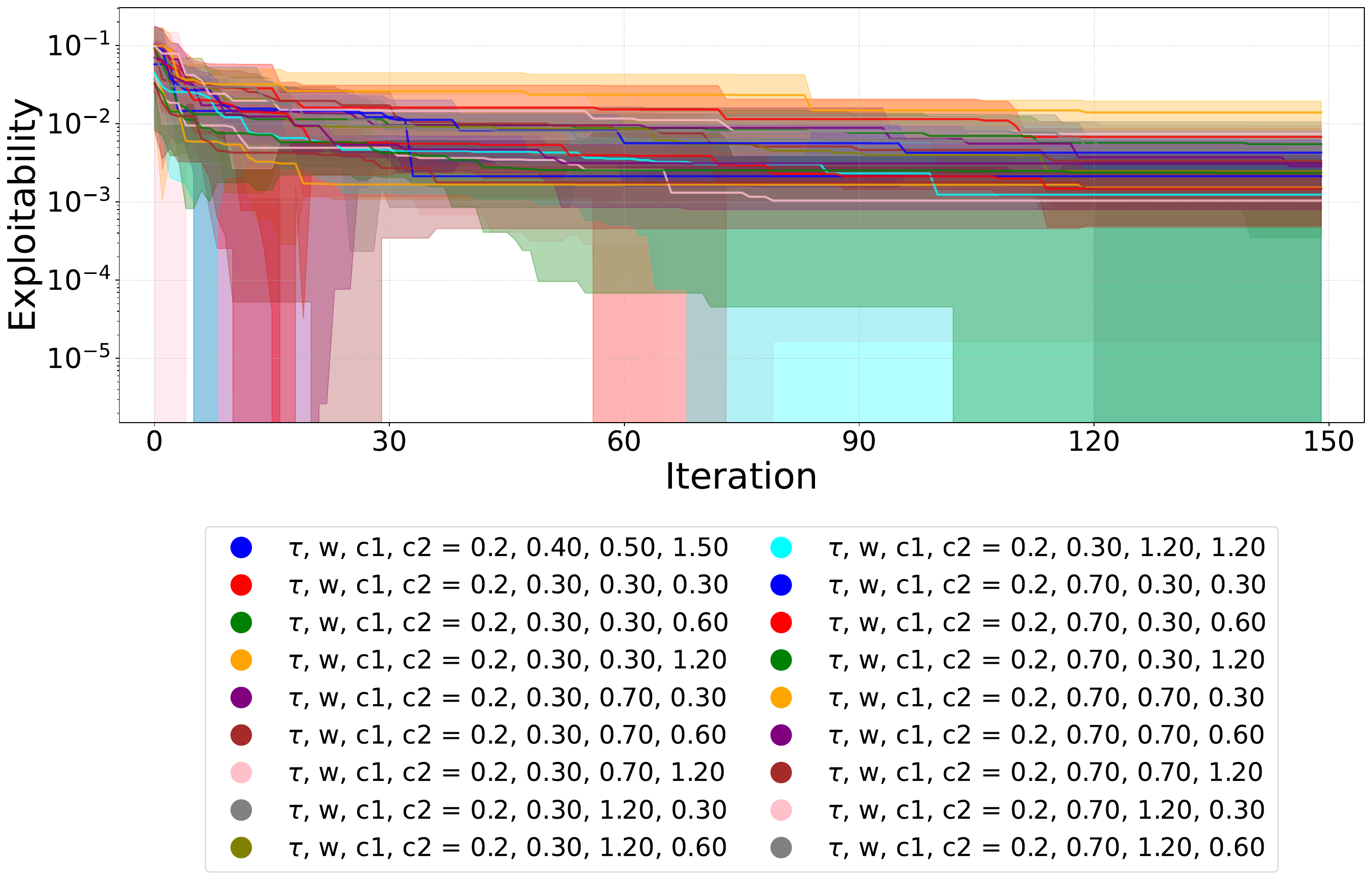}
    \includegraphics[width=0.49\linewidth]{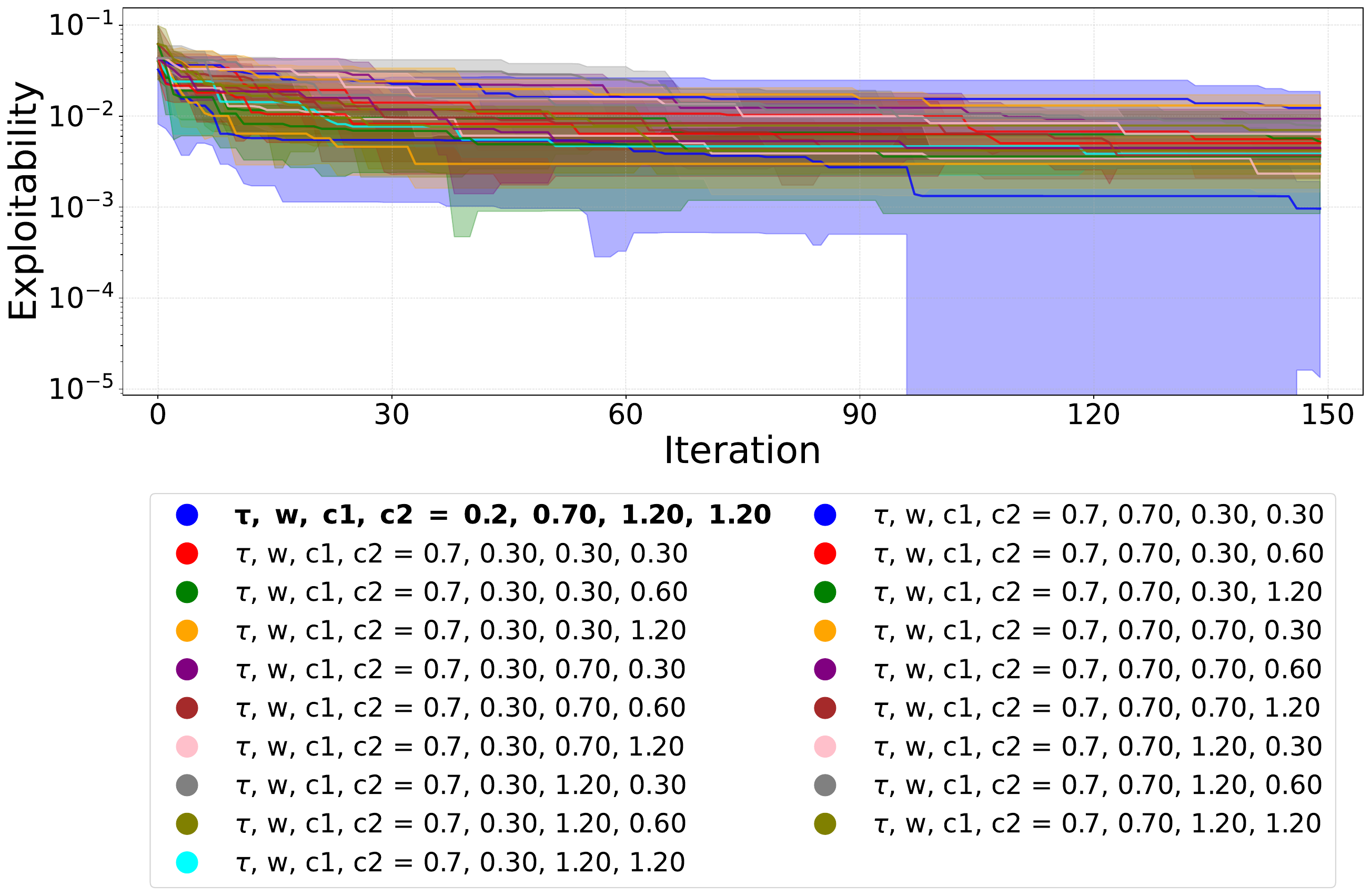}
    \caption{\textbf{Rock, Paper, Scissors}. Sensitivity of the algorithms wrt the hyperparameters}
    \label{fig:Cyclic-MFG sweep}
\end{figure}

\newpage
\subsection{Dynamics-Coupled MFG (DC-MFG)}
\begin{figure}[h!]
    \centering
    \includegraphics[width=0.49\linewidth]{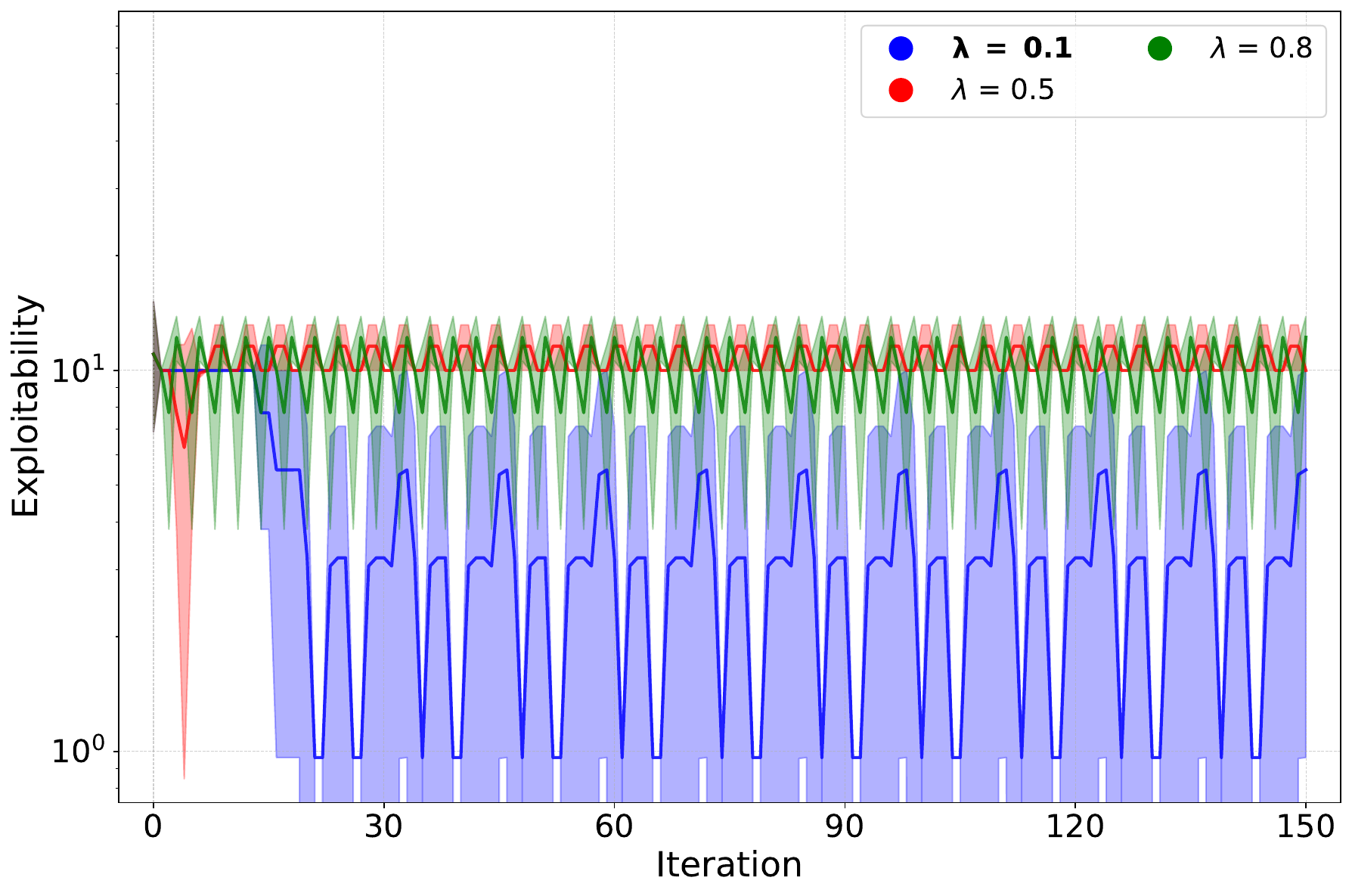}
    \includegraphics[width=0.49\linewidth]{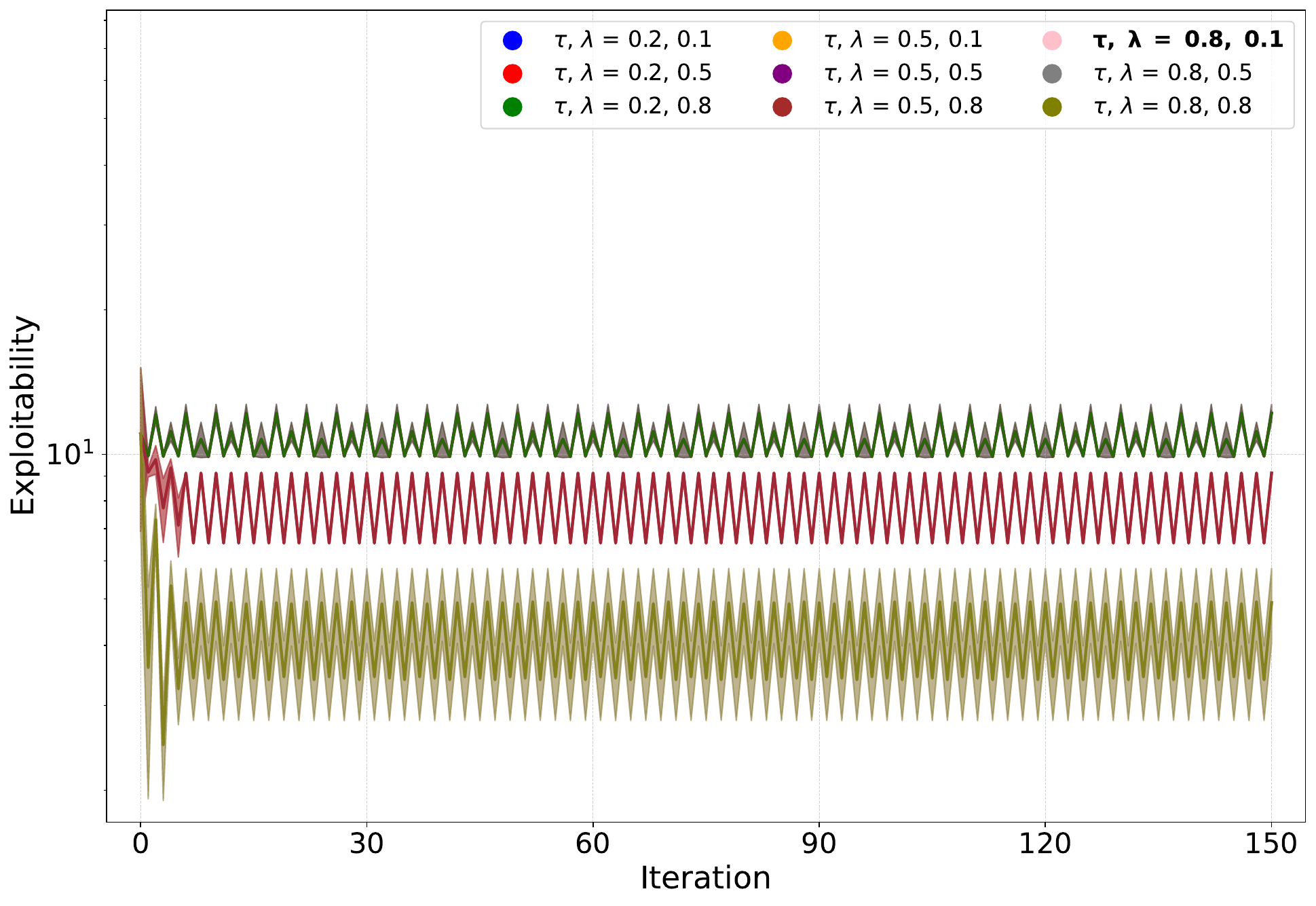}
    \includegraphics[width=0.49\linewidth]{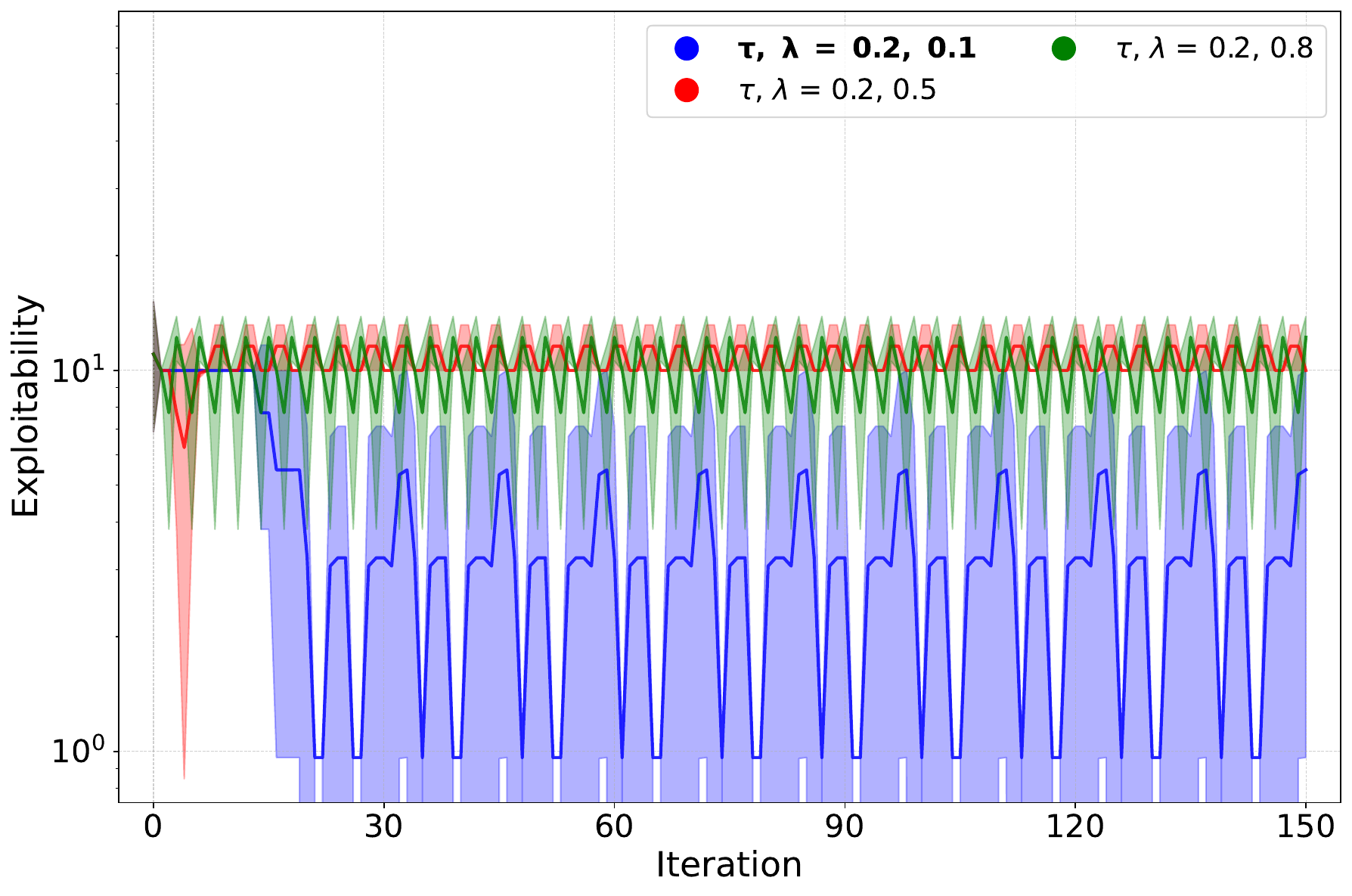}
    \includegraphics[width=0.49\linewidth]{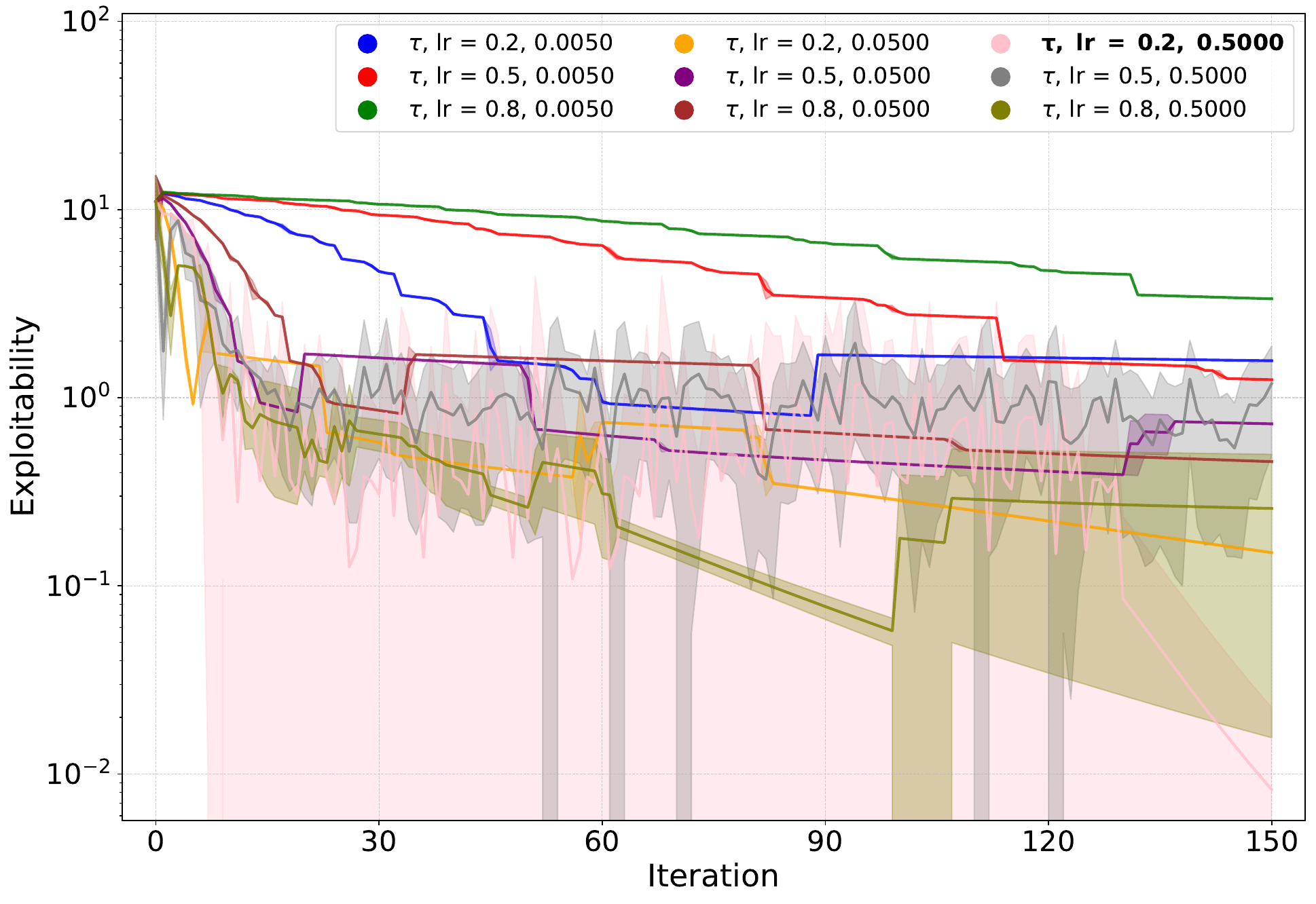}
    \includegraphics[width=0.49\linewidth]{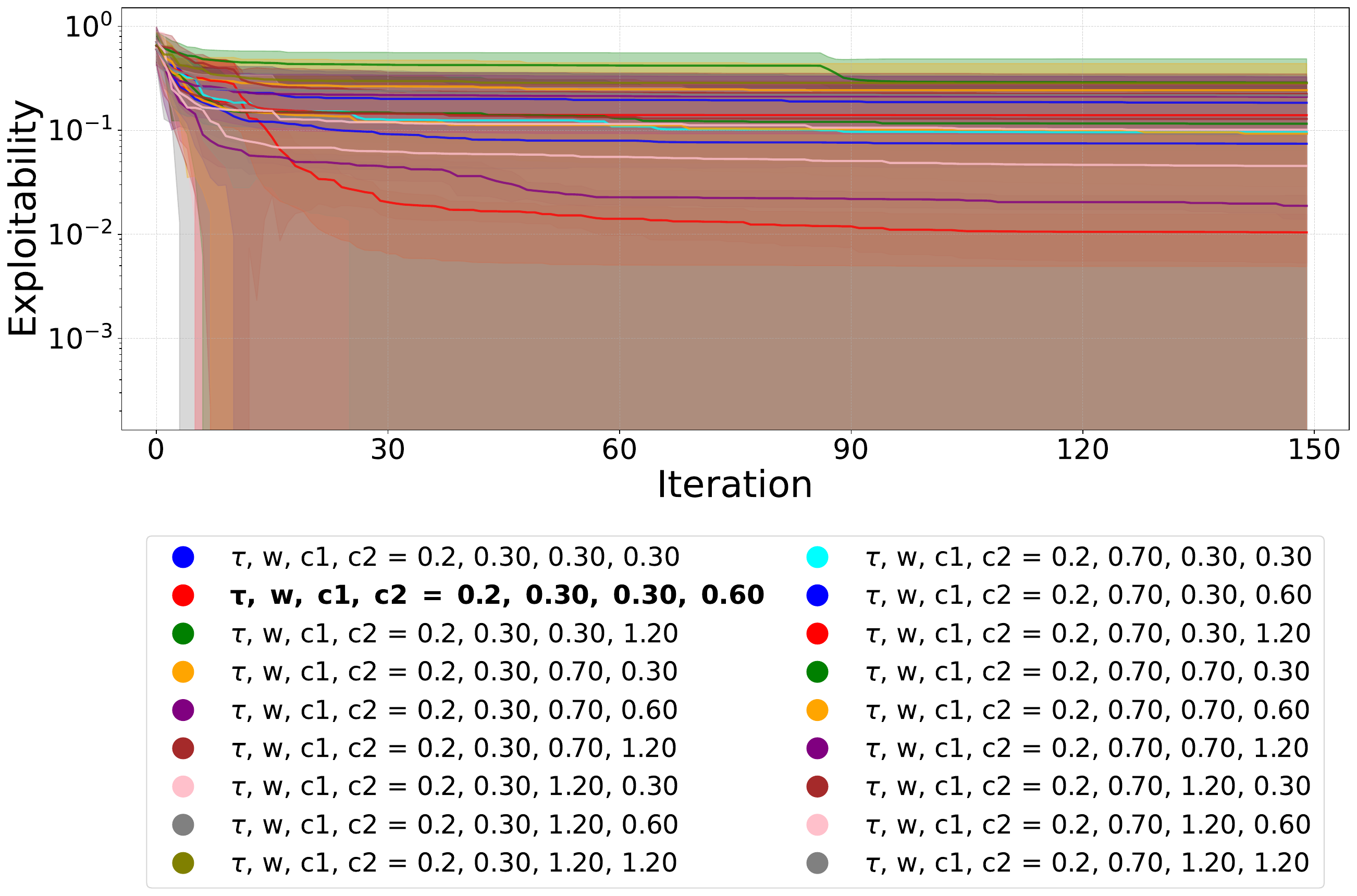}
    \includegraphics[width=0.49\linewidth]{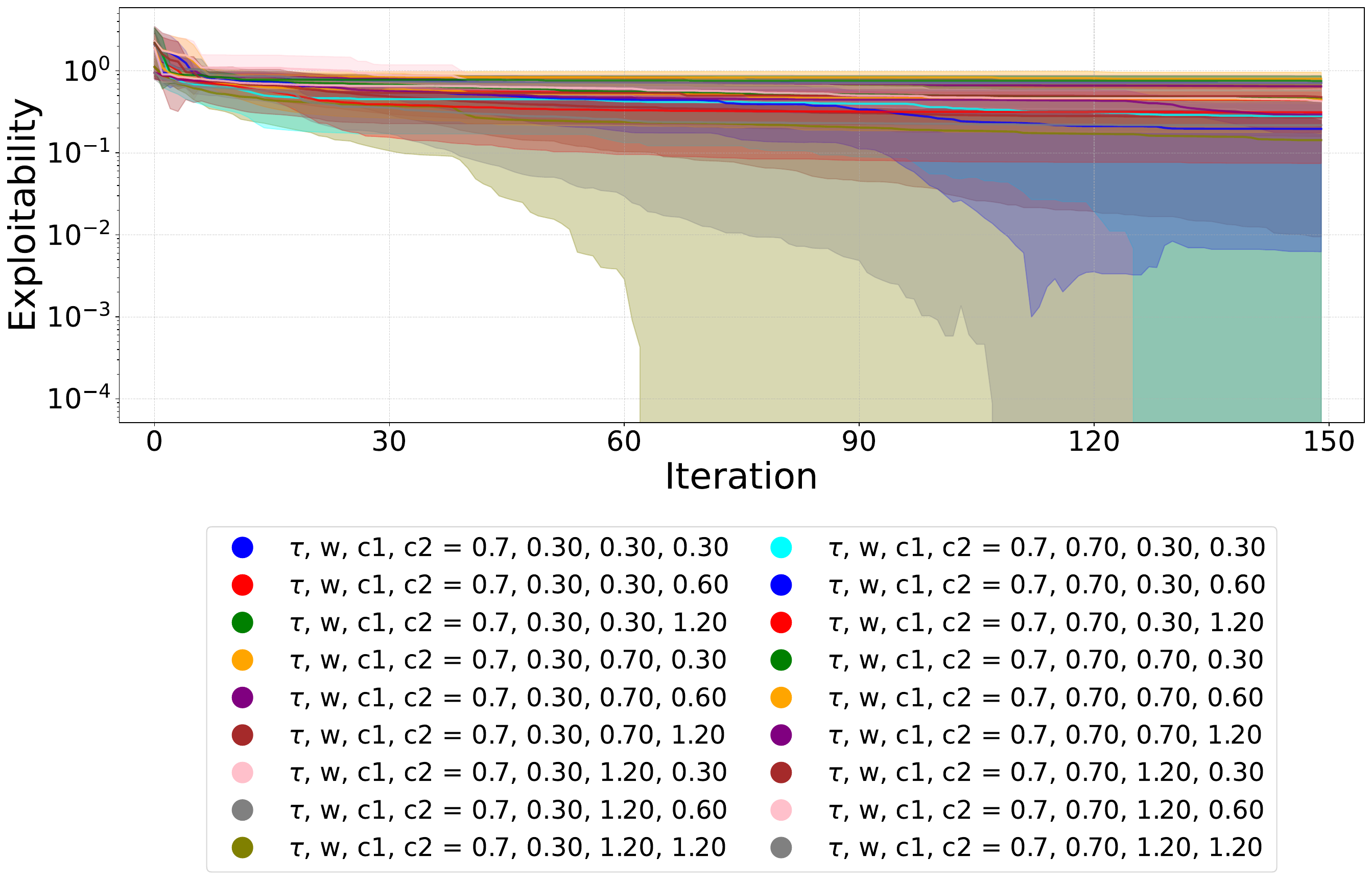}
    \caption{\textbf{SISEpidemic}. Sensitivity of the algorithms wrt the hyperparameters}
    \label{fig:SISEpidemic_sweep}
\end{figure}

\begin{figure}[h!]
    \centering
    \includegraphics[width=0.49\linewidth]{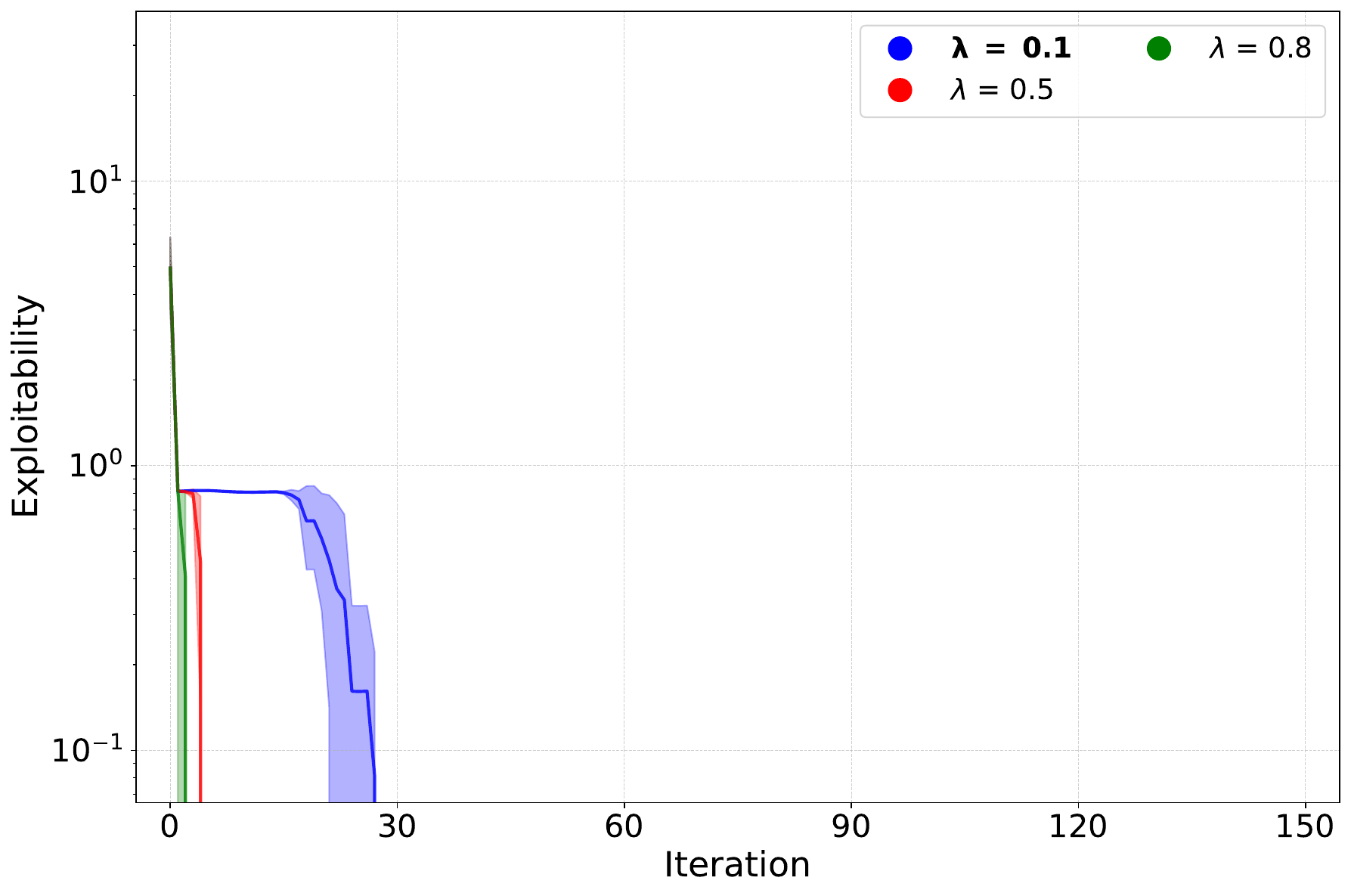}
    \includegraphics[width=0.49\linewidth]{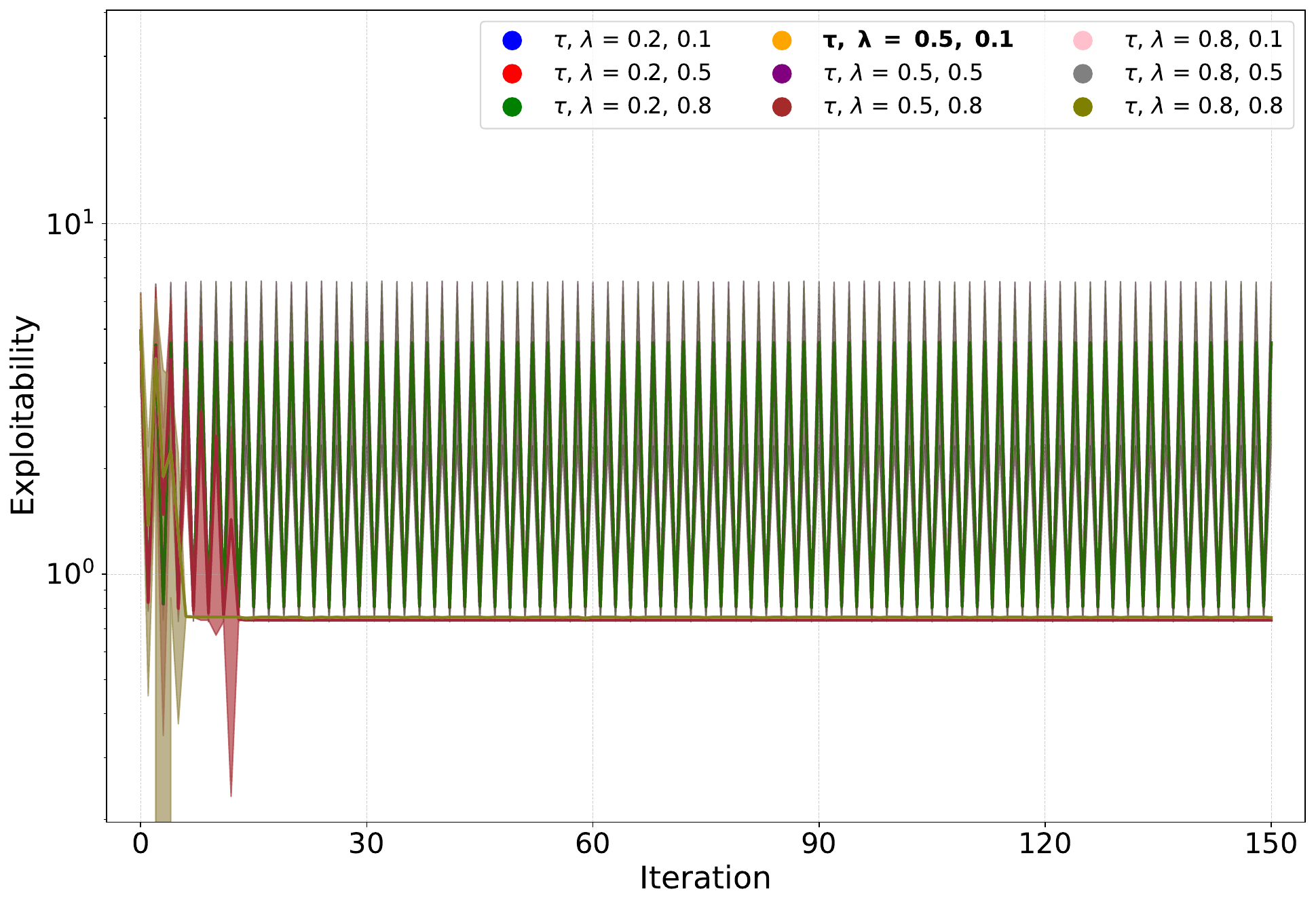}
    \includegraphics[width=0.49\linewidth]{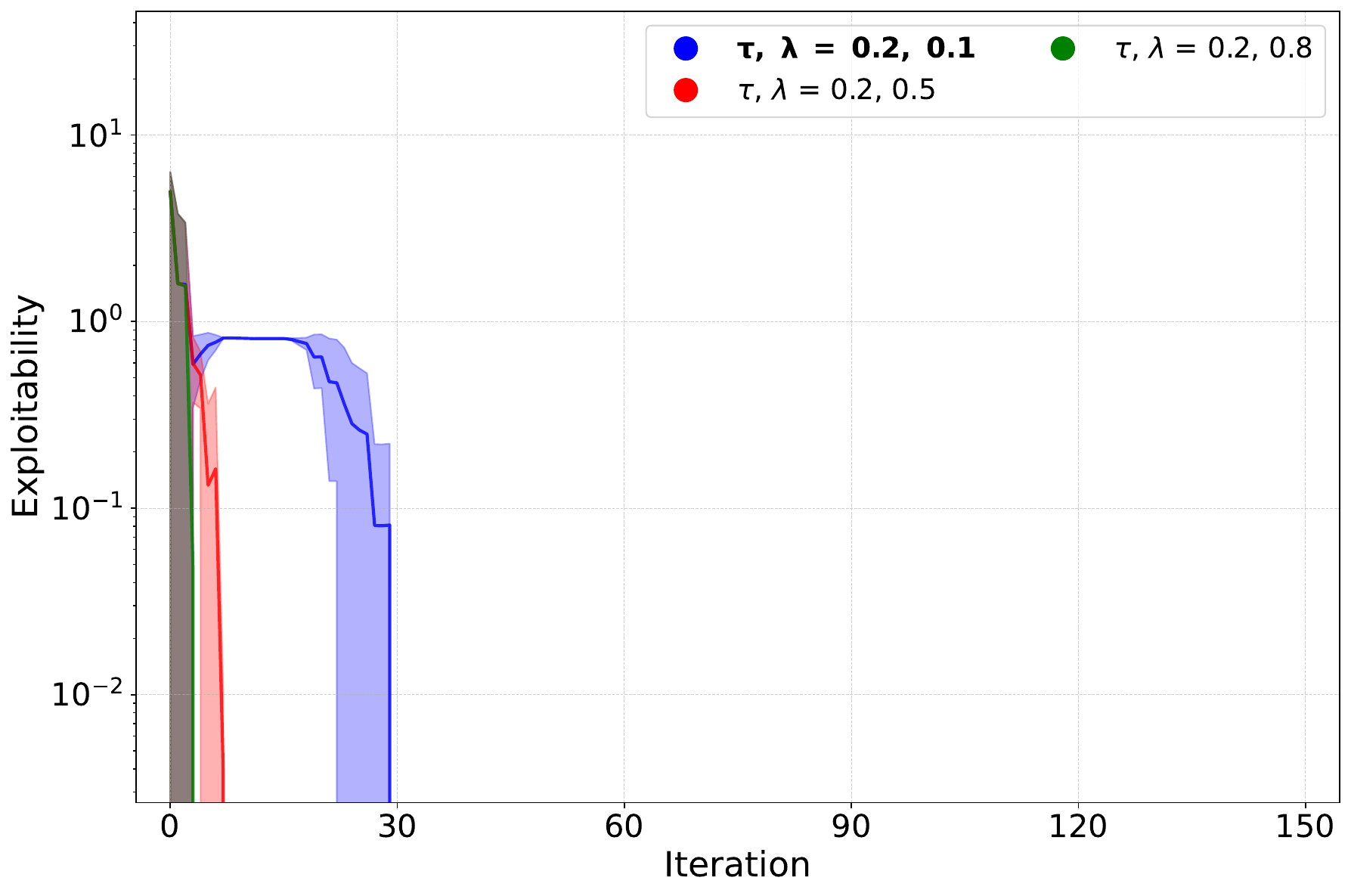}
    \includegraphics[width=0.49\linewidth]{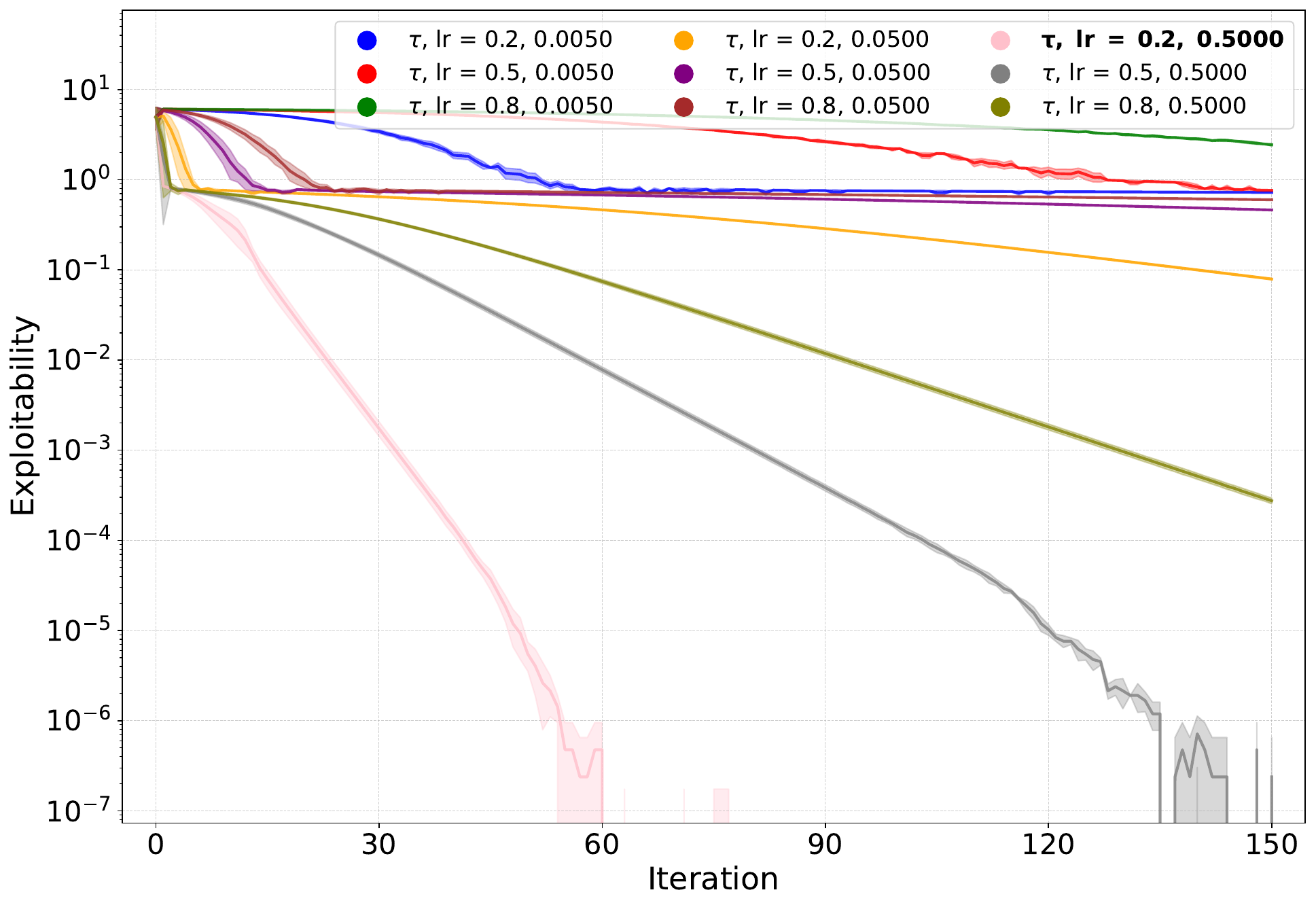}
    \includegraphics[width=0.49\linewidth]{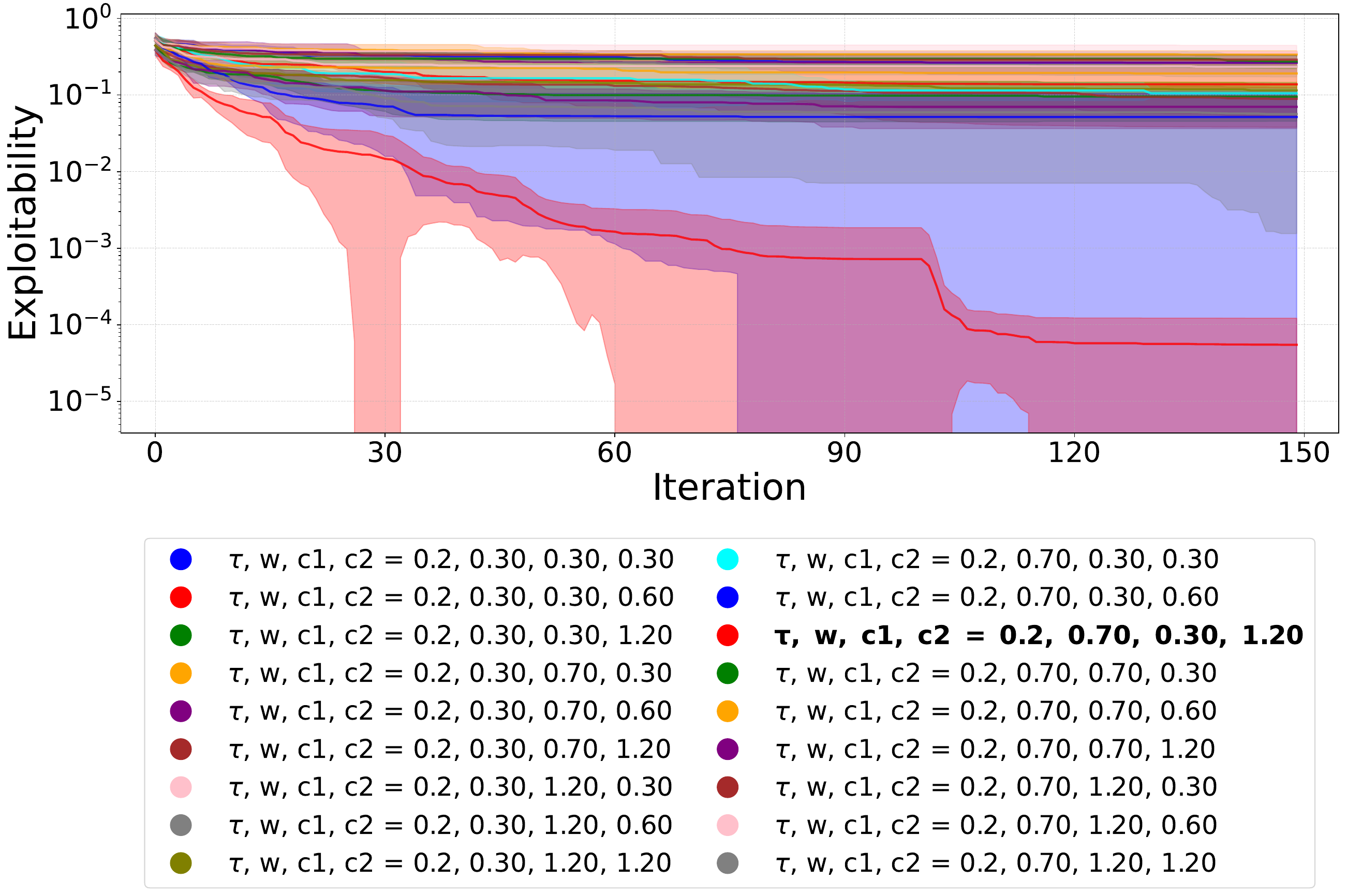}
    \includegraphics[width=0.49\linewidth]{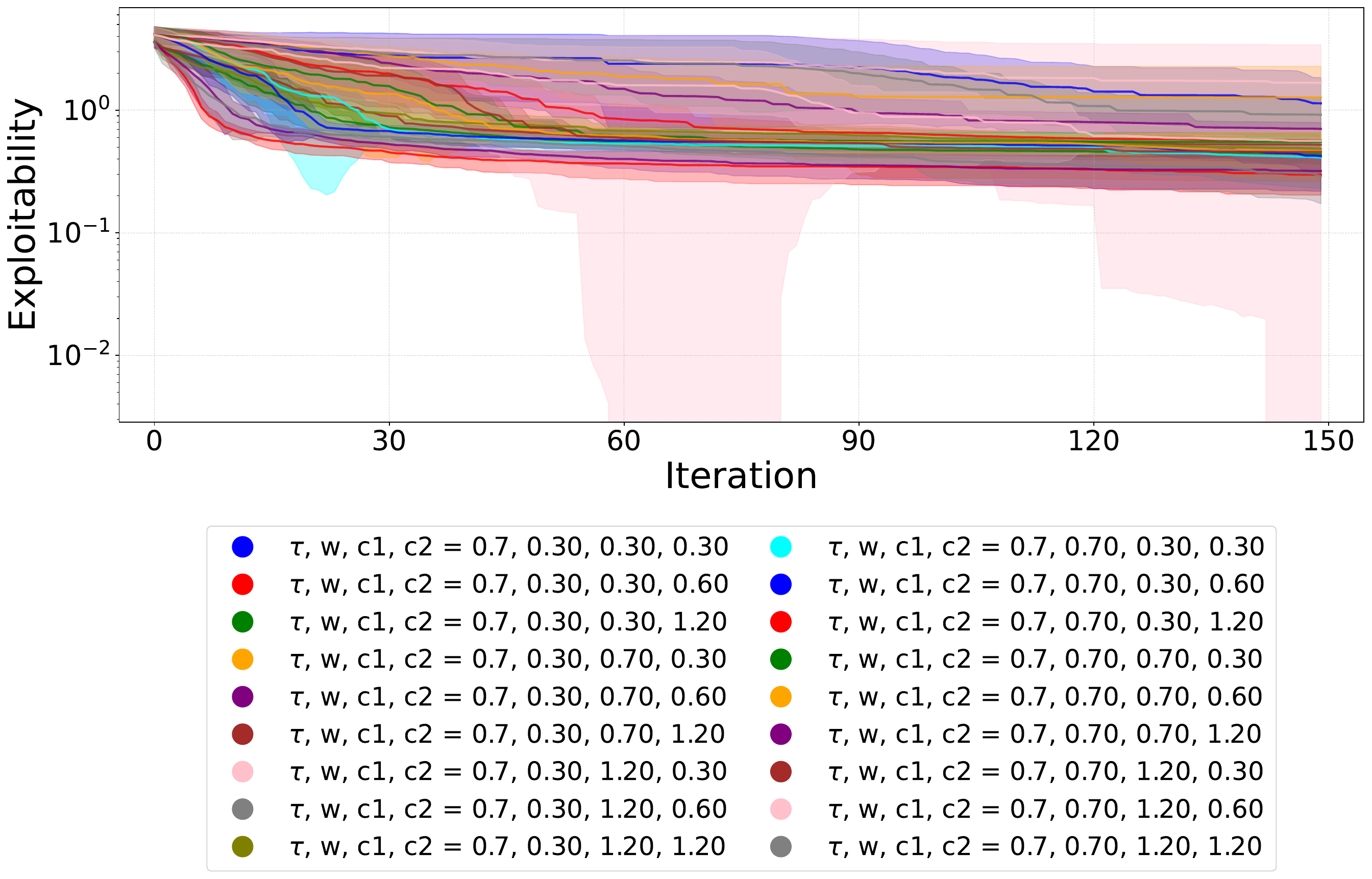}
    \caption{\textbf{KineticCongestion}. Sensitivity of the algorithms wrt the hyperparameters}
    \label{fig:KineticCongestion_sweep}
\end{figure}

\clearpage
\section{Environment Parameters}\label{app:env_sweep}
\begin{figure}[h!]
    \centering
    \includegraphics[width=0.9\linewidth]{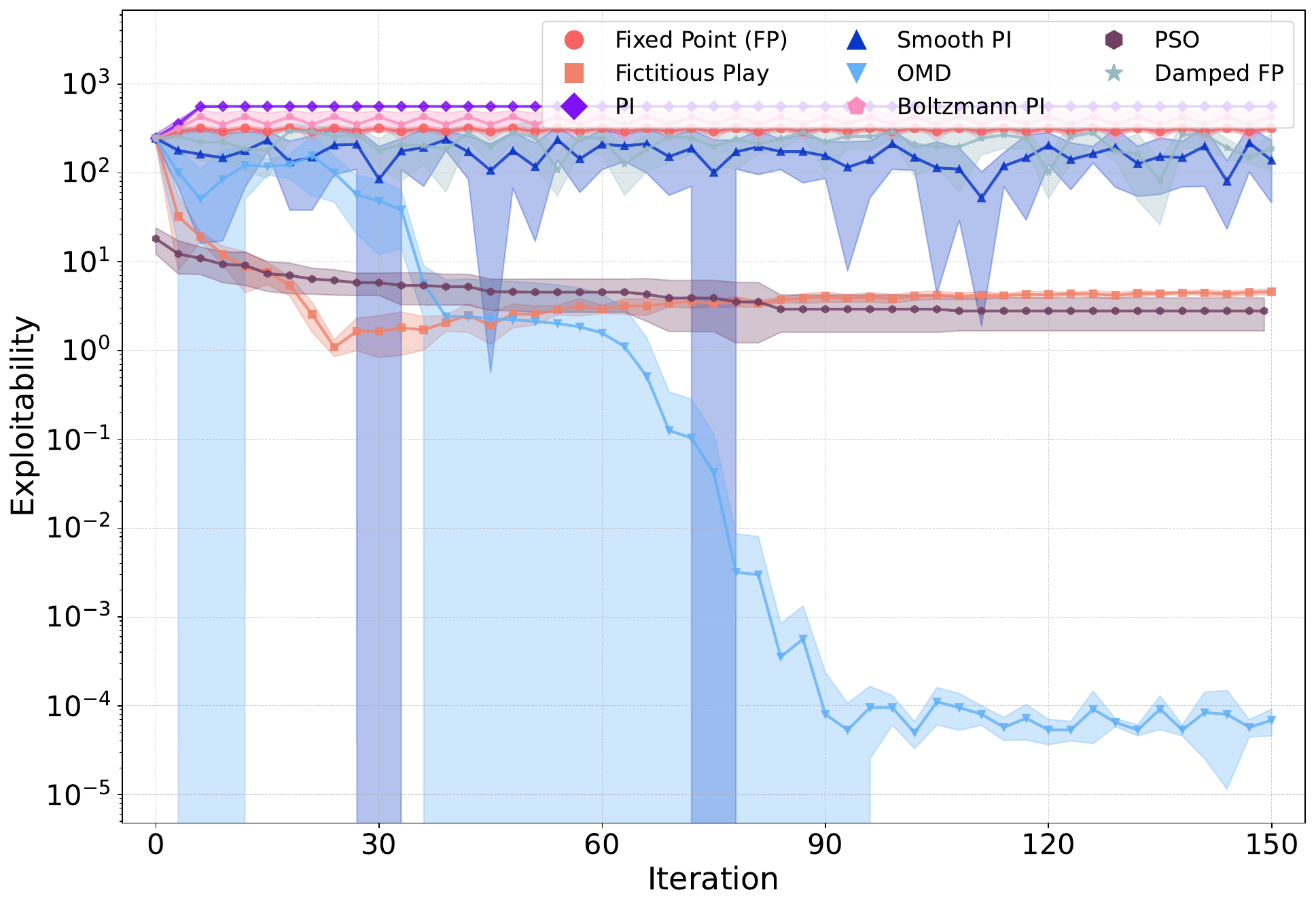}
    \includegraphics[width=0.49\linewidth]{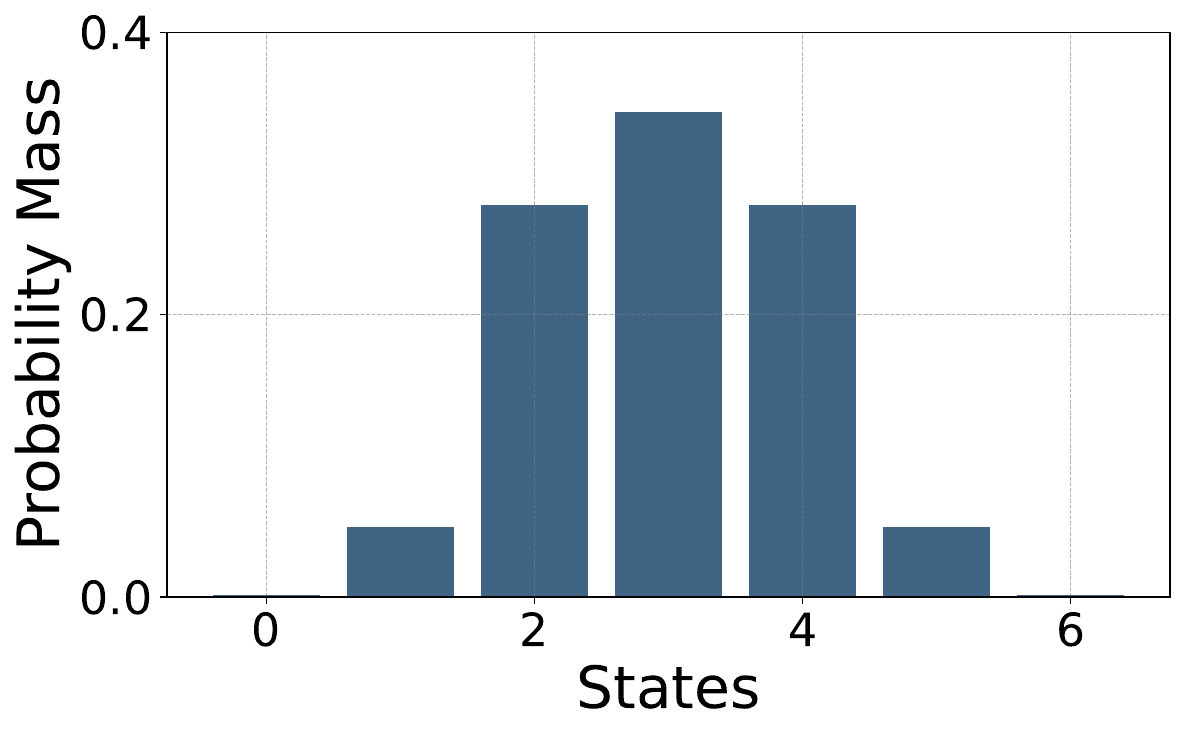}
    \includegraphics[width=0.49\linewidth]{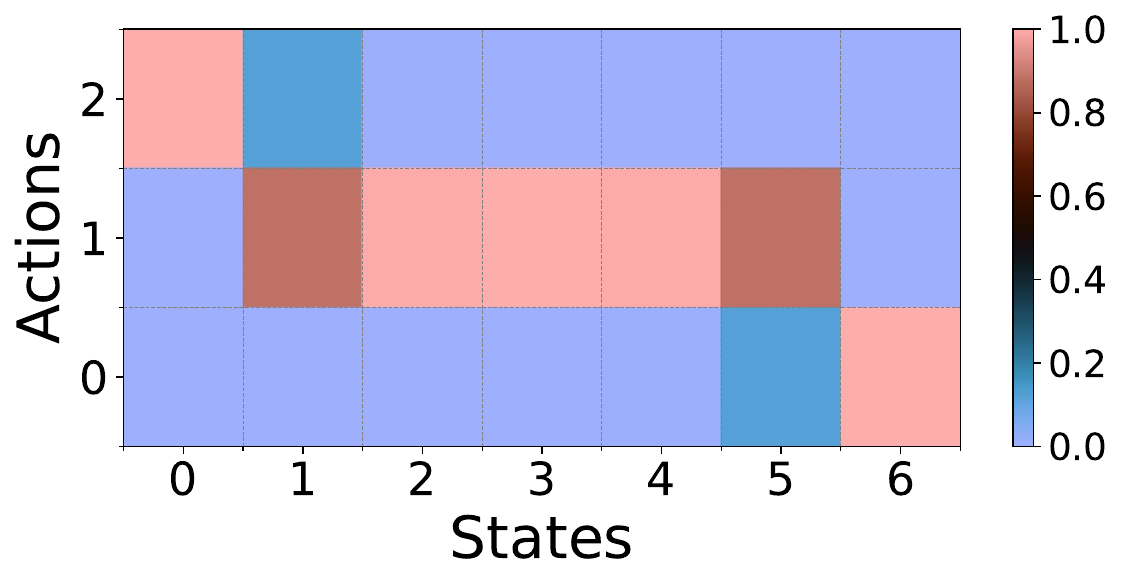}
    \caption{\textbf{Beach Bar Problem} Params. $\alpha=40, c_2=3, c_1=2$}
    \label{fig:LL_env_sweep}
\end{figure}
\begin{figure}[h!]
    \centering
    \includegraphics[width=0.9\linewidth]{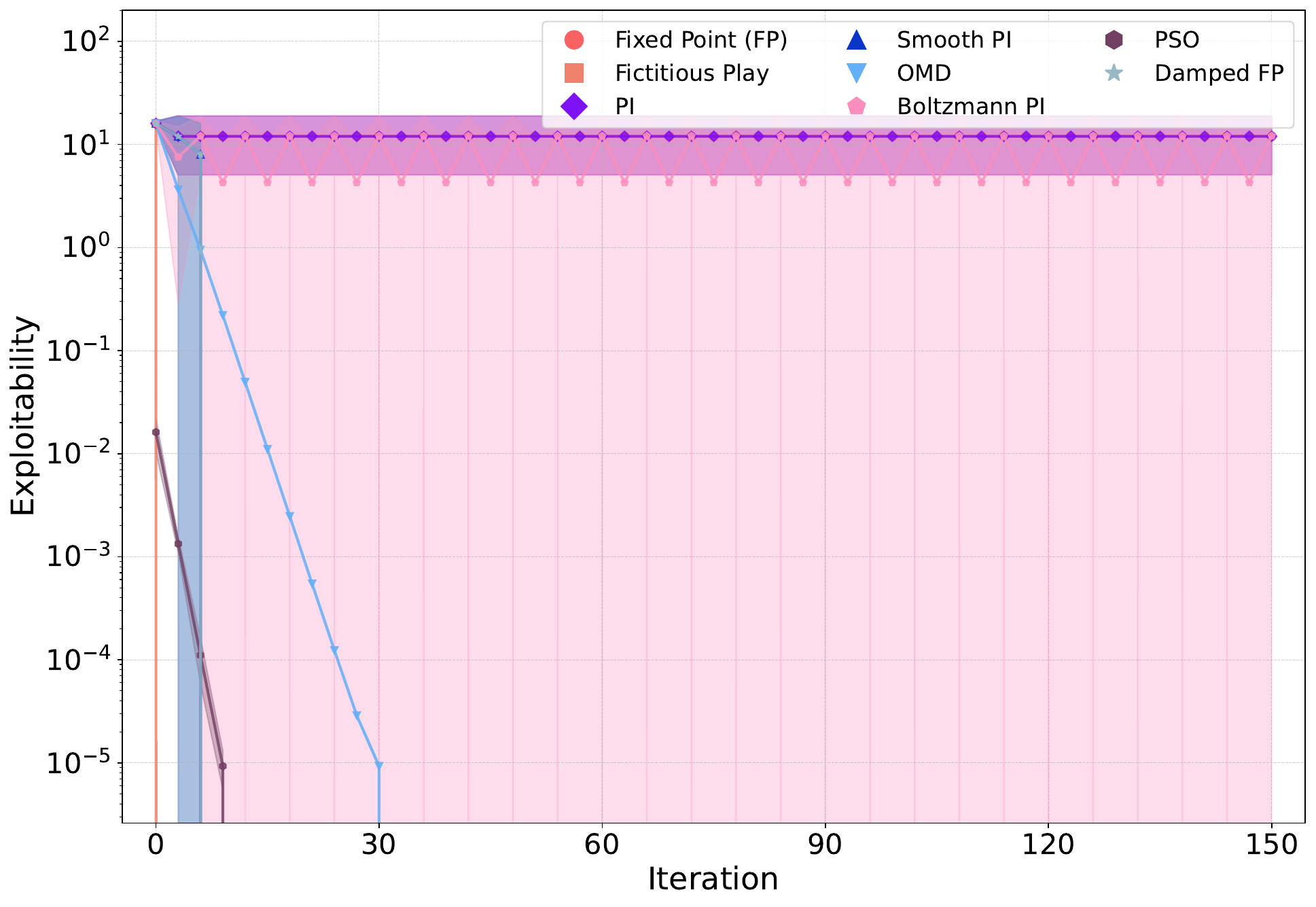}
    \includegraphics[width=0.49\linewidth]{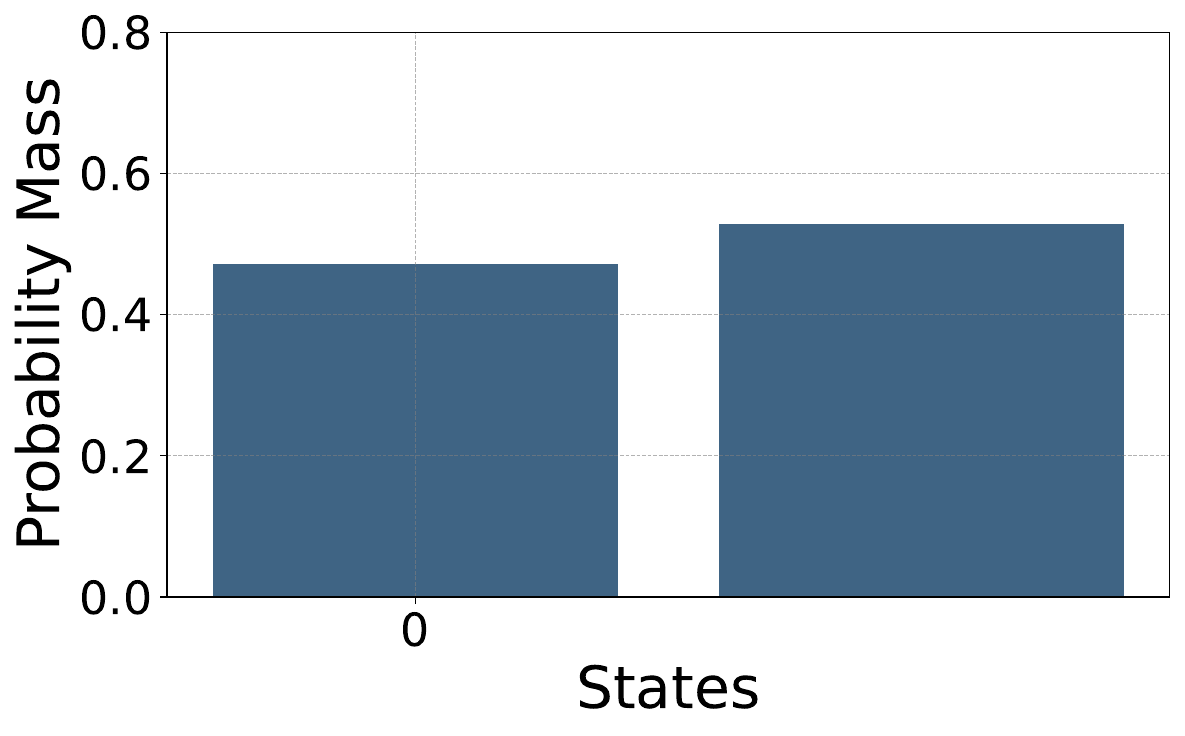}
    \includegraphics[width=0.49\linewidth]{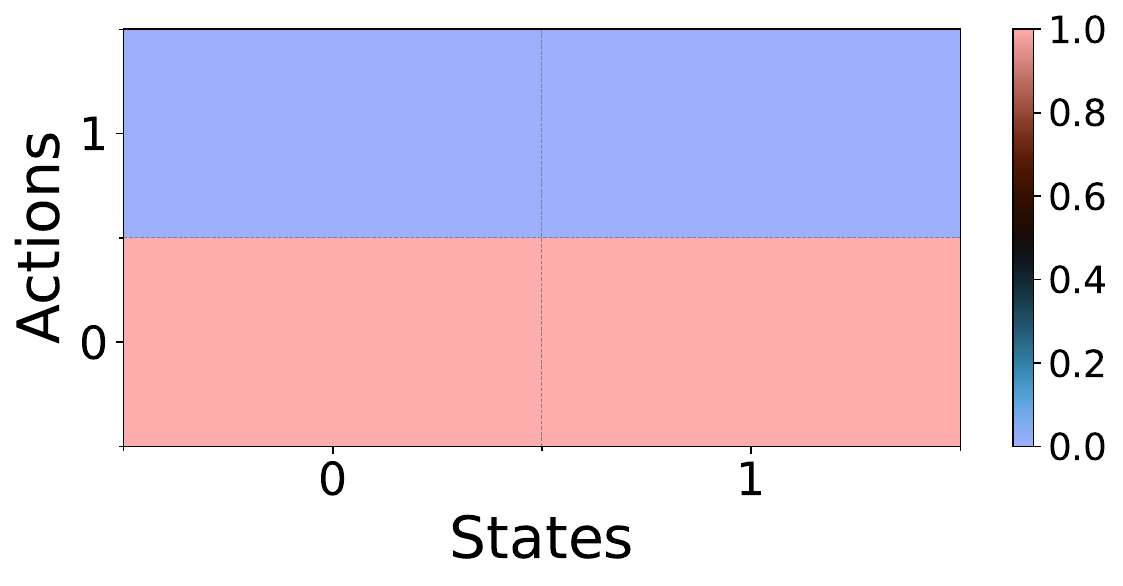}
    \caption{\textbf{Coordination Game} Params. $C=2, \alpha=2$}
    \label{fig:Contraction_env_sweep}
\end{figure}
\begin{figure}[h!]
    \centering
    \includegraphics[width=0.9\linewidth]{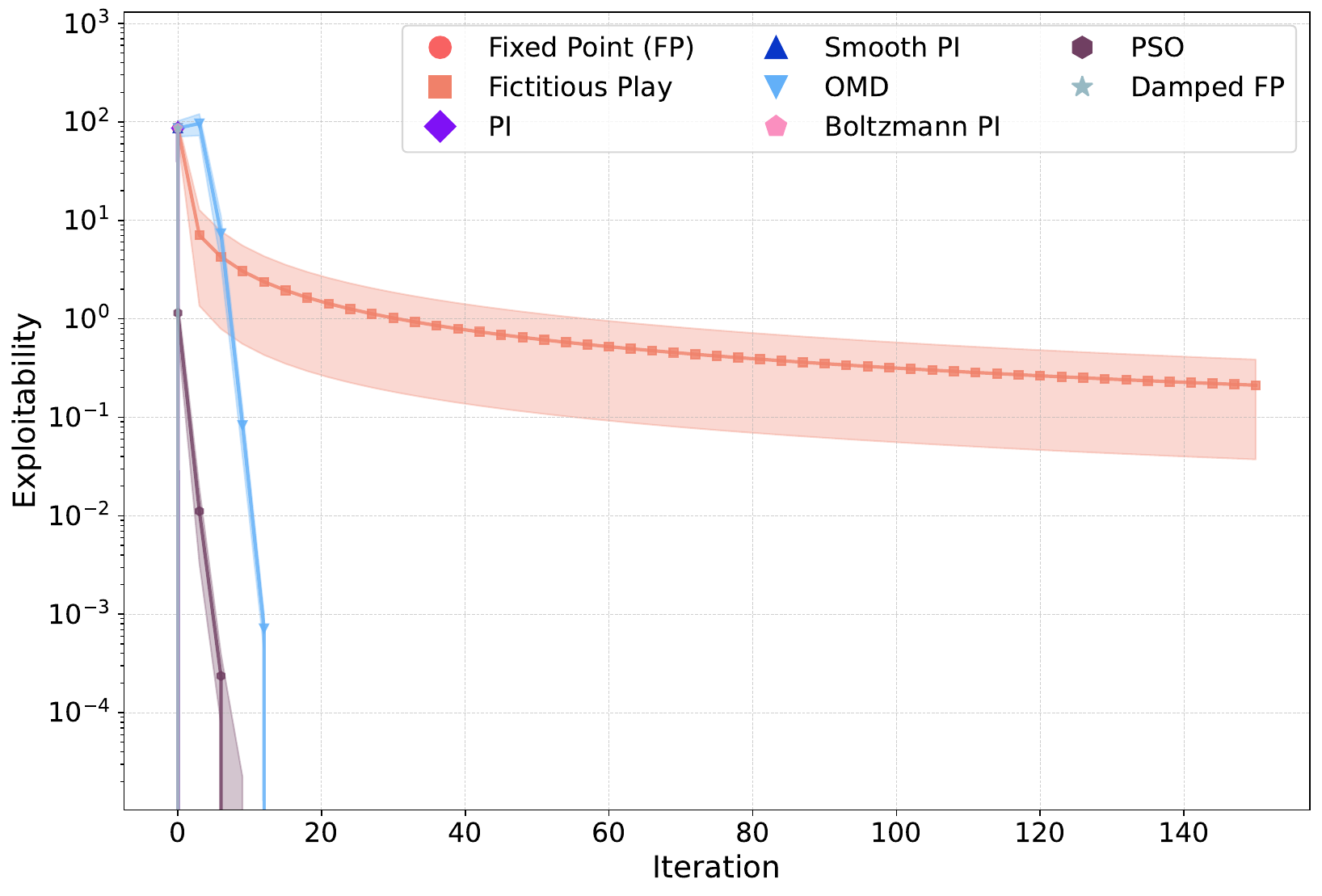}
    \includegraphics[width=0.49\linewidth]{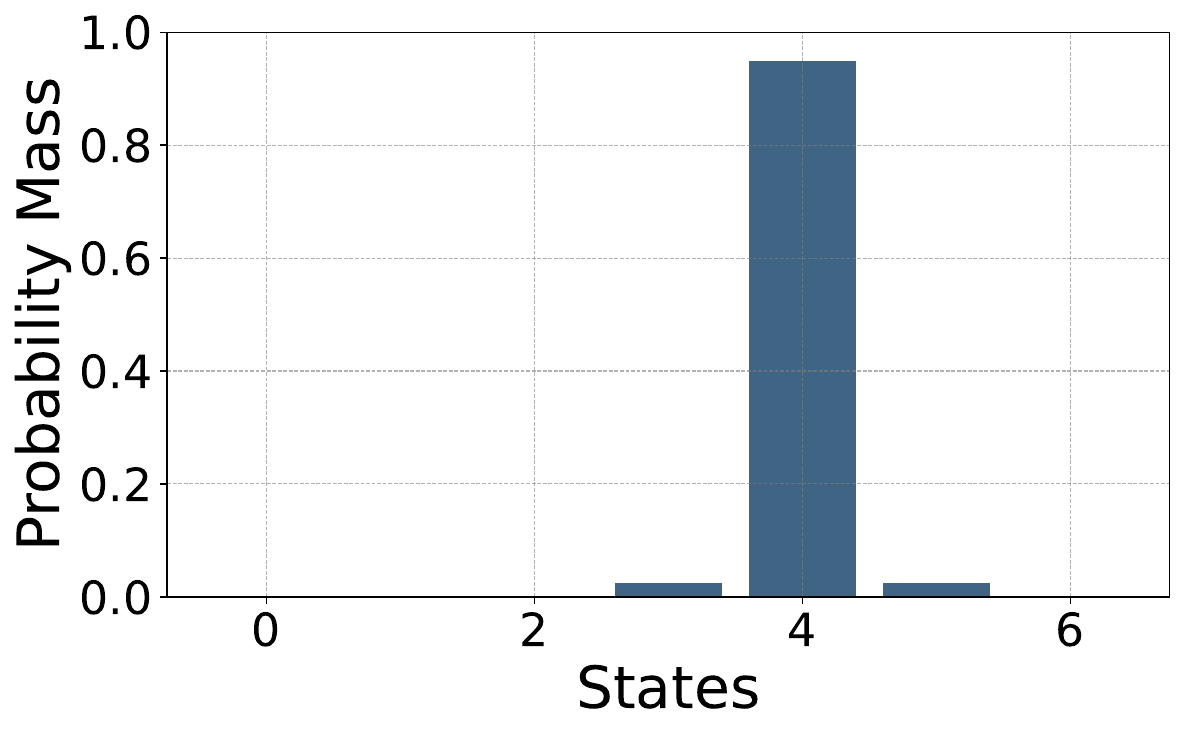}
    \includegraphics[width=0.49\linewidth]{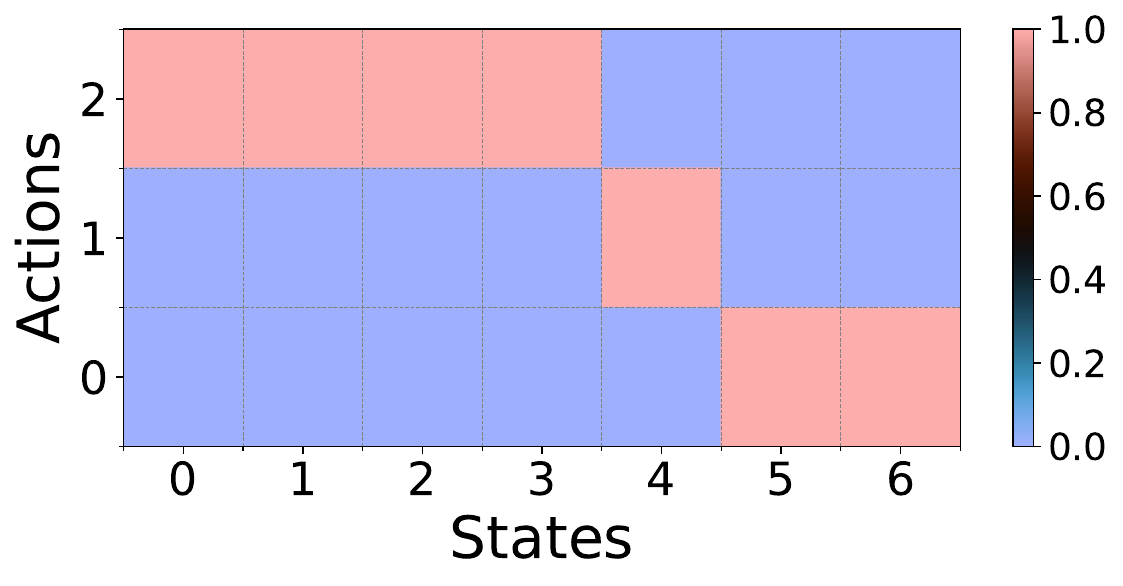}
    \caption{\textbf{Two Beach Bars Problem} Params. $\alpha=5, c_2=20, c_1=4$}
    \label{fig:Multiple_env_sweep}
\end{figure}

\end{document}